\documentclass[10pt]{article}

\usepackage[nonatbib,final]{neurips_2024}
\usepackage[sort,numbers]{natbib}

\usepackage[utf8]{inputenc} %
\usepackage[T1]{fontenc}    %
\usepackage{hyperref}       %
\usepackage{url}            %
\usepackage{booktabs}       %
\usepackage{amsfonts}       %
\usepackage{nicefrac}       %
\usepackage{microtype}      %
\usepackage{xcolor}         %

\usepackage{amsmath}
\usepackage{amssymb}
\usepackage{mathtools}
\usepackage{amsthm}
\usepackage{subcaption}
\usepackage{multirow}
\usepackage{comment}

\usepackage{amsmath,amsfonts,bm}

\def\eqref#1{equation~\ref{#1}}

\def\1{\bm{1}}

\DeclareMathAlphabet{\mathsfit}{\encodingdefault}{\sfdefault}{m}{sl}
\SetMathAlphabet{\mathsfit}{bold}{\encodingdefault}{\sfdefault}{bx}{n}

\def\gD{{\mathcal{D}}}

\def\gL{{\mathcal{L}}}

\def\gN{{\mathcal{N}}}

\def\sN{{\mathbb{N}}}

\def\sR{{\mathbb{R}}}

\newcommand{\R}{\mathbb{R}}

\DeclareMathOperator*{\argmax}{arg\,max}
\DeclareMathOperator*{\argmin}{arg\,min}

\usepackage{algorithm,algorithmic}

\usepackage[capitalize,noabbrev]{cleveref}

\newcommand{\ourmethod}{USEFUL}
\newcommand{\slow}{slow-learnable~}
\newcommand{\fast}{fast-learnable~}
\newcommand{\Slow}{Slow-learnable~}
\newcommand{\Fast}{Fast-learnable~}
\newcommand{\name}{UpSample Early For Uniform Learning}

\newcommand{\norm}[1]{ \left\| #1 \right\| }

\newcommand{\vct}[1]{\pmb{#1}}
\newcommand{\mtx}[1]{\pmb{#1}}
\newcommand{\inner}[2]{ \langle #1, #2 \rangle }

\theoremstyle{plain}
\newtheorem{theorem}{Theorem}[section]

\newtheorem{lemma}[theorem]{Lemma}

\theoremstyle{definition}
\newtheorem{definition}[theorem]{Definition}
\newtheorem{assumption}[theorem]{Assumption}
\theoremstyle{remark}

\usepackage[textsize=tiny]{todonotes}

\title{
Changing the Training Data Distribution to \\Reduce Simplicity Bias Improves\\ In-distribution 
Generalization 
}

\author{%
  Dang Nguyen \quad Paymon Haddad \quad Eric Gan \quad Baharan Mirzasoleiman \\
  Department of Computer Science, UCLA \\
}

\begin{document}

\maketitle

\begin{abstract}
Can we modify the training data distribution to encourage the underlying optimization method toward finding solutions with superior generalization performance on \textit{in-distribution} data? In this work, we approach this question for the first time by comparing the inductive bias of gradient descent (GD) with that of sharpness-aware minimization (SAM). By studying a two-layer CNN, we rigorously prove that SAM learns different features more uniformly, particularly in early epochs. That is, SAM is less susceptible to simplicity bias compared to GD. 
We also show that examples containing features that are learned early are separable from the rest based on the model's output.
Based on this observation, we propose %
a method that (i) clusters examples based on the network output early in training, %
(ii) identifies a cluster of examples with similar network output, and (iii) upsamples the rest of examples only once
to alleviate the %
simplicity bias. We show empirically that %
\ourmethod\ effectively improves the generalization performance on the \textit{original} data distribution when training with various gradient methods, including (S)GD and SAM. %
Notably, we demonstrate that our method can be combined with SAM variants and existing data augmentation strategies to achieve, to the best of our knowledge, state-of-the-art performance for training ResNet18 on %
CIFAR10, STL10, CINIC10, Tiny-ImageNet; ResNet34 on CIFAR100; and VGG19 and DenseNet121 on CIFAR10. Our code
is available at \url{https://github.com/BigML-CS-UCLA/TADA}. \looseness=-1 %
\end{abstract}

\section{Introduction}\label{sec:introduction}
Training data is a key component of machine learning pipelines and directly impacts its performance. Over the last decade, there has been a large body of efforts concerned with improving learning from a given training dataset by designing more effective optimization methods \cite{foret2020sharpness,kingma2014adam,yao2021adahessian} or neural networks with improved structures \cite{liu2018progressive,zoph2016neural,pham2018efficient} or higher-capacity \cite{nakkiran2021deep,mei2022generalization}. 
More recently, improving the quality of the training data has emerged as a popular avenue to improve generalization performance. Interestingly, higher-quality data can further improve the performance when larger models and better optimization methods are unable to do so \cite{hoffmann2022training, gadre2023datacomp}. 
Recent efforts to improve the data quality have mainly focused on filtering irrelevant, noisy, or harmful examples \cite{steinhardt2017certified, li2020dividemix, gadre2023datacomp}.
Nevertheless, it remains an open question if one can %
change the distribution of a \textit{clean} training data to further improve the \textit{in-distribution} generalization performance of models trained on it. \looseness=-1

At first glance, the above question may seem unnatural, as it disputes a fundamental assumption that training and test data should come from the same distribution \cite{james2013introduction}. 
Under this assumption, minimizing the training loss generalizes well on the test data \cite{belkin2019reconciling}. 
Nevertheless, for overparameterized neural networks with more parameters than training data, there are many zero training error solutions, all global minima of the training objective, with \textit{different generalization} performance \cite{gunasekar2017implicit}. 
Thus, one may still hope to carefully change the data distribution to drive the optimization algorithms towards finding more generalizable solutions on the \textit{original} data distribution. \looseness=-1

In this work, we take the first steps towards addressing the above problem. 
To do so, we rely on recent results in non-convex optimization, 
showing the superior generalization performance of
sharpness-aware-minimization (SAM) \cite{foret2020sharpness} over (stochastic) gradient descent (GD). SAM finds flatter local minima by simultaneously minimizing the loss value and loss sharpness. In doing so, it outperforms (S)GD and obtains state-of-the-art performance, at the expense of doubling the training time \cite{zheng2021regularizing,du2021efficient}. Our key idea is that if one can change the training data distribution such that learning shares similar properties to that of training with SAM, 
then the new distribution can drive (S)GD and even SAM toward finding more generalizable solutions. %

To address the above question, we first theoretically analyze the dynamics of training a two-layer convolutional neural network (CNN) with SAM and compare it with that of GD. We rigorously prove that SAM learns different features in a more \textit{uniform speed} compared to GD,
particularly \textit{early} in training. %
In other words, we show that \textit{SAM %
is less susceptible to simplicity bias} than GD. Simplicity bias of SGD makes the model learn simple solutions with minimum norm \cite{gunasekar2017implicit} and has long been conjectured to be the reason for the superior generalization performance of overparameterized models by providing implicit regularization  \cite{neyshabur2014search, hu2020surprising,valle2018deep,gunasekar2017implicit,belkin2019reconciling,nakkiran2021deep}. 
Nevertheless, the minimum-norm solution found by GD can have a suboptimal performance \cite{shah2018minimum}. %

Following our theoretical results, we formulate changing the distribution of a training dataset such that different features are learned at a more uniform speed.
First, we prove that the model output for examples containing features that are learned early by GD is separable from the rest of examples in their class.
Then, we propose changing the data distribution by (i) identifying a cluster of examples with similar model output early in training, (ii) upsampling the remaining examples once to speed up their learning, and (iii) restarting training on the modified training distribution.
Our method, 
\name\ (\ourmethod),
effectively alleviates the simplicity bias %
and consequently improves the generalization performance.
Intuitively, learning features in a more uniform speed prevents the model to overfit underrepresented but useful features that otherwise are learned in late training stages. When the model overfits an example, it cannot learn its features in a generalizable manner. This harms the generalization performance on the original data distribution.

We show the effectiveness of \ourmethod\ in alleviating the simplicity bias and improving the generalization via extensive experiments. %
First, we show that despite being relatively lightweight,
\ourmethod\ effectively
improves the generalization performance of SGD and SAM. Additionally, we show that \ourmethod\ can be easily applied with various optimizers and data augmentation methods to improve in-distribution generalization performance even further.
For example, applying \ourmethod\ with SAM and TrivialAugment (TA) \cite{muller2021trivialaugment} achieves, to the best of our knowledge, \textit{state-of-the-art} accuracy for image classification for training ResNet18 on %
CIFAR10, STL10, CINIC10, Tiny-ImageNet; ResNet34 on CIFAR100; and VGG19 and DenseNet121 on CIFAR10. We also empirically confirm the benefits of \ourmethod\ to out-of-distribution performance, but we emphasize that this is not the focus of our work. \looseness=-1

\section{Related Works}\label{sec:related_works}
\textbf{Sharpness-aware-minimization (SAM).}  
Motivated by the generalization advantages of flat local minima, sharpness-aware minimization (SAM) was concurrently proposed in~\cite{foret2020sharpness,zheng2021regularizing} to minimize the training loss at the worst perturbed direction from the current parameters. 
SAM has been shown to obtain state-of-the-art on a variety of tasks \cite{foret2020sharpness}.
Additionally, SAM has been shown to be beneficial %
in other settings, including label noise~\cite{foret2020sharpness,zheng2021regularizing}, %
and domain generalization~\cite{cha2021swad,wang2023sharpness}.
There have been recent efforts
to understand the %
generalization benefits of SAM.
The most popular explanation is based on the %
Hessian spectra, %
empirically~\cite{foret2020sharpness,kaur2023maximum} and theoretically~\cite{wen2022does,bartlett2023dynamics}. 
Other works showed that SAM %
finds a sparser solution in diagonal linear networks \cite{andriushchenko2022towards}, and %
exhibits benign overfitting under much weaker signal strength compared to (S)GD \cite{chen2023does}.
More recently, SAM is shown to also benefit out-of-distribution (OOD). In particular, \cite{springer2023sharpness} suggested that SAM 
promotes diverse feature learning by {empirically} studying a simplified version of SAM which only perturbs the last layer. They showed that SAM upscales the last layer's weights to induce feature diversity, which benefits OOD.
In contrast, %
we rigorously analyze a 2-layer non-linear CNN and prove that SAM learns (the same set of) features at a more uniform speed, which benefits the in-distribution (ID) settings. Our results reveal an orthogonal effect of SAM that benefits the ID generalization by reducing the simplicity bias, and provides a complementary view to prior works explaining superior ID generalization performance of SAM. %
We then propose a method to learn features more evenly by changing the data distribution.
\looseness=-1

\textbf{Simplicity bias (SB).} 
(S)GD has an inductive bias towards learning simpler solutions with minimum norm \cite{gunasekar2017implicit}. It is empirically observed \cite{kalimeris2019sgd} and theoretically proved \cite{hu2020surprising} that SGD learns linear functions in the early training phase and more complex functions later in training. 
SB of SGD has been long conjectured to be the reason for the superior in-distribution generalization performance of overparameterized models, by providing %
capacity control or implicit regularization 
~\cite{neyshabur2014search,hermann2020shapes,shah2020pitfalls,pezeshki2021gradient}. %
On the other hand, in the OOD setting, simplicity bias is known to contribute to shortcut learning %
by causing models to exclusively rely on the simplest spurious feature and remain invariant to the complex but more predictive features %
\cite{shah2020pitfalls,teney2022evading,yang2023identifying}. Prior works on mitigating simplicity bias have been shown effective in the OOD settings~\cite{teney2022evading,tiwari2023overcoming}.
In contrast, our work shows, for the first time, that reducing the simplicity bias also benefits the ID settings. By studying the mechanism of feature learning in a two-layer nonlinear CNN, we prove that SAM is less susceptible to simplicity bias than GD, in particular {early} in training, which contributes to its superior performance.
Then, we show that training data distribution can be modified to reduce the SB and improve the in-distribution generalization. In Appendix~\ref{app:simplicity_bias}, we empirically confirm that existing simplicity bias mitigation methods also improve the in-distribution performance, but to a smaller extent than ours. 

\textbf{Distinction from Existing Settings.} %
Our work is distinct from the following literature: %

(1) \textit{Distribution Shift.}
Unlike %
distribution shift and shortcut learning~\cite{sagawa2019distributionally,kirichenko2022last,deng2023robust,puli2023don}, we \textit{do not} assume existence of domain-dependent (non-generalizable) features or strong spurious correlations in the training data, or shift between training and test distribution. %
We focus on \textit{in-distribution} generalization, where training and test distributions are the same and all the features in the training data are relevant for generalization. 
In Appendix~\ref{app:ood_experiments} we empirically show the benefits of our method to distribution shift, but we emphasize that this is not the focus of our study and we leave this direction to future work. \looseness=-1

(2) \textit{Long-tail distribution.} Long-tailed data is studied as a special case of distribution shift in which (sub)classes are highly imbalanced in training but are (more) balanced in test data~\cite{van2018inaturalist,cui2019class}. Long-tail methods resample the data at the class or subclass level to match the training and test distribution. 
In contrast, in our settings, training and test data follow the same distribution. Nevertheless, our method can be applied to improve the performance of long-tail datasets, as we confirm in Appendix~\ref{app:ood_experiments}.

(3) \textit{Improving Convergence.}
A body of work speeds up convergence of (S)GD to find the \textit{same solution} faster. Such methods iteratively sample or reweight examples based on loss or gradient norm during training~\cite{zhao2015stochastic,katharopoulos2018not,johnson2018training,el2022stochastic}. 
In contrast, our work does not intend to speed up training to find the same solution faster, but intends to find a \textit{more generalizable solution} on the original data distribution. %

(4) \textit{Data Filtering Methods.} Filtering methods %
identify and discard or %
downweight
noisy labeled \cite{li2020dividemix}, domain mismatched \cite{gadre2023datacomp}, redundant \cite{lee2021deduplicating,raffel2020exploring,abbas2023semdedup}, or adversarial examples crafted by data poisoning attacks \cite{steinhardt2017certified}. 
In contrast, we assume a \textit{clean} training data and no mismatch between training and test distribution. Our work can be applied to a filtered training data to further improve the performance. 

\section{Theoretical Analysis: SAM Learns Different Features More Evenly}\label{sec:theory}
In this section, we analyze and compare feature learning mechanism of SAM. %
First, we introduce our theoretical settings including data distribution and neural network model in Sec.~\ref{subsec:theoretical_settings}. We then revisit the update rules of GD and SAM in Sec.~\ref{subsec:erm} before presenting our theoretical results in Sec.~\ref{subsec:theoretical_results}. \looseness=-1

\subsection{Theoretical Settings}\label{subsec:theoretical_settings}\vspace{-1mm}
\textbf{Notation.}
We use lowercase letters, lowercase boldface letters, and uppercase boldface letters to denote scalars $(a)$, vectors $(\vct{v})$, and matrices $(\mtx{W})$. 
For a vector $\vct{v}$, we use $\norm{\vct{v}}_2$ to denote its Euclidean norm. 
Given two sequence $\{ x_n \}$ and $\{ y_n \}$, we denote $x_n = O(y_n)$ if $|x_n| \leq C_1 |y_n|$ for some absolute positive constant $C_1$, $x_n = \Omega(y_n)$ if $|x_n| \geq C_2 |y_n|$ for some absolute positive constant $C_2$, and $x_n = \Theta(y_n)$ if $C_3 |y_n| \leq |x_n| \leq C_4 |y_n|$ for some absolute constant $C_3, C_4 > 0$. Besides, we use 
$\Tilde{O}(\cdot), \Tilde{\Omega}(\cdot),$ and $\Tilde{\Theta}(\cdot)$ to hide logarithmic factors in these notations. Furthermore, we denote $x_n = \text{poly}(y_n)$ if $x_n = O(y_n^D)$ for some positive constant D, and $x_n = \text{polylog}(y_n)$ if $x_n = \text{poly}(\log(y_n))$. \looseness=-1

\textbf{Data distribution.}\label{subsec:data_distribution}
We use a popular data distribution used in recent works on feature learning~\cite{allen2020towards,chen2022towards,jelassi2022towards,cao2022benign,kou2023benign,deng2023robust,chen2023does} to represent data as a combination of two
features and noise patches. Additionally, we introduce a probability $\alpha$ to control the frequency of \fast features in the data distribution. 

\begin{definition}[Data distribution]\label{def:data_distribution}
    A data point $(\vct{x}, y) \in (\sR^d)^P \!\!\times \{ \pm 1 \}\!$ is generated from the distribution $\gD(\beta_e, \beta_d, \alpha)$ as follows. %
        We uniformly generate the label $y \in \{ \pm 1 \}$.
        We generate $\vct{x}$ as a collection of $P$ patches: $\vct{x} = (\vct{x}^{(1)}, \vct{x}^{(2)}, \ldots, \vct{x}^{(P)}) \in (\sR^d)^P$, where
        \begin{itemize} 
            \item \textbf{\Slow Feature.} One and only one patch is given by $\beta_d \cdot y \cdot \vct{v}_d$ with $\norm{\vct{v}_d}_2 = 1$, $\inner{\vct{v}_e}{\vct{v}_d} = 0$, and $0 \leq \beta_d < \beta_e \in \sR$.
            \item \textbf{\Fast feature.} One and only one patch is given by $\beta_e \cdot y \cdot \vct{v}_e$ with $\norm{\vct{v}_e}_2 = 1$ with a probability $\alpha \leq 1$. With a probability of $1 - \alpha$, this patch is masked, i.e. \vct{0}.
            \item \textbf{Random noise.} The rest of $P - 2$ patches are Gaussian noise $\vct{\xi}$ that are independently drawn from $N(0, (\sigma_p^2/d) \cdot \mathbf{I}_d)$ with $\sigma_p$ as an absolute constant. 
        \end{itemize}
\end{definition}
For simplicity, we assume $P = 3$, and the noisy patch together with two features form an orthogonal set. Coefficients 
$\beta_e$ and $\beta_d$ characterize the feature strength in our data model. A larger coefficient means that the corresponding feature is learned faster.

\textbf{Two-layer nonlinear CNN.}\label{subsec:CNN}
To model modern state-of-the-art architectures, we analyze a two-layer nonlinear CNN which is also used in~\cite{chen2022towards,jelassi2022towards,cao2022benign,kou2023benign,deng2023robust}. %
Unlike linear models, CNN can handle a data distribution that does not require a fixed position of patches as defined above. Formally,
\begin{equation}
    f(\vct{x}; \mtx{W}) = \sum_{j \in [J]} \sum_{p=1}^P \sigma(\inner{\vct{w}_j}{\vct{x}^{(p)}}), \label{eq:model_output}
\end{equation}
where $\vct{w}_j \in \sR^d$ is the weight vector of the $j$-th filter, $J$ is the number of filters (neurons) of the network, and $\sigma(z) = z^3$ is the activation function, i.e., the main source of non-linearity. $\mtx{W} = [\vct{w}_1, \ldots, \vct{w}_J] \in \sR^{d \times J}$ is the weight matrix of the CNN. Following~\cite{jelassi2022towards,cao2022benign,deng2023robust}, we assume a mild overparameterization %
with $J = \text{polylog}(d)$. We initialize $\mtx{W}^{(0)} \sim \gN(0, \sigma_0^2)$, where $\sigma_0^2 = \text{polylog}(d)/d$.\looseness=-1 %

\subsection{Empirical Risk Minimization: GD vs SAM}\label{subsec:erm}
Consider a $N$-sample training dataset $D = \{ (\vct{x}_i, y_i) \}_{i=1}^N$ in which each data point is generated from the data distribution in Definition~\ref{def:data_distribution}. The empirical loss function of a model $f(\vct{x}; \mtx{W})$ reads
\begin{equation}\label{eq:erm}
    \gL(\mtx{W}) = \frac{1}{N} \sum_{i=1}^N l(y_i f(\vct{x}_i; \mtx{W})),
\end{equation}
where $l$ is the logistic loss defined as $l(z) = \log(1 + \exp(-z))$. The solution $\mtx{W}^\star$ of the empirical risk minimization (ERM) minimizes the above loss, i.e., $\mtx{W}^\star \coloneqq \argmin_{\mtx{W}} \gL(\mtx{W})$. 

\textbf{GD.} Typically, ERM is solved using gradient descent (GD). The update rule at iteration $t$ of GD with learning rate $\eta > 0$ reads
\begin{equation}\label{eq:gd_update}
    \mtx{W}^{(t+1)} = \mtx{W}^{(t)} - \eta \nabla \gL(\mtx{W}^{(t)}).
\end{equation}

\textbf{SAM.} To find solutions with better generalization performance, \citet{foret2020sharpness} proposed the $N$-SAM algorithm that minimizes both loss and curvature. 
SAM's update rule at iteration $t$ reads
\begin{equation}\label{eq:sam_update}
    \mtx{W}^{(t+1)} = \mtx{W}^{(t)} - \eta \nabla \gL(\mtx{W}^{(t)} + \rho^{(t)} \nabla \gL(\mtx{W}^{(t)})),
\end{equation}
where $\rho^{(t)} = \rho > 0$ is the inner step size that is usually normalized by gradient norm, i.e., $\rho^{(t)} = \rho/ \norm{\nabla \gL(\mtx{W}^{(t)})}_F$. \looseness=-1

\subsection{%
Comparing Learning Between \fast \& \slow Features for GD \& SAM
}\label{subsec:theoretical_results}

Next, we present our theoretical results on training dynamics of the two-layer nonlinear CNN using GD and SAM. We characterize the learning speed of features by studying the growth of the {model outputs before the activation function}, i.e., $\inner{\vct{w}_j^{(t)}}{\vct{v}_e}$ and $\inner{\vct{w}_j^{(t)}}{\vct{v}_d}$. 
We first prove that \textit{early} in training, both GD and SAM \textit{only} learn \fast feature. %
Then, we show SAM %
learns \slow and \fast features at a more uniform speed. 

\begin{theorem}[\textbf{GD Feature Learning}]\label{the:easy_difficult_gd_main}
    Consider training a two-layer nonlinear CNN model initialized with $\mtx{W}^{(0)} \sim \gN(0, \sigma_0^2)$ on the training dataset $D = \{ (\vct{x}_i, y_i) \}_{i=1}^N$ with distribution $\gD(\beta_e, \beta_d, \alpha)$ with $\alpha^{1/3} \beta_e > \beta_d$. %
    For a small-enough learning rate $\eta$, 
    after training for $T_{\text{GD}}$ iterations, w.h.p., the %
    model: %
    (1) learns the \fast feature $\vct{v}_e\!\!:\! \max_{j \in [J]} \inner{\vct{w}_j^{(T_{\text{GD}})}\!}{\vct{v}_e} \!\!\geq \!\tilde{\Omega}(1/\beta_e)$;\! (2) does not learn the \slow feature $\vct{v}_d\!: \!\max_{j \in [J]} \inner{\vct{w}_j^{(T_{\text{GD}})}}{\vct{v}_d} = \tilde{O}(\sigma_0)$. 
\end{theorem}

\begin{theorem}[\textbf{SAM Feature Learning}]\label{the:easy_difficult_sam_main}
    Consider training a two-layer nonlinear CNN model initialized with $\mtx{W}^{(0)} \sim \gN(0, \sigma_0^2)$ on the training dataset $D = \{ (\vct{x}_i, y_i) \}_{i=1}^N$ with distribution $\gD(\beta_e, \beta_d, \alpha)$ with $\alpha^{1/3} \beta_e > \beta_d$. For small-enough learning rate $\eta$ and perturbation radius $\rho$, 
    after training for $T_{\text{SAM}} > T_{\text{GD}}$ iterations, w.h.p., the model: %
    (1) learns the \fast feature $\vct{v}_e: \max_{j \in [J]} \inner{\vct{w}_j^{(T_{\text{SAM}})}}{\vct{v}_e} \geq \tilde{\Omega}(1/\beta_e)$; (2) does not learn the \slow feature $\vct{v}_d: \max_{j \in [J]} \inner{\vct{w}_j^{(T_{\text{SAM}})}}{\vct{v}_d} = \tilde{O}(\sigma_0)$.
\end{theorem}
The detailed proof of Theorems~\ref{the:easy_difficult_gd_main} and~\ref{the:easy_difficult_sam_main} are deferred to Appendices~\ref{app:proof_gd} and~\ref{app:proof_sam}. 

\textbf{Discussion.} Note that a larger value of $\inner{\vct{w}_j^{(t)}}{\vct{v}}$ for $\vct{v} \in \{ \vct{v}_e, \vct{v}_d \}$ indicates better learning of the feature vector $\vct{v}$ by neuron $\vct{w}_j$ at iteration $t$. From the above two theorems, the growth rate of the \fast feature is significantly faster than that of the \slow feature. 
As a small portion $(1 - \alpha)$ of the dataset does not have the \fast feature, the model needs to learn the \slow feature to improve the performance.\looseness=-1
%

%
%
%

Next, we show that SAM learns \fast and \slow features more evenly.
We denote by $G_e^{(t)} = \max_{j \in [J]} \inner{\vct{w}_j^{(t)}}{\vct{v}_e}$ and $G_d^{(t)} = \max_{j \in [J]} \inner{\vct{w}_j^{(t)}}{\vct{v}_d}$ the alignment of model weights with \fast and \slow features, when training with GD. Similarly, we denote by $S_e^{(t)}$ and $S_d^{(t)}$ the alignment of model weights with \fast and \slow, when training with SAM.\looseness=-1 %

\begin{theorem}[\textbf{SAM learns features more evenly than GD}]\label{the:sam_learn_more_uniform_than_gd_main}
    Consider the same model and training dataset as Theorems~\ref{the:easy_difficult_gd_main} and~\ref{the:easy_difficult_sam_main}.
    Assume that the learning rate $\eta$ and the perturbation radius $\rho$ are sufficiently small. Starting from the same initialization, the growth of \fast and \slow features in SAM is more balanced than that in SGD, i.e., for every iteration $t \in [1, T_0]$:
    \begin{align}\label{eq:contribution_gap}
        S_e^{(t)} - S_d^{(t)} < G_e^{(t)} - G_d^{(t)}.
    \end{align}
\end{theorem}
We prove Theorem~\ref{the:sam_learn_more_uniform_than_gd_main} by induction in Appendix~\ref{app:proof_sam} and back it by toy experiments in Section~\ref{subsec:toy_datasets}. 

\textbf{Discussion.} Intuitively, our proof is based on the fact that the difference between the growth of \fast and \slow features in SAM 
is smaller than that of GD. 
Thus, starting from the same initialization, the \slow feature contributes relatively more to the model prediction in SAM than it does in SGD. Thus, the \slow feature benefits SAM, by reducing its overreliance on the \fast features. 
We note that as neural networks are nonlinear, a small change in the output can actually result in a big change in the model and its performance. Even in the extreme setting when two features have identical strength and the \fast feature exists in all examples, i.e., $\beta_e = \beta_d = \alpha = 1$, the gap in Eq.~\ref{eq:contribution_gap} is significant as we confirm in Figure~\ref{fig:toy_dataset_extreme_setting} in Appendix~\ref{app:additional_results}.

\textbf{Remark.} The network often overfits \slow features that are learned late during the training and do not learn them in a generalizable manner. This harms the generalization performance on the test set sampled from the \textit{original} data distribution. 

Theorems~\ref{the:easy_difficult_gd_main} and~\ref{the:easy_difficult_sam_main} show that we can make the model learn more from the \slow feature by increasing the value of $\beta_d$. Based on this intuition, we have the following theorem.\looseness=-1

\begin{theorem}[\textbf{One-shot upsampling}]\label{the:one_step_upsampling_main}
Under the assumptions of Theorems~\ref{the:easy_difficult_gd_main}, ~\ref{the:easy_difficult_sam_main},
for a sufficiently small noise,
from any iteration $t$ during early training, we have the following results:
\begin{enumerate}
    \item The \slow feature has a larger contribution to the normalized gradient of the 1-step SAM update, %
    compared to that of GD. 
    \item Amplifying the strength of the \slow feature %
    increases its contribution %
    to the normalized gradients of GD and SAM.
    \item There exists an upsampling factor $k$ s.t. the normalized gradient of the 1-step GD update on $\gD(\beta_e, k \beta_d, \alpha)$ recovers the normalized gradient of the 1-step SAM update on $\gD(\beta_e, \beta_d, \alpha)$.\looseness=-1
\end{enumerate}
\end{theorem}

\textbf{Discussion.} Proof of Theorem~\ref{the:one_step_upsampling_main} is given in Appendix~\ref{app:proof_one_step_upsampling}. %
We see that we can %
learn features at a more uniform speed
by training on a new dataset $\gD(\beta_e, \beta_d', \alpha)$ with a larger strength $\beta_d' > \beta_d$. But, the value of coefficient $\beta_d'$ varies by the model weights and gradient at each iteration $t$. \looseness=-1

\begin{algorithm}[!t]
    \caption{\name\ (\ourmethod)}
    \label{alg:useful}
    \begin{algorithmic}
        \STATE {\bfseries Input:} Original dataset $D$, Model $f(\cdot, \mtx{W}^{(0)})$, Separating epoch $t$, Total epochs $T$.
        \STATE{Train the model $f(\cdot, \mtx{W}^{(0)})$ on $D$ for $t$ epochs.}
        \FOR{every class $c \in D$}
        \STATE $\{C_1, C_2\} \leftarrow k$-means($f(\vct{x}_j; \mtx{W}^{(t)}))$  %
        \STATE $D = D\cup C_2$, where $C_2$ is the cluster with higher average loss
        \ENDFOR
        \STATE Train $f(\cdot, \mtx{W}^{(0)})$ on $D$ for $T$ epochs
        \STATE {\bfseries Output:} Model $f(\cdot, \mtx{W}^{(T)})$
    \end{algorithmic}
\end{algorithm}

\textbf{Remark.}
Intuitively, Theorem \ref{the:one_step_upsampling_main} implies that by descending over a flatter trajectory, SAM learns \slow features relatively earlier in training, compared to GD. 
While the largest difference between feature learning of SAM and (S)GD is attributed to early training dynamics (due to the simplicity bias of (S)GD and its largest contribution early in training), SAM learns features at a more uniform speed during the \textit{entire} training. 
This effect is, however, difficult to theoretically characterize exactly. %
Therefore, learning features at a more uniform speed via SAM help yield flatter minima with better generalization performance. %
We note that while SAM learns \fast and \slow features at a more uniform speed, it still suffers from simplicity bias and learns \fast features earlier (although less so than GD) as evidenced in our Theorem \ref{the:easy_difficult_sam_main}. %

\begin{figure*}[t!]
    \centering
    \includegraphics[width=\textwidth]{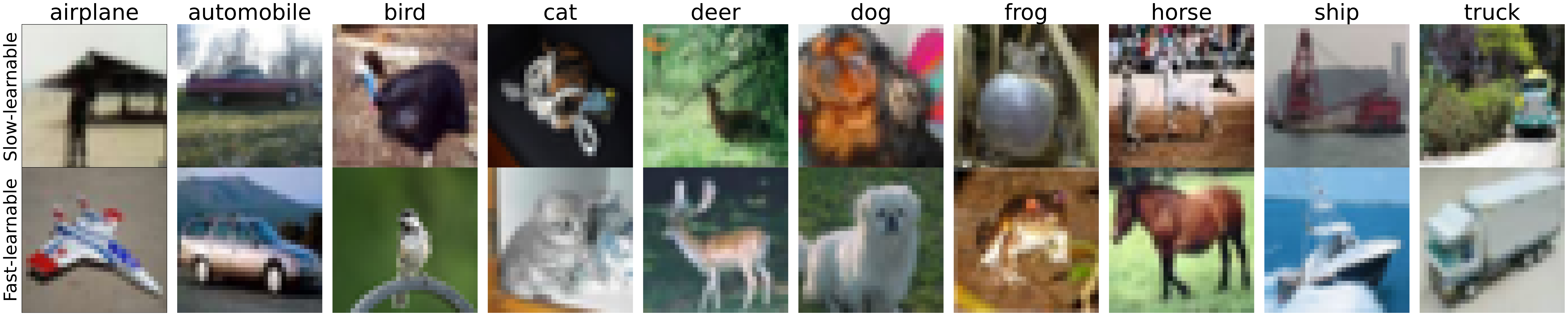}
    \caption{Examples of \slow (top) and \fast (bottom) in CIFAR-10 found by our method. Examples in the top row (\slow\!\!) are harder to identify visually and look more ambiguous (part of the object is in the image or the object is smaller and the area associated with the background is larger). In contrast, examples in the bottom row (\fast\!\!) are not ambiguous and are clear representatives of their corresponding class, hence are very easy to visually classify (the entire object is in the image and the area associated with the background is small).}
    \label{fig:clustering_cifar10}
\end{figure*}
\section{Method: \name\ (\ourmethod)}\label{sec:method}
Motivated by our theoretical results, we aim to speed up learning the \slow features in the training data. 
This drive the network to learn \fast and \slow features at a more uniformly speed, and ultimately improves the {in-distribution} generalization performance.

\textbf{Step 1: Identifying examples with \fast features.} %
As shown in Theorems~\ref{the:easy_difficult_gd_main} and~\ref{the:easy_difficult_sam_main}, \fast features are learned early in training, and the \textit{model output for examples containing \fast features %
are highly separable} from the rest of examples in their class, early in training. This is illustrated for one class of a toy example and CIFAR-10 in Fig. \ref{fig:clustering_toy}.
Motivated by our theory, we seek to find a cluster of examples with similar model outputs early in training.
To do so, we apply $k$-means clustering to the last-layer activation vectors of examples in every class, to separate examples with \fast features from the rest of examples. %
Formally, for examples in every class with $y_j = c$, we find: %
\begin{equation}
    \argmin_C \sum_{i\in\{1,2\}} \sum_{y_j=c, j\in C_i} \|
    f(\vct{x}_j; \mtx{W}^{(t)}) -\vct{\mu}_i\|^2,
\end{equation} 
where $\vct{\mu}_i$ is the center of cluster $S_i$.
The cluster with lower average loss will contain the majority of examples containing \fast features, whereas the remaining examples contain \slow features in the training data. Examples of images in \fast and \slow clusters of CIFAR-10 found by \ourmethod\ are illustrated in Fig. \ref{fig:clustering_cifar10}. 

\textbf{The choice of clustering.} Our choice of clustering is motivated by our Theorems~\ref{the:easy_difficult_gd_main} and~\ref{the:easy_difficult_sam_main} which show that examples with \fast features are separable based on model output from the rest of examples in their class. While examples with \fast features are expected to have a lower loss, loss of examples may oscillate during the training and makes it difficult to find an accurate cut-off for separating the examples. Besides, as \fast features may not be \textit{fully} learned early in training, examples containing \fast features may not necessarily have the right prediction, %
thus misclassification cannot separate examples accurately. In contrast, clustering does not require hyperparameter tuning and performs well for separating examples, as we confirm in our ablation studies. \looseness=-1

\textbf{Step 2: One-shot upsampling of \slow features.} 
Next, we upsample examples that are not in the cluster of points containing \fast features. 
This speeds up learning \slow features and encourages the model to learn different features at a more uniform speed. Thus, it improves the in-distribution performance based on Theorem~\ref{the:one_step_upsampling_main}.
As discussed earlier, the number of times we upsample these examples should change based on the model weight at each iteration. 
Hence, a multi-stage clustering and sampling can yield the best results.
Nevertheless, we empirically confirm that a 1-shot algorithm that finds \fast examples at an \textit{early} training iteration and upsample the remaining examples by a factor of $k=2$ effectively improves the %
performance.
Notably, in contrast to dynamic sampling or reweighting, \ourmethod\ upsamples examples only once and restart training on the modified but fix distribution. 

\begin{figure*}[t!]
    \centering
    \begin{subfigure}{0.45\textwidth}
        \centering
        \includegraphics[width=0.95\columnwidth]{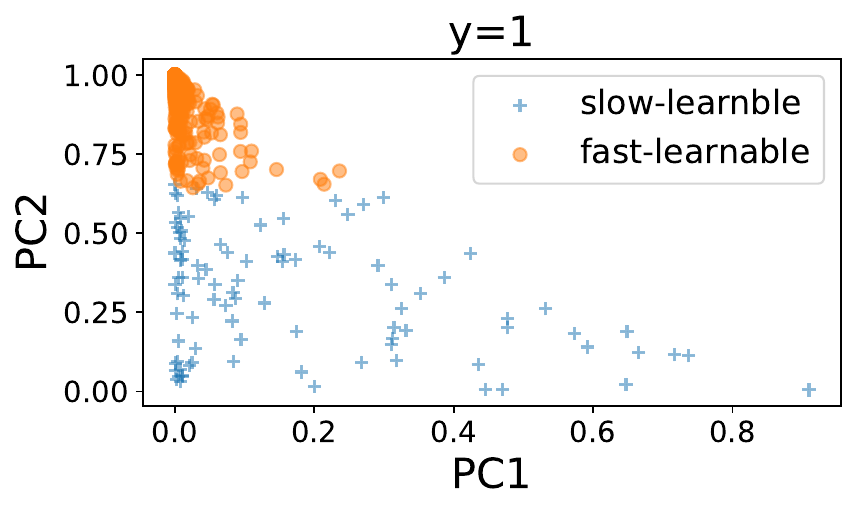}
        \label{fig:cifar10_label1}
    \end{subfigure}
    \hfill
    \begin{subfigure}{0.45\textwidth}
        \centering
        \includegraphics[width=0.85\columnwidth]{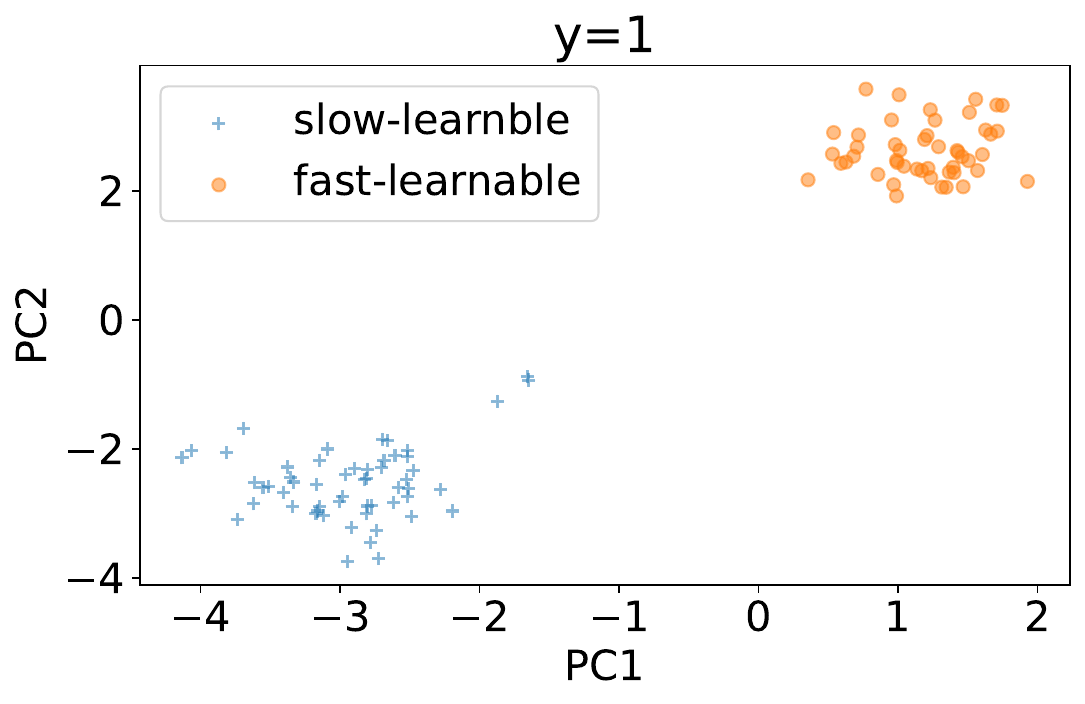}
        \label{fig:clustering_toy_label1}
    \end{subfigure}
    \hfill
    \caption{TSNE visualization of output vectors. %
    (left) ResNet18/CIFAR-10 at epoch 8. (right) CNN/toy data generated based on Definition~\ref{def:data_distribution} with $\beta_d = 0.2, \beta_e = 1, \alpha = 0.9$, iteration 200. 
    }
    \label{fig:clustering_toy}
\end{figure*}

\begin{figure*}[t!]
    \centering
    \begin{subfigure}{0.3\textwidth}
        \centering
        \includegraphics[width=0.95\columnwidth]{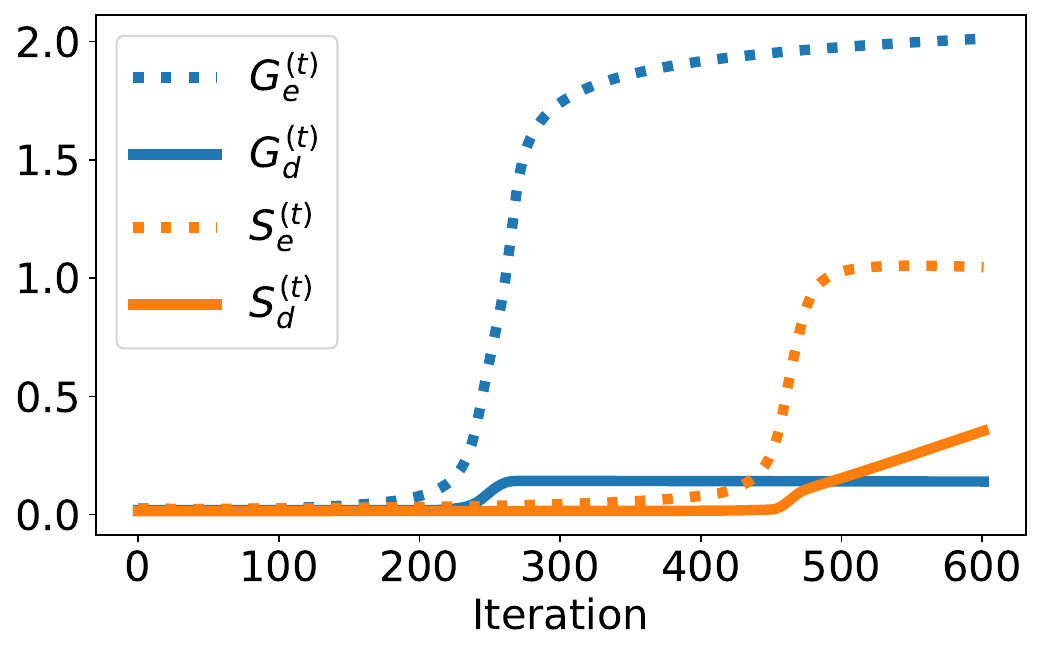}
        \caption{$\beta_d = 0.2$}
        \label{fig:toy_data_original}
    \end{subfigure}
    \hfill
    \begin{subfigure}{0.3\textwidth}
        \centering
        \includegraphics[width=0.95\columnwidth]{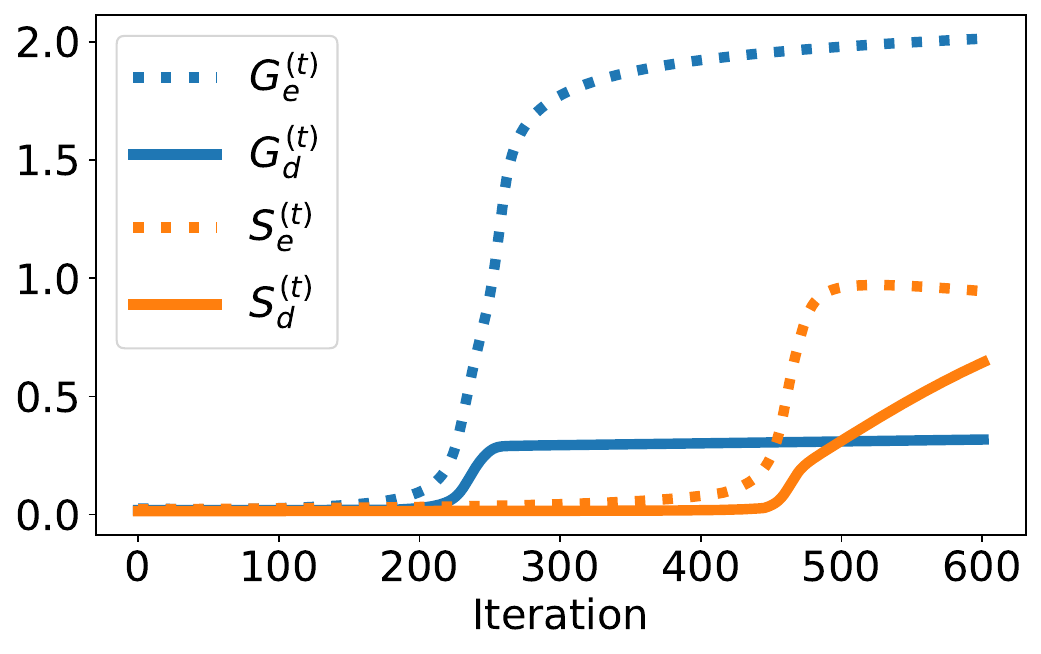}
        \caption{$\beta_d = 0.4$}
        \label{fig:toy_data_upsampled}
    \end{subfigure}
    \hfill
    \begin{subfigure}{0.3\textwidth}
        \centering
        \includegraphics[width=0.95\columnwidth]{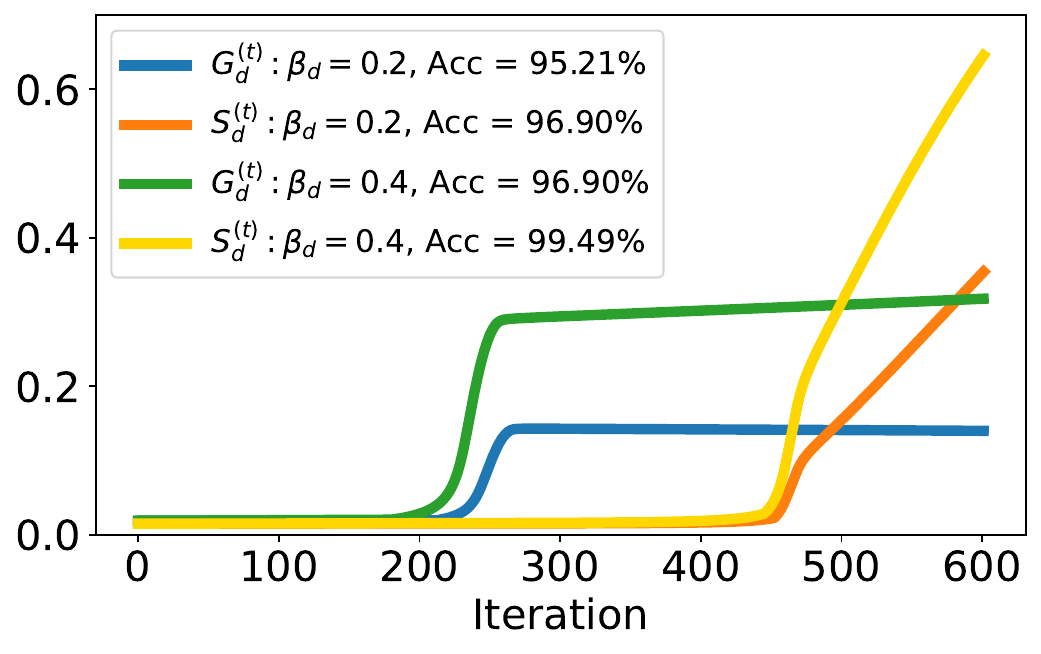}
        \caption{$\beta_d = 0.2 \; \text{vs.} \; \beta_d = 0.4$}
        \label{fig:toy_data_vary_bd}
    \end{subfigure}
    \hfill
    \caption{\textbf{GD (blue) vs. SAM (orange) on toy datasets.} Data is generated based on Definition~\ref{def:data_distribution} with different $\beta_d$ and fixed $\beta_e = 1,$ $\alpha = 0.9$. %
    $\cdot\cdot$ and $--$ lines denote the alignment (i.e., inner product) of \fast ($\vct{v}_e$) and \slow ($\vct{v}_d$) features with the model weight ($\vct{w}_j^{(t)}$). %
    (a), (b) GD and SAM first learn the \fast feature. Notably, GD learns the \fast feature very early.
    (c) %
    Test accuracy of GD \& SAM improves by increasing the strength of the \slow feature.}
    \label{fig:toy_exps}
\end{figure*}

\textbf{When to separate the examples.}
It is crucial to separate examples \textit{early} in training, to accurately identify examples that contribute the most to simplicity bias. %
We empirically verify the intuition that the optimal epoch $t$ to separate examples is when the change in training error starts to shrink as visualized in Figure~\ref{fig:training_error}. {More details can be found in Appendices~\ref{app:implementation_details} and~\ref{app:ablation_studies}.}

The pseudocode of \ourmethod~is illustrated in Alg. \ref{alg:useful} and the workflow is shown in Appendix Fig. ~\ref{fig:method_overview}.

\vspace{-1mm}\section{Experiments}\label{sec:experiments}
\vspace{-1mm}

\textbf{Outline.} %
In Sec.~\ref{subsec:toy_datasets}, we empirically validate our theoretical results on toy datasets. We then evaluate the performance of \ourmethod~on several real-world datasets in Sec.~\ref{subsec:vary_datasets} and different model architectures in Sec.~\ref{subsec:vary_architectures}. 
In addition, Sec.~\ref{subsec:random_upsampling} highlights the advantages of \ourmethod~over random upsampling.
Furthermore, we show that \ourmethod~shares several properties with SAM in Sec.~\ref{subsec:close_to_sam}.
Additional experimental results are deferred to Appendix~\ref{app:additional_results} where we show that \ourmethod~also boosts the performance of other SAM variants, and present promising results for \ourmethod\ applied to the OOD setting (spurious correlation, long-tail distribution), transfer learning, and label noise settings.
We further conduct ablation studies on the effect of our data selection strategy for upsampling, training batch size, learning rate, upsampling factor, and separating epoch in Appendix~\ref{app:ablation_studies}.

\textbf{Settings.} We used common datasets for image classification including CIFAR10, CIFAR100~\cite{krizhevsky2009learning}, {STL10~\cite{coates2011analysis}, CINIC10~\cite{darlow2018cinic}, and Tiny-ImageNet~\cite{le2015tiny}}. {Both CINIC10 and Tiny ImageNet are large-scale datasets containing images from the ImageNet dataset~\cite{deng2009imagenet}.} We trained ResNet18 on all datasets except for CIFAR100 on which we trained ResNet34. We closely followed the setting from~\cite{andriushchenko2022towards} in which our models are trained for 200 epochs with a batch size of 128. We used SGD with the momentum parameter of 0.9 and set weight decay to 0.0005. We also fixed $\rho = 0.1$ for SAM in all experiments unless explicitly stated. We used a linear learning rate schedule starting at 0.1 and decay by a factor of 10 once at epoch 100 and again at epoch 150. {More details are given in Appendix~\ref{app:additional_settings}.} \looseness=-1

\subsection{Toy Datasets}\label{subsec:toy_datasets}\vspace{-1mm}

\textbf{Datasets.} Following~\cite{deng2023robust}, our toy dataset consists of training and test sets, each containing 10K examples generated from the data distribution defined in~\ref{def:data_distribution} with dimension $d \!=\! 50$ and $P \!=\! 3$. %
We set $\beta_e \!=\! 1, \beta_d \!=\! 0.2, \alpha \!=\! 0.9$, and $\sigma_p / \sqrt{d} \!=\! 0.125$. %
We also consider a scenario with larger $\beta_d\!=\!0.4$. We shuffle the order of patches randomly to confirm that our %
theory holds with arbitrary order of patches.\looseness=-1

\textbf{Training.} We used the two-layer nonlinear CNN in Section~\ref{subsec:CNN} with  $J = 40$ filters. For %
GD, we set the learning rate to $\eta = 0.1$ and did not use momentum. For SAM, we used the same base GD optimizer and chose a smaller value of the inner step, $\rho = 0.02$, %
than other experiments to satisfy the constraint in Theorem~\ref{the:sam_learn_more_uniform_than_gd_main}. We trained the model for 600 iterations till convergence for GD and SAM. %

\begin{figure*}[t!]
    \centering
    \includegraphics[width=0.9\textwidth]{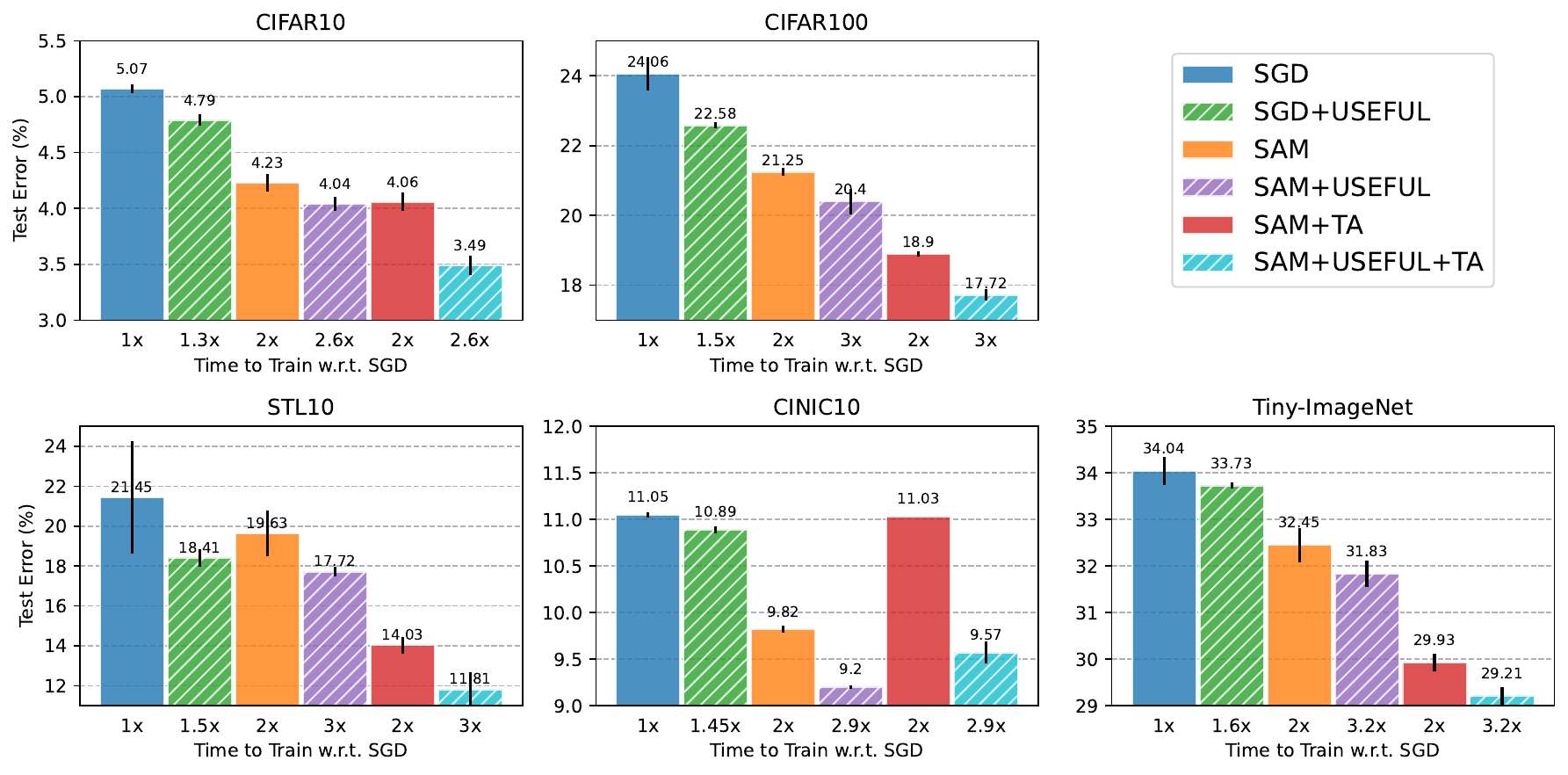} 
    \caption{\textbf{Test classification error of ResNet18 on CIFAR10, STL10, TinyImageNet and ResNet34 on CIFAR100.%
    } The numbers below bars indicate the approximate training cost %
    and the tick on top shows the std over three runs. \ourmethod~enhances the performance of SGD and SAM on all 5 datasets. TrivialAugment (TA) further boosts SAM's performance (except for CINIC10). Remarkably, \ourmethod~consistently boosts the performance across all scenarios and achieves (to our knowledge) SOTA performance for ResNet18 and ResNet34 on the selected datasets when combined with SAM and TA. \looseness=-1 }
    \label{fig:vary_datasets}
\end{figure*}

\begin{figure*}[t!]
    \centering
    \includegraphics[width=1.0\textwidth]{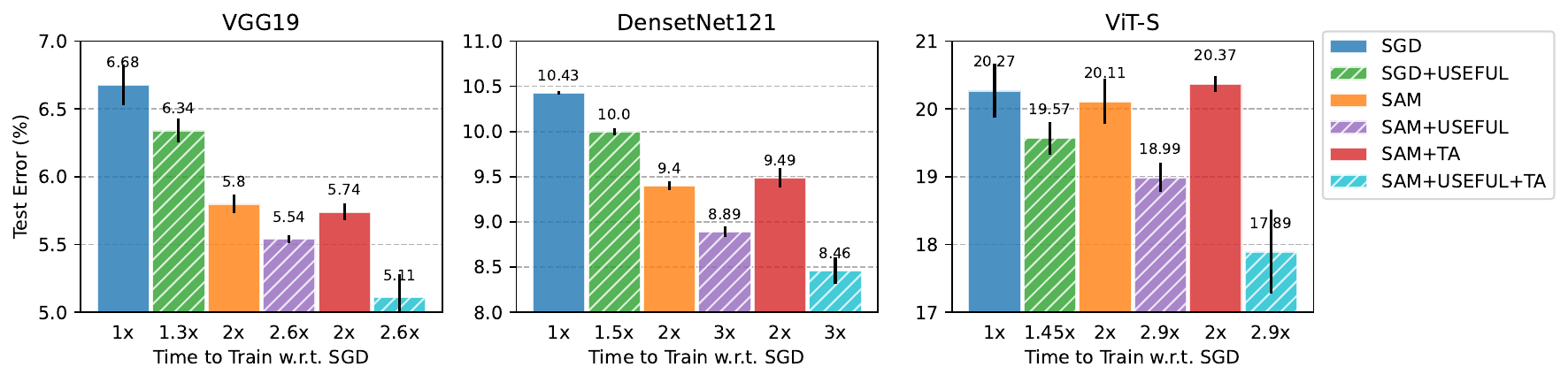} 
    \caption{\textbf{{Test classification errors of different architectures on CIFAR10.}} 
    \ourmethod~improves the performance of SGD and SAM when training different architectures. TrivialAugment (TA) further boosts SAM's capabilities. The results for 3-layer MLP can be found in Figure~\ref{fig:vary_architectures_app}.}
    \label{fig:vary_architectures}
\end{figure*}

\textbf{Results.} Figure~\ref{fig:toy_data_original} illustrates that 
both GD (blue) and SAM (orange) first learn the \fast feature. 
In particular, the blue dotted line ($G_e$) accelerates quickly at around epoch 250 while the orange dotted line ($S_e$) increases drastically much later at around epoch 450. 
That is: \textbf{\textit{(1) GD learns the \fast feature very early in training.}}
This is well-aligned with our Theorems~\ref{the:easy_difficult_gd_main} and~\ref{the:easy_difficult_sam_main} and their discussion. 
Furthermore, the gap between contribution of \fast and \slow features towards the model output in SAM $(S_e^{(t)} - S_d^{(t)})$ is much smaller than that of GD $(G_e^{(t)} - G_d^{(t)})$. That is: \textbf{\textit{(2) \fast and \slow features are learned more evenly in SAM.}} This validates our %
Theorem~\ref{the:sam_learn_more_uniform_than_gd_main}. From around epoch 500 onwards, the contribution of the \slow feature in SAM surpasses the level of that in GD while the contribution of the \fast feature in SAM is still lower than the counterpart in GD. 
When increasing the \slow feature strength $\beta_d$ from 0.2 to 0.4 in Figure~\ref{fig:toy_data_upsampled}, the same conclusion for the growth speed of \fast and \slow features holds.
Notably, there is a clear increase in the classification accuracy of the model trained with either GD or SAM by increasing $\beta_d$, as can be seen in Figure~\ref{fig:toy_data_vary_bd}. 
That is: \textbf{\textit{(3) amplifying the strength of the \slow feature improves the generalization performance.}}
Effectively, this enables the model %
{successfully predict examples} in which the \fast feature is missing.\looseness=-1

\subsection{\ourmethod\ is Effective across Datasets}\label{subsec:vary_datasets}\vspace{-1mm}
Figure~\ref{fig:vary_datasets} illustrates the performance of models trained with SGD and SAM on original vs modified data distribution by \ourmethod. We see that \ourmethod~effectively reduces the test classification error of both SGD and SAM. %
Interestingly, \ourmethod\ further improves SAM's generalization performance by reducing its simplicity bias.
Notably, on the STL10 dataset, \ourmethod~boosts the performance of SGD to surpass that of SAM. The percentages of examples found for upsampling by \ourmethod~for CIFAR10, CIFAR100, STL10, CINIC10, and Tiny-ImageNet are roughly 30\%, 50\%, 50\%, 45\%, and 60\%, respectively. Thus, training SGD on the modified data distribution only incurs a cost of 1.3x, 1.5x, 1.5x, 1.45x, and 1.6x compared to 2x of SAM. \looseness=-1

\textbf{\ourmethod+TA is particularly effective. } 
Stacking strong augmentation methods e.g. TrivialAugment~\cite{muller2021trivialaugment} further improves the performance, achieving state-of-the-art for ResNet on all datasets. 
When strong augmentation is combined with USEFUL, it makes more variations of the (upsampled) \slow features and enhances their learning. Hence, it further boost the performance. \looseness=-1

\subsection{\ourmethod\ is Effective across Architectures \& Settings}\label{subsec:vary_architectures}
\textbf{Model architectures: CNN, ViT, MLP.} Next, we confirm the versatility of our method, by applying it to different model architectures including 3-layer MLP, CNNs (ResNet18, VGG19, DenseNet121), and Transformers %
(ViT-S). Figure~\ref{fig:vary_architectures} shows that \ourmethod~is effective across different model architectures. Remarkably, when applying to non-CNN architectures, it reduces the test error of SGD to a lower level than that of SAM alone. Detailed results for 3-layer MLP is given in Appendix~\ref{app:mlp}.\looseness=-1

\textbf{Settings: batch-size, learning rate, and SAM variants.}
In Appendix~\ref{app:additional_results}, we confirm the effectiveness of \ourmethod\ for different batch sizes of 128, 256, 512, and different initial learning rates of 0.1, 0.2, 0.4. %
In Appendix~\ref{app:sam_variants}, we confirm that \ourmethod~applied to ASAM~\cite{kwon2021asam}---a SAM variant which uses a scale-invariant sharpness measure---further reduces the test error. %

\begin{figure*}[t!]
    \centering
    \includegraphics[width=0.63\textwidth]{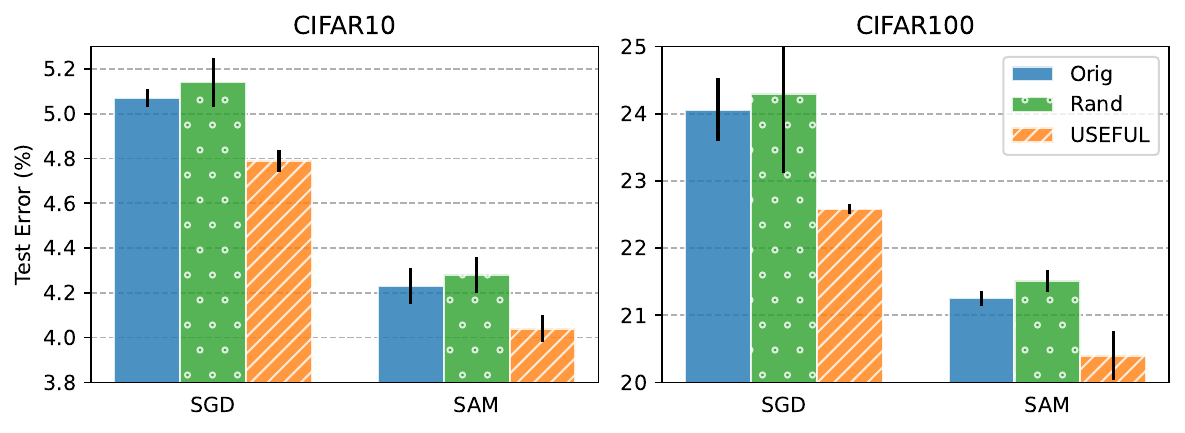}
    \caption{\ourmethod~vs. Random Upsampling, when training ResNet18 %
    on CIFAR10 and CIFAR100.}
    \label{fig:comparing_upsampling_strategies}
\end{figure*}

\textbf{\ourmethod\ vs Random Upsampling.}\label{subsec:random_upsampling}
Fig. \ref{fig:comparing_upsampling_strategies} shows that \ourmethod\ considerably outperforms SGD and SAM on randomly upsampled %
CIFAR10 \& CIFAR100.
This confirms that the main benefit of \ourmethod\ is due to the modified distribution and not longer training time.
In Appendix~\ref{app:ablation_studies}, we also confirm that upsampling outperforms upweighting for SAM \& SGD.

\subsection{\ourmethod's Solution has Similar Properties to SAM%
}\label{subsec:close_to_sam}%

\textbf{SAM \& \ourmethod~Find Sparser Solutions than SGD.}
\citep%
{andriushchenko2022towards} showed that SAM’s solution has a better sparsity-inducing property indicated by the L1 norm than the standard ERM. 
Fig. \ref{fig:l1_scatter} shows the L1 norm of ResNet18 trained on CIFAR10 and ResNet34 trained on CIFAR100 at the end of training. %
We see that \ourmethod~drives both SGD and SAM to find solutions with smaller L1 norms. %

\textbf{SAM \& \ourmethod~Find Less Sharp Solutions than SGD.} {While our goal is not to directly find a flatter minimum or the same solution as SAM, we showed that USEFUL finds flatter minima. Following~\cite{andriushchenko2022towards}, we used} the maximum Hessian eigenvalue 
($\lambda_{max}$) and the bulk of the spectrum ($\lambda_{max}/\lambda_5$)~\cite{jastrzebski2020break}, which are commonly used metrics for sharpness~\cite{keskar2016large,jastrzkebski2017three,chaudhari2019entropy,wen2019empirical}. Table~\ref{tab:compare_sharpness} illustrates that SGD+\ourmethod~on CIFAR10 reduces sharpness metrics significantly compared to SGD, proving that \ourmethod~successfully reduces the sharpness of the solution.
{We note that to capture the sharpness/flatness, multiple different criteria have been proposed (largest Hessian eigenvalue and bulk of Hessian), and one criterion is not enough to accurately capture the sharpness. While the solution of SGD+\ourmethod~has a higher largest Hessian eigenvalue than SAM, it achieves the smallest bulk.}

\textbf{SAM \& \ourmethod~Reduce Forgetting Scores.} 
Forgetting scores ~\cite{toneva2018empirical} count the number of times an example is misclassified after being correctly classified during training and is an indicator of the learning speed and difficulty of examples.
We show in Appendix~\ref{app:close_to_sam} that both SAM and \ourmethod~successfully reduce the forgetting scores, thus learn \slow features faster than SGD.
This aligns with our %
Theorem~\ref{the:sam_learn_more_uniform_than_gd_main} and results on the toy datasets. By upsampling \slow examples in the dataset, they contribute more to learning and hence SGD+\ourmethod~learns them faster than SGD.\looseness=-1

\textbf{\ourmethod\ also Benefits Distribution Shift.}
While our main contribution is providing a novel and effective method to improve the in-distribution generalization performance, we conduct experiments confirming the benefits of our method to %
distribution shift. We discuss this experiment and its results in Appendix~\ref{app:ood_experiments}. %
On Waterbirds dataset~\cite{sagawa2019distributionally} with strong spurious correlation (95\%), both SAM and \ourmethod~successfully improve the performance on the balanced test set by 6.21\% and 5.8\%, respectively. We also show the applicability of \ourmethod\ to fine-tuning a ResNet50 pre-trained on ImageNet. \looseness=-1

\vspace{-2mm}
\section{Conclusion}\label{sec:conclusion}\vspace{-2mm}
In this paper, we made the first attempt to %
improve the in-distribution generalization performance of machine learning methods by modifying the distribution of training data. We first analyzed learning dynamics of sharpness-aware minimization (SAM), and attributed its superior performance over GD to mitigating the simplicity bias, %
and learning features at a more speed. %
Inspired by SAM, we upsampled the {examples that contain \slow features} %
to alleviate the %
simplicity bias. This allows learning features more uniformly, thus improving the performance. Our method boosts the performance of image classifiers trained with SGD or SAM and easily stacks with data augmentation. \looseness=-1

\section*{Acknowledgments}
This research was partially supported by the National Science Foundation CAREER Award 2146492, National Science Foundation 2421782 and Simons Foundation, Cisco Systems, Optum AI, and a UCLA Hellman Fellowship.

\bibliography{refs}
\bibliographystyle{plainnat}

\newpage
\appendix
\section{Formal Proofs}
\subsection{Proof of Theorem~\ref{the:easy_difficult_gd_main}}\label{app:proof_gd}

\textbf{Notation.} In this paper, we use lowercase letters, lowercase boldface letters, and uppercase boldface letters to respectively denote scalars $(a)$, vectors $(\vct{v})$, and matrices $(\mtx{W})$. For a vector $\vct{v}$, we use $\norm{\vct{v}}_2$ to denote its Euclidean norm. Given two sequence $\{ x_n \}$ and $\{ y_n \}$, we denote $x_n = O(y_n)$ if $|x_n| \leq C_1 |y_n|$ for some absolute positive constant $C_1$, $x_n = \Omega(y_n)$ if $|x_n| \geq C_2 |y_n|$ for some absolute positive constant $C_2$, and $x_n = \Theta(y_n)$ if $C_3 |y_n| \leq |x_n| \leq C_4 |y_n|$ for some absolute constant $C_3, C_4 > 0$. In addition, we use $\Tilde{O}(\cdot), \Tilde{\Omega}(\cdot),$ and $\Tilde{\Theta}(\cdot)$ to hide logarithmic factors in these notations. Furthermore, we denote $x_n = \text{poly}(y_n)$ if $x_n = O(y_n^D)$ for some positive constant D, and $x_n = \text{polylog}(y_n)$ if $x_n = \text{poly}(\log(y_n))$.

First, we have the following assumption for the model weight initialization.

\begin{assumption}[Weight initialization]\label{ass:weight_initialization}
     Assume that we initialize $\mtx{W}^{(0)} \sim \gN(0, \sigma_0^2)$ such that for all $j \in [J]$, $\inner{\vct{w}_j^{(0)}}{\vct{v}_e}, \inner{\vct{w}_j^{(0)}}{\vct{v}_d} \geq \rho > 0$. 
\end{assumption}

The above assumption is reasonable because we later show that both sequences $\inner{\vct{w}_j^{(t)}}{\vct{v}_e}$ and $\inner{\vct{w}_j^{(t)}}{\vct{v}_d}$ are non-decreasing. So, we can obtain the above initialization by training the model for several iterations. For simplicity of the notation, we assume that $\alpha N$ is an integer and the first $\alpha N$ data examples have the \fast feature while the rest do not. Before going into the analysis, we denote the derivative of a data example $i$ at iteration $t$ to be
\begin{equation}
    l_i^{(t)} = \frac{\exp(-y_i f(\vct{x}_i; \mtx{W}^{(t)}))}{1 + \exp(-y_i f(\vct{x}_i; \mtx{W}^{(t)}))} = \text{sigmoid}(-y_i f(\vct{x}_i; \mtx{W}^{(t)})).
\end{equation}

\begin{lemma}[Gradient]\label{lem:gradient_gd}
    Let the loss function $\gL$ be as defined in Equation~\ref{eq:erm}. For $t \geq 0$ and $j \in [J]$, the gradient of the loss $\gL(\mtx{W}^{(t)})$ with regard to neuron $\vct{w}_j^{(t)}$ is
    \begin{align}
        \nabla_{\vct{w}_j^{(t)}} \gL(\mtx{W}^{(t)}) = - &\frac{3}{N} \sum_{i=1}^{\alpha N} l_i^{(t)} \left( \beta_d^3 \inner{\vct{w}_j^{(t)}}{\vct{v}_d}^2 \vct{v}_d + \beta_e^3 \inner{\vct{w}_j^{(t)}}{\vct{v}_e}^2 \vct{v}_e + y_i \inner{\vct{w}_j^{(t)}}{\vct{\xi}_i}^2 \vct{\xi}_i \right) -\nonumber \\ 
        &\frac{3}{N} \sum_{i=\alpha N + 1}^N l_i^{(t)} \left( \beta_d^3 \inner{\vct{w}_j^{(t)}}{\vct{v}_d}^2 \vct{v}_d + y_i \inner{\vct{w}_j^{(t)}}{\vct{\xi}_i}^2 \vct{\xi}_i \right).
    \end{align}
\end{lemma}

\begin{proof}
    We have the following gradient
    \begin{align*}
        \nabla_{\vct{w}_j^{(t)}} \gL(\mtx{W}^{(t)}) = -&\frac{1}{N} \sum_{i=1}^N \frac{\exp(-y_i f(\vct{x}_i; \mtx{W}^{(t)}))}{1 + \exp(-y_i f(\vct{x}_i; \mtx{W}^{(t)}))} \cdot y_i f'(\vct{x}_i; \mtx{W}^{(t)}) \\
        = -&\frac{3}{N} \sum_{i=1}^N l_i^{(t)} y_i \sum_{p=1}^P \inner{\vct{w}_j^{(t)}}{\vct{x}^{(p)}}^2 \cdot \vct{x}^{(p)} \\
        = -&\frac{3}{N} \sum_{i=1}^N l_i^{(t)} \left( \beta_d^3 \inner{\vct{w}_j^{(t)}}{\vct{v}_d}^2 \vct{v}_d + \beta_e^3 \inner{\vct{w}_j^{(t)}}{\vct{v}_e}^2 \vct{v}_e + y_i \inner{\vct{w}_j^{(t)}}{\vct{\xi}_i}^2 \vct{\xi}_i \right) - \\
        &\frac{3}{N} \sum_{i=\alpha N + 1}^N l_i^{(t)} \left( \beta_d^3 \inner{\vct{w}_j^{(t)}}{\vct{v}_d}^2 \vct{v}_d + y_i \inner{\vct{w}_j^{(t)}}{\vct{\xi}_i}^2 \vct{\xi}_i \right)
    \end{align*}
\end{proof}

With the above formula of gradient, we have the following equations:

\textbf{\Fast feature gradient.} The projection of the gradient on $\vct{v}_e$ is
\begin{equation}\label{eq:easy_feature_gradient_gd}
    \inner{\nabla_{\vct{w}_j^{(t)}} \gL(\mtx{W}^{(t)})}{\vct{v}_e} = -\frac{3\beta_e^3}{N} \sum_{i=1}^{\alpha N} l_i^{(t)} \inner{\vct{w}_j^{(t)}}{\vct{v}_e}^2
\end{equation}

\textbf{\Slow feature gradient.} The projection of the gradient on $\vct{v}_d$ is
\begin{equation}\label{eq:difficult_feature_gradient_gd}
    \inner{\nabla_{\vct{w}_j^{(t)}} \gL(\mtx{W}^{(t)})}{\vct{v}_d} = -\frac{3\beta_d^3}{N} \sum_{i=1}^N l_i^{(t)} \inner{\vct{w}_j^{(t)}}{\vct{v}_d}^2
\end{equation}

\textbf{Noise gradient.} The projection of the gradient on $\vct{\xi}_i$ is
\begin{equation}\label{eq:noise_gradient_gd}
    \inner{\nabla_{\vct{w}_j^{(t)}} \gL(\mtx{W}^{(t)})}{\vct{\xi}_i} = -\frac{3}{N} \left( l_i^{(t)} y_i \inner{\vct{w}_j^{(t)}}{\vct{\xi}_i}^2 \norm{\vct{\xi}_i}_2^2 + \sum_{k=1, k \neq i}^N l_k^{(t)} y_k \inner{\vct{w}_j^{(t)}}{\vct{\xi}_k}^2 \inner{\vct{\xi}_k}{\vct{\xi}_i} \right)
\end{equation}

\textbf{Derivative of data example $i$.} For $1 \leq i \leq \alpha N$, $l_i^{(t)}$ can be rewritten as
\begin{equation}\label{eq:deriv_data_gd}
    l_i^{(t)} = \text{sigmoid} \left( \sum_{j=1}^J -\beta_d^3 \inner{\vct{w}_j^{(t)}}{\vct{v}_d}^3 - \beta_e^3 \inner{\vct{w}_j^{(t)}}{\vct{v}_e}^3 - y_i \inner{\vct{w}_j^{(t)}}{\vct{\xi}_i}^3 \right)
\end{equation}
while for $\alpha N + 1 \leq i \leq N$, $l_i^{(t)}$ can be rewritten as
\begin{align}
    l_i^{(t)} &= \text{sigmoid} \left( \sum_{j=1}^J -\beta_d^3 \inner{\vct{w}_j^{(t)}}{\vct{v}_d}^3 - y_i \inner{\vct{w}_j^{(t)}}{\vct{\xi}_i}^3 \right) \nonumber \\
    &\geq \text{sigmoid} \left( \sum_{j=1}^J -\beta_d^3 \inner{\vct{w}_j^{(t)}}{\vct{v}_d}^3 - \beta_e^3 \inner{\vct{w}_j^{(t)}}{\vct{v}_e}^3 - y_i \inner{\vct{w}_j^{(t)}}{\vct{\xi}_i}^3 \right)
\end{align}

Note that $0 < l_i^{(t)} < 1$ due to the property of the sigmoid function. Furthermore, we similarly consider that the sum of the sigmoid terms for all time steps is bounded up to a logarithmic dependence~\cite{chen2022towards}. The sigmoid term is considered small for a $\kappa$ such that
\begin{equation}
    \sum_{t=0}^T \frac{1}{1 + \exp(\kappa)} \leq \tilde{O}(1),
\end{equation}
which implies $\kappa \geq \tilde{\Omega}(1)$.

We present the detailed proofs that build up to Theorem~\ref{the:easy_difficult_gd}. We begin by considering the update for the \fast and \slow features.

\begin{lemma}[\Fast feature update.]\label{lem:easy_feature_update_gd}
    For all $t \geq 0$ and $j \in [J]$, the \fast feature update is 
    \begin{equation}
        \inner{\vct{w}_j^{(t+1)}}{\vct{v}_e} = \inner{\vct{w}_j^{(t)}}{\vct{v}_e} + \tilde{\Theta}(\eta) \alpha \beta_e^3 g_1(t) \inner{\vct{w}_j^{(t)}}{\vct{v}_e}^2,
    \end{equation}
    where $g_1(t) = \text{sigmoid} \left( \sum_{j=1}^J -\beta_d^3 \inner{\vct{w}_j^{(t)}}{\vct{v}_d}^3 - \beta_e^3 \inner{\vct{w}_j^{(t)}}{\vct{v}_e}^3 \right)$.
\end{lemma}

\begin{proof}
    Plugging the update rule of GD, we have
    \begin{align*}
        \inner{\vct{w}_j^{(t+1)}}{\vct{v}_e} &= \inner{\vct{w}_j^{(t)} - \eta \nabla_{\vct{w}_j^{(t)}} \gL(\mtx{W}^{(t)})}{\vct{v}_e} \\
        &= \inner{\vct{w}_j^{(t)}}{\vct{v}_e} + \frac{3 \eta \beta_e^3}{N} \sum_{i=1}^{\alpha N} l_i^{(t)} \inner{\vct{w}_j^{(t)}}{\vct{v}_e}^2 \\
        &= \inner{\vct{w}_j^{(t)}}{\vct{v}_e} + \tilde{\Theta}(\eta) \alpha \beta_e^3 g_1(t) \inner{\vct{w}_j^{(t)}}{\vct{v}_e}^2,
    \end{align*}
    where the last equality holds due to Lemma~\ref{lem:approx_deriv_data_gd}.
\end{proof}

Similarly, we obtain the following update rule for \slow features.

\begin{lemma}[\Slow feature update.] For all $t \geq 0$ and $j \in [J]$, the \fast feature update is 
    \begin{equation}
        \inner{\vct{w}_j^{(t+1)}}{\vct{v}_d} = \inner{\vct{w}_j^{(t)}}{\vct{v}_d} + \frac{3 \eta \beta_d^3}{N} \sum_{i=1}^N l_i^{(t)} \inner{\vct{w}_j^{(t)}}{\vct{v}_d}^2,
    \end{equation}
    which gives
    \begin{align}
        \tilde{\Theta}(\eta) \beta_d^3 g_1(t) \inner{\vct{w}_j^{(t)}}{\vct{v}_d}^2 \leq \inner{\vct{w}_j^{(t+1)}}{\vct{v}_d} - \inner{\vct{w}_j^{(t)}}{\vct{v}_d} \leq \tilde{\Theta}(\eta) \beta_d^3 (\alpha g_1(t) + 1 - \alpha) \inner{\vct{w}_j^{(t)}}{\vct{v}_d}^2
    \end{align}
    where $g_1(t) = \text{sigmoid} \left( \sum_{j=1}^J -\beta_d^3 \inner{\vct{w}_j^{(t)}}{\vct{v}_d}^3 - \beta_e^3 \inner{\vct{w}_j^{(t)}}{\vct{v}_e}^3 \right)$.
\end{lemma}

\begin{proof}
    Plugging the update rule of GD, we have
    \begin{align*}
        \inner{\vct{w}_j^{(t+1)}}{\vct{v}_d} &= \inner{\vct{w}_j^{(t)} - \eta \nabla_{\vct{w}_j^{(t)}} \gL(\mtx{W}^{(t)})}{\vct{v}_d} \\
        &= \inner{\vct{w}_j^{(t)}}{\vct{v}_d} + \frac{3 \eta \beta_d^3}{N} \sum_{i=1}^N l_i^{(t)} \inner{\vct{w}_j^{(t)}}{\vct{v}_d}^2
    \end{align*}
    From Lemma~\ref{lem:approx_deriv_data_gd}, we have for $1 \leq i \leq \alpha N, l_i^{(t)} = \Theta(1) g_1(t)$ and for $\alpha N + 1 \leq i \leq N, \Theta(1) g_1(t) \leq l_i^{(t)} \leq 1$. Combining with the above equality, we obtain the desired inequalities.
\end{proof}

Next, we simplify the two above update rules in the early training stage.

\begin{lemma}[\Fast feature update in early iterations]\label{lem:easy_feature_update_gd_simplified}
    Let $T_0 > 0$ be such that $\max_{j \in [J]} \inner{\vct{w}_j^{(T_0)}}{\vct{v}_e} \geq \tilde{\Omega}(1/\beta_e)$. For $t \in [0, T_0]$, the \fast feature update has the following rule
    \begin{equation}
        \inner{\vct{w}_j^{(t+1)}}{\vct{v}_e} = \inner{\vct{w}_j^{(t)}}{\vct{v}_e} + \tilde{\Theta}(\eta) \alpha \beta_e^3 \inner{\vct{w}_j^{(t)}}{\vct{v}_e}^2,
    \end{equation}
\end{lemma}

\begin{proof}
    Let $T_0 > 0$ be such that either $\max_{j \in [J]} \inner{\vct{w}_j^{(T_0)}}{\vct{v}_e} \geq \tilde{\Omega}(1/\beta_e)$ or $\max_{j \in [J]} \inner{\vct{w}_j^{(T_0)}}{\vct{v}_d} \geq \tilde{\Omega}(1/\beta_d)$. We will show later that the first condition will be met and we have $\max_{j \in [J]} \inner{\vct{w}_j^{(T_0)}}{\vct{v}_d} \leq \tilde{\Omega}(1/\beta_d)$ for all $j \in [J]$ and $t \in [0, T_0]$.

    Recall that $g_1(t) = \text{sigmoid} \left(\sum_{j=1}^J -\beta_d^3 \inner{\vct{w}_j^{(t)}}{\vct{v}_d}^3 - \beta_e^3 \inner{\vct{w}_j^{(t)}}{\vct{v}_e}^3 \right)$. Then, for $t \in [0, T_0]$, we have
    \begin{align*}
        g_1(t) &= \frac{1}{1 + \exp(\sum_{j=1}^J -\beta_d^3 \inner{\vct{w}_j^{(t)}}{\vct{v}_d}^3 - \beta_e^3 \inner{\vct{w}_j^{(t)}}{\vct{v}_e}^3)} \\
        &\geq \frac{1}{1 + \exp(\kappa + \kappa)} \\
        &= \frac{1}{1 + \exp(\tilde{\Omega}(1))},
    \end{align*}
    where the first inequality holds due to $\inner{\vct{w}_j^{(t)}}{\vct{v}_e} \leq \kappa/(J^{1/3} \beta_e)$ and $\inner{\vct{w}_j^{(t)}}{\vct{v}_d} \leq \kappa/(J^{1/3} \beta_d)$ for $t \in [0, T_0]$~\cite{deng2023robust}[Lemma E.3]. Therefore, similar to~\cite{deng2023robust,jelassi2022towards}, we have $g_1(t) = \Theta(1)$ in the early iterations. This implies the result in Lemma~\ref{lem:easy_feature_update_gd} as
    \begin{equation}
        \inner{\vct{w}_j^{(t+1)}}{\vct{v}_e} = \inner{\vct{w}_j^{(t)}}{\vct{v}_e} + \tilde{\Theta}(\eta) \alpha \beta_e^3 \inner{\vct{w}_j^{(t)}}{\vct{v}_e}^2.
    \end{equation}
\end{proof}

Similarly, we obtain the following simplified update rule for \slow features in the early iterations.

\begin{lemma}[\Slow feature update in early iterations]\label{lem:difficult_feature_update_gd_simplified}
    Let $T_0 > 0$ be such that $\max_{j \in [J]} \inner{\vct{w}_j^{(T_0)}}{\vct{v}_e} \geq \tilde{\Omega}(1/\beta_e)$. For $t \in [0, T_0]$, the \fast feature update has the following rule
    \begin{equation}
        \inner{\vct{w}_j^{(t+1)}}{\vct{v}_d} = \inner{\vct{w}_j^{(t)}}{\vct{v}_d} + \tilde{\Theta}(\eta) \beta_d^3 \inner{\vct{w}_j^{(t)}}{\vct{v}_d}^2,
    \end{equation}
\end{lemma}

We next show that GD will learn the \fast feature quicker than learning the \slow feature.

\begin{lemma}\label{lem:gd_learn_easy_quicker}
    Assume $\eta = \tilde{o}(\beta_d \sigma_0)$. Let $T_0$ be the iteration number that $\max_{j \in [J]} \inner{\vct{w}_j^{(T_0)}}{\vct{v}_e}$ reaches $\tilde{\Omega}(1/\beta_e) = \tilde{\Theta(1)}$. Then, we have for all $t \leq T_0$, it holds that $\max_{j \in [J]} \inner{\vct{w}_j^{(T_0)}}{\vct{v}_d} = \tilde{O}(\sigma_0)$.
\end{lemma}

\begin{proof}
    Among all the possible indices $j \in [J]$ , we focus on the index $j^\star = \argmax_{j \in [J]} \inner{\vct{w}_j^{(0)}}{\vct{v}_e}$. 
    Therefore, for $C_t = \alpha \beta_e^3 = \Theta(1)$, we apply Lemma~\ref{lem:grow_speed_two_sequences} with two positive sequences $\inner{\vct{w}_{j^\star}^{(t)}}{\vct{v}_e}$ and $\inner{\vct{w}_j^{(t)}}{\vct{v}_d}$
    defined in Lemmas~\ref{lem:easy_feature_update_gd_simplified} and~\ref{lem:difficult_feature_update_gd_simplified} and get
    \begin{equation}
        \inner{\vct{w}_j^{(t)}}{\vct{v}_d} \leq O(\inner{\vct{w}_j^{(0)}}{\vct{v}_d}) = \tilde{O}(\sigma_0)
    \end{equation}
    for all $j \in [J]$.
\end{proof}

\begin{theorem}[Restatement of Theorem~\ref{the:easy_difficult_gd_main}]\label{the:easy_difficult_gd}
    We consider training a two-layer nonlinear CNN model initialized with $\mtx{W}^{(0)} \sim \gN(0, \sigma_0^2)$ on the training dataset $D = \{ (\vct{x}_i, y_i) \}_{i=1}^N$ that follows the data distribution $\gD(\beta_e, \beta_d, \alpha)$ with $\alpha^{1/3} \beta_e > \beta_d$. After training with GD in Equation~\ref{eq:gd_update} for $T_\text{GD}$ iterations where
    \begin{equation}
        T_\text{GD} = \frac{\tilde{\Theta}(1)}{\eta \alpha \beta_e^3 \sigma_0} + \tilde{\Theta}(1)\Big\lceil \frac{-\log(\sigma_0 \beta_e)}{\log(2)} \Big\rceil, \label{eq:learning_time_gd}
    \end{equation}
    for all $j \in [J]$ and $t \in [0, T_\text{GD})$, we have
    \begin{align}
        \inner{\vct{w}_j^{(t+1)}}{\vct{v}_e} &= \inner{\vct{w}_j^{(t)}}{\vct{v}_e} + \tilde{\Theta}(\eta) \alpha \beta_e^3 \inner{\vct{w}_j^{(t)}}{\vct{v}_e}^2. \\
        \inner{\vct{w}_j^{(t+1)}}{\vct{v}_d} &= \inner{\vct{w}_j^{(t)}}{\vct{v}_d} + \tilde{\Theta}(\eta) \beta_d^3 \inner{\vct{w}_j^{(t)}}{\vct{v}_d}^2
    \end{align}
    After training for $T_\text{GD}$ iterations, with high probability, the learned weight has the following properties: (1) it learns the \fast feature $\vct{v}_e: \max_{j \in [J]} \inner{\vct{w}_j^{(T_\text{GD})}}{\vct{v}_e} \geq \tilde{\Omega}(1/\beta_e)$; (2) it does not learn the \slow feature $\vct{v}_d: \max_{j \in [J]} \inner{\vct{w}_j^{(T_\text{GD})}}{\vct{v}_d} = \tilde{O}(\sigma_0)$.
\end{theorem}

\begin{proof}
    From the results of Lemmas~\ref{lem:easy_feature_update_gd_simplified}-~\ref{lem:gd_learn_easy_quicker}, it remains to calculate the time $T_\text{GD}$. Plugging $v = \tilde{\Omega}(1/\beta_e), m = M = \tilde{\Theta}(\eta) \alpha \beta_e^3, z_0 = \tilde{O}(\sigma_0)$ into Lemma~\ref{lem:grow_speed_one_sequence_two_equations}, we have $T_\text{GD}$ as 
    \begin{equation}
        T_\text{GD} = \frac{\tilde{\Theta}(1)}{\eta \alpha \beta_e^3 \sigma_0} + \tilde{\Theta}(1)\Big\lceil \frac{-\log(\sigma_0 \beta_e)}{\log(2)} \Big\rceil
    \end{equation}
\end{proof}

\subsection{Proof of Theorem~\ref{the:easy_difficult_sam_main}}\label{app:proof_sam}
Before going into the analysis, we denote the derivative of a data example $i$ at iteration $t$ to be
\begin{equation}
    l_{i, \vct{\epsilon}}^{(t)} = \frac{\exp(-y_i f(\vct{x}_i; \mtx{W}^{(t)} + \mtx{\epsilon}^{(t)})))}{1 + \exp(-y_i f(\vct{x}_i; \mtx{W}^{(t)} + \mtx{\epsilon}^{(t)}))} = \text{sigmoid}(-y_i f(\vct{x}_i; \mtx{W}^{(t)} + \mtx{\epsilon}^{(t)}),
\end{equation}
where $\vct{\epsilon}^{(t)} = \rho^{(t)} \nabla \gL(\mtx{W}^{(t)})$ is the weighted ascent direction at the current parameter $\mtx{W}^{(t)}$. We denote the weight vector of the $j$-th filter after being perturbed by SAM as

\begin{equation}
    \vct{w}_{j, \vct{\epsilon}}^{(t)} = \vct{w}_j^{(t)} + \vct{\epsilon}_j^{(t)} = \vct{w}_j^{(t)} + \rho^{(t)} \nabla_{\vct{w}_j^{(t)}} \gL(\mtx{W}^{(t)}),
\end{equation}

where $\rho^{(t)} = \rho/ \norm{\nabla \gL(\mtx{W}^{(t)})}_F$.

First, we have the following inequalities regarding the gradient norm:
\begin{align}
    \norm{\nabla \gL(\mtx{W}^{(t)})}_F &\geq \norm{\nabla_{\vct{w}_j^{(t)}} \gL(\mtx{W}^{(t)})} \\
    &= \inner{\nabla_{\vct{w}_j^{(t)}} \gL(\mtx{W}^{(t)})}{\nabla_{\vct{w}_j^{(t)}} \gL(\mtx{W}^{(t)})}^{1/2} \\
    &= \Bigg[\left( \frac{3 \beta_d^3}{N} \sum_{i=1}^N l_i^{(t)} \inner{\vct{w}_j^{(t)}}{\vct{v}_d}^2 \right)^2 + \left( \frac{3 \beta_e^3}{N} \sum_{i=1}^{\alpha N} l_i^{(t)} \inner{\vct{w}_j^{(t)}}{\vct{v}_e}^2 \right)^2 \nonumber \\
    &+ \norm{\frac{3}{N} \sum_{i=1}^N l_i^{(t)} y_i \inner{\vct{w}_j^{(t)}}{\vct{\xi}_i}^2 \vct{\xi}_i} \Bigg]^{1/2}
\end{align}
Thus,
\begin{align}
    \norm{\nabla \gL(\mtx{W}^{(t)})}_F &\geq \frac{3 \beta_d^3}{N} \sum_{i=1}^N l_i^{(t)} \inner{\vct{w}_j^{(t)}}{\vct{v}_d}^2 \label{eq:bound_gradient_norm_by_difficult_feature}\\
    \norm{\nabla \gL(\mtx{W}^{(t)})}_F &\geq \frac{3 \beta_e^3}{N} \sum_{i=1}^{\alpha N} l_i^{(t)} \inner{\vct{w}_j^{(t)}}{\vct{v}_e}^2 \label{eq:bound_gradient_norm_by_easy_feature}
\end{align}

\begin{lemma}[Gradient]\label{lem:gradient_sam}
    Let the loss function $\gL$ be as defined in Equation~\ref{eq:erm}. For $t \geq 0$ and $j \in [J]$, the gradient of the loss $\gL(\mtx{W}^{(t)} + \mtx{\epsilon}^{(t)})$ with regard to neuron $\vct{w}_{j, \vct{\epsilon}}^{(t)}$ is
    \begin{align}
        \nabla_{\vct{w}_{j, \vct{\epsilon}}^{(t)}} \gL(\mtx{W}^{(t)} + \mtx{\epsilon}^{(t)}) = -&\frac{3}{N} \sum_{i=1}^{\alpha N} l_{i, \vct{\epsilon}}^{(t)} \left( \beta_d^3 \inner{\vct{w}_{j, \vct{\epsilon}}^{(t)}}{\vct{v}_d}^2 \vct{v}_d + \beta_e^3 \inner{\vct{w}_{j, \vct{\epsilon}}^{(t)}}{\vct{v}_e}^2 \vct{v}_e + y_i \inner{\vct{w}_{j, \vct{\epsilon}}^{(t)}}{\vct{\xi}_i}^2 \vct{\xi}_i \right) - \nonumber \\
        &\frac{3}{N} \sum_{i=\alpha N + 1}^N l_{i, \vct{\epsilon}}^{(t)} \left( \beta_d^3 \inner{\vct{w}_{j, \vct{\epsilon}}^{(t)}}{\vct{v}_d}^2 \vct{v}_d + y_i \inner{\vct{w}_{j, \vct{\epsilon}}^{(t)}}{\vct{\xi}_i}^2 \vct{\xi}_i \right)
    \end{align}
\end{lemma}

\begin{proof}
    We have the following gradient
    \begin{align*}
        \nabla_{\vct{w}_{j, \vct{\epsilon}}^{(t)}} \gL(\mtx{W}^{(t)} + \mtx{\epsilon}^{(t)}) = -&\frac{1}{N} \sum_{i=1}^N \frac{\exp(-y_i f(\vct{x}_i; \mtx{W}^{(t)} + \mtx{\epsilon}^{(t)}))}{1 + \exp(-y_i f(\vct{x}_i; \mtx{W}^{(t)} + \mtx{\epsilon}^{(t)}))} \cdot y_i f'(\vct{x}_i; \mtx{W}^{(t)} + \mtx{\epsilon}^{(t)}) \\
        = -&\frac{3}{N} \sum_{i=1}^N l_{i, \vct{\epsilon}}^{(t)} y_i \sum_{p=1}^P \inner{\vct{w}_{j, \vct{\epsilon}}^{(t)}}{\vct{x}^{(p)}}^2 \cdot \vct{x}^{(p)} \\
        = -&\frac{3}{N} \sum_{i=1}^{\alpha N} l_{i, \vct{\epsilon}}^{(t)} \left( \beta_d^3 \inner{\vct{w}_{j, \vct{\epsilon}}^{(t)}}{\vct{v}_d}^2 \vct{v}_d + \beta_e^3 \inner{\vct{w}_{j, \vct{\epsilon}}^{(t)}}{\vct{v}_e}^2 \vct{v}_e + y_i \inner{\vct{w}_{j, \vct{\epsilon}}^{(t)}}{\vct{\xi}_i}^2 \vct{\xi}_i \right) - \\
        &\frac{3}{N} \sum_{i=\alpha N + 1}^N l_{i, \vct{\epsilon}}^{(t)} \left( \beta_d^3 \inner{\vct{w}_{j, \vct{\epsilon}}^{(t)}}{\vct{v}_d}^2 \vct{v}_d + y_i \inner{\vct{w}_{j, \vct{\epsilon}}^{(t)}}{\vct{\xi}_i}^2 \vct{\xi}_i \right)
    \end{align*}
\end{proof}

With the above formula of gradient, we have the projection of perturbed weight on $\vct{v}_e$ is

\begin{align}
    \inner{\vct{w}_{j, \vct{\epsilon}}^{(t)}}{\vct{v}_e} &= \inner{\vct{w}_j^{(t)}}{\vct{v}_e} + \inner{ \vct{\epsilon}_j^{(t)}}{\vct{v}_e} \nonumber \\
    &= \inner{\vct{w}_j^{(t)}}{\vct{v}_e} + \inner{\rho^{(t)} \nabla_{\vct{w}_j^{(t)}} \gL(\mtx{W}^{(t)})}{\vct{v}_e} \nonumber \\
    &= \inner{\vct{w}_j^{(t)}}{\vct{v}_e} - \frac{3 \rho^{(t)} \beta_e^3}{N} \sum_{i=1}^{\alpha N} l_i^{(t)} \inner{\vct{w}_j^{(t)}}{\vct{v}_e}^2 \label{eq:proj_perturb_weight_on_easy_feature}
\end{align}
From Equations~\ref{eq:bound_gradient_norm_by_easy_feature} and~\ref{eq:proj_perturb_weight_on_easy_feature}, we have
\begin{equation} 
    0 \leq \inner{\vct{w}_j^{(t)}}{\vct{v}_e} - \rho \leq \inner{\vct{w}_{j, \vct{\epsilon}}^{(t)}}{\vct{v}_e} \leq \inner{\vct{w}_j^{(t)}}{\vct{v}_e} \label{eq:bound_proj_perturb_weight_on_easy_feature} 
\end{equation}

Similarly, the projection of perturbed weight on $\vct{v}_d$ is
\begin{align}
    \inner{\vct{w}_{j, \vct{\epsilon}}^{(t)}}{\vct{v}_d} &= \inner{\vct{w}_j^{(t)}}{\vct{v}_d} - \frac{3 \rho^{(t)} \beta_d^3}{N} \sum_{i=1}^N l_i^{(t)} \inner{\vct{w}_j^{(t)}}{\vct{v}_d}^2 \label{eq:proj_perturb_weight_on_difficult_feature} \\
    0 &\leq \inner{\vct{w}_j^{(t)}}{\vct{v}_d} - \rho \leq \inner{\vct{w}_{j, \vct{\epsilon}}^{(t)}}{\vct{v}_d} \leq \inner{\vct{w}_j^{(t)}}{\vct{v}_d} \label{eq:bound_proj_perturb_weight_on_difficult_feature}
\end{align}

\textbf{\Fast feature gradient.} The projection of the gradient on $\vct{v}_e$ is
\begin{equation}\label{eq:easy_feature_gradient_sam}
    \inner{\nabla_{\vct{w}_{j, \vct{\epsilon}}^{(t)}} \gL(\mtx{W}^{(t)} + \mtx{\epsilon}^{(t)})}{\vct{v}_e} = -\frac{3\beta_e^3}{N} \sum_{i=1}^{\alpha N} l_{i, \vct{\epsilon}}^{(t)} \inner{\vct{w}_{j, \vct{\epsilon}}^{(t)}}{\vct{v}_e}^2
\end{equation}

\textbf{\Slow feature gradient.} The projection of the gradient on $\vct{v}_d$ is
\begin{equation}\label{eq:difficult_feature_gradient_sam}
    \inner{\nabla_{\vct{w}_{j, \vct{\epsilon}}^{(t)}} \gL(\mtx{W}^{(t)} + \mtx{\epsilon}^{(t)})}{\vct{v}_d} = -\frac{3\beta_d^3}{N} \sum_{i=1}^N l_{i, \vct{\epsilon}}^{(t)} \inner{\vct{w}_{j, \vct{\epsilon}}^{(t)}}{\vct{v}_d}^2
\end{equation}

\textbf{Noise gradient.} The projection of the gradient on $\vct{\xi}_i$ is
\begin{equation}\label{eq:noise_gradient_sam}
    \inner{\nabla_{\vct{w}_{j, \vct{\epsilon}}^{(t)}} \gL(\mtx{W}^{(t)} + \mtx{\epsilon}^{(t)})}{\vct{\xi}_i} = -\frac{3}{N} \left( l_{i, \vct{\epsilon}}^{(t)} y_i \inner{\vct{w}_{j, \vct{\epsilon}}^{(t)}}{\vct{\xi}_i}^2 \norm{\vct{\xi}_i}_2^2 + \sum_{k=1, k \neq i}^N l_{k, \vct{\epsilon}}^{(t)} y_k \inner{\vct{w}_{j, \vct{\epsilon}}^{(t)}}{\vct{\xi}_k}^2 \inner{\vct{\xi}_k}{\vct{\xi}_i} \right)
\end{equation}

\textbf{Derivative of data example $i$.} 
For $1 \leq i \leq \alpha N$, $l_{i, \vct{\epsilon}}^{(t)}$ can be rewritten as
\begin{equation}\label{eq:deriv_data_sam}
    l_{i, \vct{\epsilon}}^{(t)} = \text{sigmoid} \left( \sum_{j=1}^J -\beta_d^3 \inner{\vct{w}_{j, \vct{\epsilon}}^{(t)}}{\vct{v}_d}^3 - \beta_e^3 \inner{\vct{w}_{j, \vct{\epsilon}}^{(t)}}{\vct{v}_e}^3 - y_i \inner{\vct{w}_{j, \vct{\epsilon}}^{(t)}}{\vct{\xi}_i}^3 \right)
\end{equation}
while for $\alpha N + 1 \leq i \leq N$, $l_i^{(t)}$ can be rewritten as
\begin{align}
    l_{i, \vct{\epsilon}}^{(t)} &= \text{sigmoid} \left( \sum_{j=1}^J -\beta_d^3 \inner{\vct{w}_{j, \vct{\epsilon}}^{(t)}}{\vct{v}_d}^3 - y_i \inner{\vct{w}_{j, \vct{\epsilon}}^{(t)}}{\vct{\xi}_i}^3 \right) \nonumber \\
    &\geq \text{sigmoid} \left( \sum_{j=1}^J -\beta_d^3 \inner{\vct{w}_{j, \vct{\epsilon}}^{(t)}}{\vct{v}_d}^3 - \beta_e^3 \inner{\vct{w}_{j, \vct{\epsilon}}^{(t)}}{\vct{v}_e}^3 - y_i \inner{\vct{w}_{j, \vct{\epsilon}}^{(t)}}{\vct{\xi}_i}^3 \right)
\end{align}

We present the detailed proofs that build up to Theorem~\ref{the:easy_difficult_sam}. We begin by considering the update for the \fast and \slow features.

\begin{lemma}[\Fast feature update.]\label{lem:easy_feature_update_sam}
    For all $t \geq 0$ and $j \in [J]$, the \fast feature update is 
    \begin{align}
        \inner{\vct{w}_j^{(t)}}{\vct{v}_e} + \tilde{\Theta}(\eta) \alpha \beta_e^3 g_2(t) (\inner{\vct{w}_j^{(t)}}{\vct{v}_e} - \rho)^2 &\leq \inner{\vct{w}_j^{(t+1)}}{\vct{v}_e} = \inner{\vct{w}_j^{(t)}}{\vct{v}_e} + \tilde{\Theta}(\eta) \alpha \beta_e^3 g_2(t) \inner{\vct{w}_{j, \vct{\epsilon}}^{(t)}}{\vct{v}_e}^2 \nonumber\\
        &\leq \inner{\vct{w}_j^{(t)}}{\vct{v}_e} + \tilde{\Theta}(\eta) \alpha \beta_e^3 g_2(t) (\inner{\vct{w}_j^{(t)}}{\vct{v}_e})^2
    \end{align}
    where $g_2(t) = \text{sigmoid} \left( \sum_{j=1}^J -\beta_d^3 \inner{\vct{w}_{j, \vct{\epsilon}}^{(t)}}{\vct{v}_d}^3 - \beta_e^3 \inner{\vct{w}_{j, \vct{\epsilon}}^{(t)}}{\vct{v}_e}^3 \right)$.
\end{lemma}

\begin{proof}
    Plugging the update rule of SAM, we have
    \begin{align*}
        \inner{\vct{w}_j^{(t+1)}}{\vct{v}_e} &= \inner{\vct{w}_j^{(t)} - \eta \nabla_{\vct{w}_{j, \vct{\epsilon}}^{(t)}} \gL(\mtx{W}^{(t)} + \mtx{\epsilon}^{(t)})}{\vct{v}_e} \\
        &= \inner{\vct{w}_j^{(t)}}{\vct{v}_e} + \frac{3 \eta \alpha \beta_e^3}{N} \sum_{i=1}^N l_{i, \vct{\epsilon}}^{(t)} \inner{\vct{w}_{j, \vct{\epsilon}}^{(t)}}{\vct{v}_e}^2 \\
        &= \inner{\vct{w}_j^{(t)}}{\vct{v}_e} + \tilde{\Theta}(\eta) \alpha \beta_e^3 g_2(t) \inner{\vct{w}_{j, \vct{\epsilon}}^{(t)}}{\vct{v}_e}^2,
    \end{align*}
    where the last equality holds due to Lemma~\ref{lem:approx_deriv_data_sam}. Combining with Equation~\ref{eq:bound_proj_perturb_weight_on_easy_feature}, we obtain the desired inequalities.
\end{proof}

Similarly, we obtain the following update rule for \slow features.

\begin{lemma}[\Slow feature update.] For all $t \geq 0$ and $j \in [J]$, the \slow feature update is 
    \begin{align}
        \inner{\vct{w}_j^{(t)}}{\vct{v}_d} + \tilde{\Theta}(\eta) \beta_d^3 g_2(t) (\inner{\vct{w}_j^{(t)}}{\vct{v}_d} - \rho)^2 &\leq \inner{\vct{w}_j^{(t+1)}}{\vct{v}_d} = \inner{\vct{w}_j^{(t)}}{\vct{v}_d} + \frac{3 \eta \beta_d^3}{N} \sum_{i=1}^N l_i^{(t)} \inner{\vct{w}_{j, \vct{\epsilon}}^{(t)}}{\vct{v}_d}^2 \nonumber\\
        &\leq \inner{\vct{w}_j^{(t)}}{\vct{v}_d} + \tilde{\Theta}(\eta) \beta_d^3 (\alpha g_2(t) + 1 - \alpha) (\inner{\vct{w}_j^{(t)}}{\vct{v}_d})^2
    \end{align}
    where $g_2(t) = \text{sigmoid} \left( \sum_{j=1}^J -\beta_d^3 \inner{\vct{w}_{j, \vct{\epsilon}}^{(t)}}{\vct{v}_d}^3 - \beta_e^3 \inner{\vct{w}_{j, \vct{\epsilon}}^{(t)}}{\vct{v}_e}^3 \right)$.
\end{lemma}

\begin{proof}
    Plugging the update rule of GD, we have
    \begin{align*}
        \inner{\vct{w}_j^{(t+1)}}{\vct{v}_d} &= \inner{\vct{w}_j^{(t)} - \eta \nabla_{\vct{w}_{j, \vct{\epsilon}}^{(t)}} \gL(\mtx{W}^{(t)})}{\vct{v}_d} \\
        &= \inner{\vct{w}_j^{(t)}}{\vct{v}_d} + \frac{3 \eta \beta_d^3}{N} \sum_{i=1}^N l_i^{(t)} \inner{\vct{w}_{j, \vct{\epsilon}}^{(t)}}{\vct{v}_d}^2
    \end{align*}
    From Lemma~\ref{lem:approx_deriv_data_gd}, we have for $1 \leq i \leq \alpha N, l_i^{(t)} = \Theta(1) g_1(t)$ and for $\alpha N + 1 \leq i \leq N, \Theta(1) g_1(t) \leq l_i^{(t)} \leq 1$. Combining with the above equality and Equation~\ref{eq:bound_proj_perturb_weight_on_difficult_feature}, we obtain the desired inequalities.
\end{proof}

Next, we simplify the two above update rules in the early training stage.

\begin{lemma}[\Fast feature update in early iterations]\label{lem:easy_feature_update_sam_simplified}
    Let $T_0 > 0$ be such that $\max_{j \in [J]} \inner{\vct{w}_j^{(T_0)}}{\vct{v}_e} \geq \tilde{\Omega}(1/\beta_e)$. For $t \in [0, T_0]$, the \fast feature update has the following rule
    \begin{align}
         \tilde{\Theta}(\eta) \alpha \beta_e^3 (\inner{\vct{w}_j^{(t)}}{\vct{v}_e} - \rho)^2 \leq \inner{\vct{w}_j^{(t+1)}}{\vct{v}_e} - \inner{\vct{w}_j^{(t)}}{\vct{v}_e} \leq \inner{\vct{w}_j^{(t)}}{\vct{v}_e} + \tilde{\Theta}(\eta) \alpha \beta_e^3 (\inner{\vct{w}_j^{(t)}}{\vct{v}_e})^2
    \end{align}
\end{lemma}

\begin{proof}
    Let $T_0 > 0$ be such that either $\max_{j \in [J]} \inner{\vct{w}_j^{(T_0)}}{\vct{v}_e} \geq \tilde{\Omega}(1/\beta_e)$ or $\max_{j \in [J]} \inner{\vct{w}_j^{(T_0)}}{\vct{v}_d} \geq \tilde{\Omega}(1/\beta_d)$. We will show later that the first condition will be met and we have $\max_{j \in [J]} \inner{\vct{w}_j^{(T_0)}}{\vct{v}_d} \leq \tilde{\Omega}(1/\beta_d)$ for all $j \in [J]$ and $t \in [0, T_0]$.

    Recall that $g_2(t) = \text{sigmoid} \left(\sum_{j=1}^J -\beta_d^3 \inner{\vct{w}_{j, \vct{\epsilon}}^{(t)}}{\vct{v}_d}^3 - \beta_e^3 \inner{\vct{w}_{j, \vct{\epsilon}}^{(t)}}{\vct{v}_e}^3 \right)$. Then, for $t \in [0, T_0]$, we have
    \begin{align*}
        g_2(t) &= \frac{1}{1 + \exp(\sum_{j=1}^J -\beta_d^3 \inner{\vct{w}_{j, \vct{\epsilon}}^{(t)}}{\vct{v}_d}^3 - \beta_e^3 \inner{\vct{w}_{j, \vct{\epsilon}}^{(t)}}{\vct{v}_e}^3)} \\
        &\geq \frac{1}{1 + \exp(\kappa + \kappa)} \\
        &= \frac{1}{1 + \exp(\tilde{\Omega}(1))},
    \end{align*}
    where the first inequality holds due to $\inner{\vct{w}_{j, \vct{\epsilon}}^{(t)}}{\vct{v}_e} \leq \inner{\vct{w}_j^{(t)}}{\vct{v}_e} \leq \kappa/(J^{1/3} \beta_e)$ and $\inner{\vct{w}_{j, \vct{\epsilon}}^{(t)}}{\vct{v}_d} \leq \inner{\vct{w}_j^{(t)}}{\vct{v}_d} \leq \kappa/(J^{1/3} \beta_d)$ for $t \in [0, T_0]$. Therefore, we have $g_2(t) = \Theta(1)$ in the early iterations. Replacing $g_2(t) = \Theta(1)$ into the results of Lemma~\ref{lem:easy_feature_update_sam}, we obtain the desired results.
\end{proof}

Similarly, we obtain the following simplified update rule for \slow features in the early iterations.

\begin{lemma}[\Slow feature update in early iterations]\label{lem:difficult_feature_update_sam_simplified}
    Let $T_0 > 0$ be such that $\max_{j \in [J]} \inner{\vct{w}_j^{(T_0)}}{\vct{v}_e} \geq \tilde{\Omega}(1/\beta_e)$. For $t \in [0, T_0]$, the \fast feature update has the following rule
    \begin{align}
        \tilde{\Theta}(\eta) \beta_d^3 (\inner{\vct{w}_j^{(t)}}{\vct{v}_d} - \rho)^2 \leq \inner{\vct{w}_j^{(t+1)}}{\vct{v}_d} - \inner{\vct{w}_j^{(t)}}{\vct{v}_d} \leq  \tilde{\Theta}(\eta) \beta_d^3 (\inner{\vct{w}_j^{(t)}}{\vct{v}_d})^2
    \end{align}
\end{lemma}

We next show that SAM will learn the \fast feature quicker than the \slow one.

\begin{lemma}\label{lem:sam_learn_easy_quicker}
    Assume $\eta = \tilde{o}(\beta_d \sigma_0)$. Let $T_0$ be the iteration number that $\max_{j \in [J]} \inner{\vct{w}_j^{(T_0)}}{\vct{v}_e}$ reaches $\tilde{\Omega}(1/\beta_e) = \tilde{\Theta(1)}$. Then, we have for all $t \leq T_0$, it holds that $\max_{j \in [J]} \inner{\vct{w}_j^{(T_0)}}{\vct{v}_d} = \tilde{O}(\sigma_0)$.
\end{lemma}

\begin{proof}
    Among all the possible indices $j \in [J]$ 
    , we focus on the index $j^\star = \argmax_{j \in [J]} \inner{\vct{w}_j^{(0)}}{\vct{v}_e}$. 
    Therefore, for $C_t = \alpha \beta_e^3 = \Theta(1)$, we apply Lemma~\ref{lem:grow_speed_two_sequences} with two positive sequences $\inner{\vct{w}_{j^\star}^{(t)}}{\vct{v}_e}$ and $\inner{\vct{w}_j^{(t)}}{\vct{v}_d}$ defined in Lemmas~\ref{lem:easy_feature_update_sam_simplified} and~\ref{lem:difficult_feature_update_sam_simplified} and get
    \begin{equation}
        \inner{\vct{w}_j^{(t)}}{\vct{v}_d} \leq O(\inner{\vct{w}_j^{(0)}}{\vct{v}_d}) = \tilde{O}(\sigma_0)
    \end{equation}
    for all $j \in [J]$.
\end{proof}

\begin{theorem}[Restatement of Theorem~\ref{the:easy_difficult_sam_main}]\label{the:easy_difficult_sam}
    We consider training a two-layer nonlinear CNN model initialized with $\mtx{W}^{(0)} \sim \gN(0, \sigma_0^2)$ on the training dataset $D = \{ (\vct{x}_i, y_i) \}_{i=1}^N$ that follows the data distribution $\gD(\beta_e, \beta_d, \alpha)$ with $\alpha^{1/3} \beta_e > \beta_d$. After training with SAM in Equation~\ref{eq:gd_update} for $T_\text{SAM}$ iterations where
    \begin{equation}
        T_\text{SAM} = \frac{\tilde{\Theta}(\sigma_0)}{\eta \alpha \beta_e^3 (\sigma_0 - \rho)^2} + \frac{\tilde{\Theta}(\sigma_0^2)}{(\sigma_0 - \rho)^2}\Big\lceil \frac{-\log(\sigma_0 \beta_e)}{\log(2)} \Big\rceil, \label{eq:learning_time_sam}
    \end{equation}
    for all $j \in [J]$ and $t \in [0, T_\text{SAM})$, we have
    \begin{align}
        \inner{\vct{w}_j^{(t)}}{\vct{v}_e} + \tilde{\Theta}(\eta) \alpha \beta_e^3 (\inner{\vct{w}_j^{(t)}}{\vct{v}_e} - \rho)^2 &\leq \inner{\vct{w}_j^{(t+1)}}{\vct{v}_e} = \inner{\vct{w}_j^{(t)}}{\vct{v}_e} + \tilde{\Theta}(\eta) \alpha \beta_e^3 \inner{\vct{w}_{j, \vct{\epsilon}}^{(t)}}{\vct{v}_e}^2 \nonumber\\
        &\leq \inner{\vct{w}_j^{(t)}}{\vct{v}_e} + \tilde{\Theta}(\eta) \alpha \beta_e^3 (\inner{\vct{w}_j^{(t)}}{\vct{v}_e})^2 \\
        \inner{\vct{w}_j^{(t)}}{\vct{v}_d} + \tilde{\Theta}(\eta) \beta_d^3 (\inner{\vct{w}_j^{(t)}}{\vct{v}_d} - \rho)^2 &\leq \inner{\vct{w}_j^{(t+1)}}{\vct{v}_d} = \inner{\vct{w}_j^{(t)}}{\vct{v}_d} + \tilde{\Theta}(\eta) \beta_d^3 \inner{\vct{w}_{j, \vct{\epsilon}}^{(t)}}{\vct{v}_d}^2 \nonumber\\
        &\leq \inner{\vct{w}_j^{(t)}}{\vct{v}_d} + \tilde{\Theta}(\eta) \beta_d^3 (\inner{\vct{w}_j^{(t)}}{\vct{v}_d})^2
    \end{align}
    After training for $T_\text{SAM}$ iterations, with high probability, the learned weight has the following properties: (1) it learns the \fast feature $\vct{v}_e: \max_{j \in [J]} \inner{\vct{w}_j^{(T_\text{SAM})}}{\vct{v}_e} \geq \tilde{\Omega}(1/\beta_e)$; (2) it does not learn the \slow feature $\vct{v}_d: \max_{j \in [J]} \inner{\vct{w}_j^{(T_\text{SAM})}}{\vct{v}_d} = \tilde{O}(\sigma_0)$.
\end{theorem}

\begin{proof}
    With the results of Lemmas~\ref{lem:easy_feature_update_sam_simplified}-~\ref{lem:sam_learn_easy_quicker}, it remains to calculate the time $T_\text{SAM}$. Plugging $v = \tilde{\Omega}(1/\beta_e), m = M = \tilde{\Theta}(\eta) \alpha \beta_e^3, z_0 = \tilde{O}(\sigma_0)$ into Lemma~\ref{lem:grow_speed_one_sequence_two_equations}, we have $T_\text{SAM}$ as 
    \begin{equation}
        T_\text{SAM} = \frac{\tilde{\Theta}(\sigma_0)}{\eta \alpha \beta_e^3 (\sigma_0 - \rho)^2} + \frac{\tilde{\Theta}(\sigma_0^2)}{(\sigma_0 - \rho)^2}\Big\lceil \frac{-\log(\sigma_0 \beta_e)}{\log(2)} \Big\rceil
    \end{equation}

Comparing Eq.~\ref{eq:learning_time_gd} and ~\ref{eq:learning_time_sam}, we can see that SAM learns the \fast features later than GD. %
Particularly, if we remove the approximate notations, we have 
the following inequality
\begin{equation}
    \frac{1}{\eta \beta_e^3 \sigma_0} + \Big\lceil \frac{-\log(\sigma_0 \beta_e)}{\log(2)} \Big\rceil \geq \frac{\sigma_0}{\eta \beta_e^3 (\sigma_0 - \rho)^2} + \frac{\sigma_0^2}{(\sigma_0 - \rho)^2}\Big\lceil \frac{-\log(\sigma_0 \beta_e)}{\log(2)} \Big\rceil,
\end{equation}
which holds due to Assumption~\ref{ass:weight_initialization} about weight initialization in Appendix~\ref{app:proof_gd}, i.e., ($\sigma_0 \geq \rho \geq 0$). \looseness=-1
\end{proof}

\subsection{Proof of Theorem~\ref{the:sam_learn_more_uniform_than_gd_main}}\label{app:proof_uniformity}

In this section, we show that SAM learns \fast and \slow features at a more uniform speed. To ease the notation, we denote $G_{e}^{(t)} = \max_{j \in [J]} \inner{\vct{w}_j^{(t)}}{\vct{v}_e}$ and $G_{d}^{(t)} = \max_{j \in [J]} \inner{\vct{w}_j^{(t)}}{\vct{v}_d}$ for model weights trained with GD. Similarly, we denote $S_e^{(t)}$ and $S_d^{(t)}$ for model weights trained with SAM. We use $\hat{S}_e^{(t)}$ and $\hat{S}_d^{(t)}$ to denote the inner products with perturbed weights. We simplify Equation~\ref{eq:proj_perturb_weight_on_easy_feature} and~\ref{eq:proj_perturb_weight_on_difficult_feature} for early iterations $t \leq T_0$ as
\begin{align}
    \hat{S}_e^{(t)} &= S_e^{(t)} - \tilde{\Theta}(1) \rho^{(t)} \alpha \beta_e^3 (S_e^{(t)})^2 \label{eq:proj_perturb_weight_on_easy_feature_simplified}\\
    \hat{S}_d^{(t)} &= S_d^{(t)} - \tilde{\Theta}(1) \rho^{(t)} \beta_d^3 (S_d^{(t)})^2 \label{eq:proj_perturb_weight_on_difficult_feature_simplified}
\end{align}
Before introducing the theorem, we assume that the model is initialized in favor of the \fast feature, i.e. $G_{e}^{(0)} - G_{d}^{(0)} \geq \rho$. This is reasonable as a consequence of Theorem~\ref{the:easy_difficult_gd} because we can just train the model for several iterations to achieve this initialization (similar argument for Assumption~\ref{ass:weight_initialization}).

\begin{theorem}[Restatement of Theorem~\ref{the:sam_learn_more_uniform_than_gd_main}]\label{the:sam_learn_more_uniform_than_gd}
    Consider the training dataset $D = \{ (\vct{x}_i, y_i) \}_{i=1}^N$ that follows the data distribution $\gD(\beta_e, \beta_d, \alpha)$ using the two-layer nonlinear CNN model initialized with $\mtx{W}^{(0)} \sim \gN(0, \sigma_0^2)$. Assume that the \fast feature strength is significantly larger $\alpha^{1/3} \beta_e > \beta_d$. Training the same model initialization, we have that for every iteration $t \leq T_0$
    \begin{align}
        \rho + S_d^{(t)} &\leq S_e^{(t)} \\
        \hat{S}_d^{(t)} &< \hat{S}_e^{(t)} \\
        S_e^{(t)} &\leq G_{e}^{(t)} \\
        S_d^{(t)} &\leq G_{d}^{(t)} \\
        S_e^{(t)} - S_d^{(t)} &\leq G_{e}^{(t)} - G_{d}^{(t)}
    \end{align}
\end{theorem}

\begin{proof}
     We prove this by induction. For $t = 0$, the above hypotheses immediately hold because we use train two methods from the same initialization. Particularly, we have $0 < S_d^{(0)} = G_{d}^{(0)} < G_{e}^{(0)} = S_e^{(0)}$ and $\hat{S}_e^{(0)} - \hat{S}_d^{(0)} \geq S_{e}^{(0)} - \rho - S_{d}^{(0)} \geq (G_{e}^{(0)} - G_{d}^{(0)}) - \rho \geq 0$.
    
    Assume that the induction hypotheses hold for $t$, i.e.
    \begin{align}
        \rho + S_d^{(t)} &\leq S_e^{(t)} \\
        \hat{S}_d^{(t)} &< \hat{S}_e^{(t)} \\
        S_e^{(t)} &\leq G_{e}^{(t)} \\
        S_d^{(t)} &\leq G_{d}^{(t)} \\
        S_e^{(t)} - S_d^{(t)} &\leq G_{e}^{(t)} - G_{d}^{(t)}
    \end{align}
    We need to prove that they also hold for $t + 1$. From Lemma~\ref{the:easy_difficult_sam} and the first two induction hypotheses,
    \begin{align}
        \rho + S_d^{(t+1)} = \rho + S_d^{(t)} + \tilde{\Theta}(\eta) \beta_d^3 (\hat{S}_d^{(t)})^2 < S_e^{(t)} + \tilde{\Theta}(\eta) \alpha \beta_e^3 (\hat{S}_e^{(t)})^2 = S_e^{(t+1)}
    \end{align}
    Then, 
    $\hat{S}_e^{(t+1)} - \hat{S}_d^{(t+1)} \geq S_{e}^{(t+1)} - \rho - S_{d}^{(t+1)} \geq 0$.
    From Equation~\ref{eq:bound_proj_perturb_weight_on_easy_feature} and Lemma~\ref{lem:easy_feature_update_sam_simplified}, 
    \begin{align}
        S_e^{(t+1)} \leq S_e^{(t)} + \tilde{\Theta}(\eta) \alpha \beta_e^3 (S_e^{(t)})^2 \leq G_{e}^{(t)} + \tilde{\Theta}(\eta) \alpha \beta_e^3 (G_{e}^{(t)})^2 \leq G_{e}^{(t+1)}.
    \end{align}
    Similarly, we have $S_d^{(t+1)} \leq G_{d}^{(t+1)}$. From Equations~\ref{eq:proj_perturb_weight_on_easy_feature_simplified} and~\ref{eq:proj_perturb_weight_on_difficult_feature_simplified},
    \begin{align}
        S_d^{(t)} - \hat{S}_d^{(t)} = \tilde{\Theta}(1) \rho^{(t)} \beta_d^3 (S_d^{(t)})^2) &< \tilde{\Theta}(1) \rho^{(t)} \alpha \beta_e^3 (S_e^{(t)})^2 = S_e^{(t)} - \hat{S}_e^{(t)} \\
        0 \leq \hat{S}_e^{(t)} - \hat{S}_d^{(t)} < S_e^{(t)} - S_d^{(t)} &\leq G_{e}^{(t)} - G_{d}^{(t)}
    \end{align}
    Combining with Equations~\ref{eq:bound_proj_perturb_weight_on_easy_feature} and~\ref{eq:bound_proj_perturb_weight_on_difficult_feature}, we have
    \begin{align}
        (\hat{S}_e^{(t)})^2 - (\hat{S}_d^{(t)})^2 < (S_e^{(t)})^2 - (S_d^{(t)})^2 &\leq (G_{e}^{(t)})^2 - (G_{d}^{(t)})^2  \\
        (G_d^{(t)})^2 - (\hat{S}_d^{(t)})^2 &< (G_e^{(t)})^2 - (\hat{S}_e^{(t)})^2 \\
        \tilde{\Theta}(\eta) \beta_d^3 ((G_d^{(t)})^2 - (\hat{S}_d^{(t)})^2) &< \tilde{\Theta}(\eta) \alpha \beta_e^3 ((G_e^{(t)})^2 - (\hat{S}_e^{(t)})^2) \\
        G_d^{(t)} - S_d^{(t)} + \tilde{\Theta}(\eta) \beta_d^3 ((G_d^{(t)})^2 - (\hat{S}_d^{(t)})^2) &< G_{e}^{(t)} - S_{e}^{(t)} + \tilde{\Theta}(\eta) \alpha \beta_e^3 ((G_{e}^{(t)})^2 - (\hat{S}_{e}^{(t)})^2) \\
        G_d^{(t+1)} - S_d^{(t+1)} &< G_{e}^{(t+1)} - S_{e}^{(t+1)} \\
        S_e^{(t+1)} - S_d^{(t+1)} &< G_{e}^{(t+1)} - G_{d}^{(t+1)}
    \end{align}
    Therefore, the induction hypotheses hold for $t+1$.
\end{proof}

\subsection{Proof of Theorem~\ref{the:one_step_upsampling_main}}\label{app:proof_one_step_upsampling}
From Theorem~\ref{the:sam_learn_more_uniform_than_gd}, we have the following result for switching between SAM and GD during training.

\begin{lemma}\label{lem:one_step_normalized_gradient}
Consider the training dataset $D = \{ (\vct{x}_i, y_i) \}_{i=1}^N$ that follows the data distribution $\gD(\beta_e, \beta_d, \alpha)$ using the two-layer nonlinear CNN model initialized with $\mtx{W}^{(0)} \sim \gN(0, \sigma_0^2)$. Assume that the noise is sufficiently small (ref. Lemmas~\ref{lem:bound_noise_gd} and~\ref{lem:bound_noise_sam}) and the \fast feature strength is significantly larger $\alpha^{1/3} \beta_e > \beta_d$. From any iteration $t$ during early training, the normalized gradient of the one-step SAM update has a larger weight on the \slow feature compared to that of GD.
\end{lemma}

\begin{proof}
     First, recall the gradients of GD and SAM are as follows.
    \begin{align}
        \nabla_{\vct{w}_j^{(t)}} \gL(\mtx{W}^{(t)}) = - &\frac{3}{N} \sum_{i=1}^{\alpha N} l_i^{(t)} \left( \beta_d^3 \inner{\vct{w}_j^{(t)}}{\vct{v}_d}^2 \vct{v}_d + \beta_e^3 \inner{\vct{w}_j^{(t)}}{\vct{v}_e}^2 \vct{v}_e + y_i \inner{\vct{w}_j^{(t)}}{\vct{\xi}_i}^2 \vct{\xi}_i \right) -\nonumber \\ 
        &\frac{3}{N} \sum_{i=\alpha N + 1}^N l_i^{(t)} \left( \beta_d^3 \inner{\vct{w}_j^{(t)}}{\vct{v}_d}^2 \vct{v}_d + y_i \inner{\vct{w}_j^{(t)}}{\vct{\xi}_i}^2 \vct{\xi}_i \right) \\
        \nabla_{\vct{w}_{j, \vct{\epsilon}}^{(t)}} \gL(\mtx{W}^{(t)} + \mtx{\epsilon}^{(t)}) = -&\frac{3}{N} \sum_{i=1}^{\alpha N} l_{i, \vct{\epsilon}}^{(t)} \left( \beta_d^3 \inner{\vct{w}_{j, \vct{\epsilon}}^{(t)}}{\vct{v}_d}^2 \vct{v}_d + \beta_e^3 \inner{\vct{w}_{j, \vct{\epsilon}}^{(t)}}{\vct{v}_e}^2 \vct{v}_e + y_i \inner{\vct{w}_{j, \vct{\epsilon}}^{(t)}}{\vct{\xi}_i}^2 \vct{\xi}_i \right) - \nonumber \\
        &\frac{3}{N} \sum_{i=\alpha N + 1}^N l_{i, \vct{\epsilon}}^{(t)} \left( \beta_d^3 \inner{\vct{w}_{j, \vct{\epsilon}}^{(t)}}{\vct{v}_d}^2 \vct{v}_d + y_i \inner{\vct{w}_{j, \vct{\epsilon}}^{(t)}}{\vct{\xi}_i}^2 \vct{\xi}_i \right)
    \end{align}
    
    Because the noise is sufficiently small, the above equations can be simplified as
    \begin{align}
        \nabla_{\vct{w}^{(t)}} \gL(\mtx{W}^{(t)}) &= -\left(\frac{3}{N} \sum_{i=1}^N l_i^{(t)}  \beta_d^3 \inner{\vct{w}^{(t)}}{\vct{v}_d}^2 \right)\vct{v}_d - \left(\frac{3}{N} \sum_{i=1}^{\alpha N} l_i^{(t)}  \beta_e^3 \inner{\vct{w}^{(t)}}{\vct{v}_e}^2 \right) \vct{v}_e  \\
        \nabla_{\vct{w}_{\vct{\epsilon}}^{(t)}} \gL(\mtx{W}^{(t)} + \mtx{\epsilon}^{(t)}) &= -\left( \frac{3}{N} \sum_{i=1}^N l_{i, \vct{\epsilon}}^{(t)} \beta_d^3 \inner{\vct{w}_{\vct{\epsilon}}^{(t)}}{\vct{v}_d}^2 \right) \vct{v}_d - \left( \frac{3}{N} \sum_{i=1}^{\alpha N} l_{i, \vct{\epsilon}}^{(t)}  \beta_e^3 \inner{\vct{w}_{\vct{\epsilon}}^{(t)}}{\vct{v}_e}^2 \right)  \vct{v}_e 
    \end{align}
    Note that in early training, we have an approximation for the logit terms as $l_i^{(t)} = l_{i, \vct{\epsilon}}^{(t)} = \Theta(1)$, we can further simplify the gradients as
    \begin{align}
        \nabla_{\vct{w}^{(t)}} \gL(\mtx{W}^{(t)}) &= -\left(3 \beta_d^3 \inner{\vct{w}^{(t)}}{\vct{v}_d}^2 \right)\vct{v}_d - \left(3 \alpha \beta_e^3 \inner{\vct{w}^{(t)}}{\vct{v}_e}^2 \right) \vct{v}_e  \\
        \nabla_{\vct{w}_{\vct{\epsilon}}^{(t)}} \gL(\mtx{W}^{(t)} + \mtx{\epsilon}^{(t)}) &= -\left(3 \beta_d^3 \inner{\vct{w}_{\vct{\epsilon}}^{(t)}}{\vct{v}_d}^2 \right) \vct{v}_d - \left(3 \alpha \beta_e^3 \inner{\vct{w}_{\vct{\epsilon}}^{(t)}}{\vct{v}_e}^2 \right)  \vct{v}_e 
    \end{align}
    Both gradients of GD and SAM can be decomposed into the linear combination of \fast and \slow features. To prove that the normalized gradient of SAM favors the \slow feature compared to GD, it is sufficient to show the ratio of coefficients in SAM is larger than GD. In other words, we need to verify that
    \begin{align}
        \frac{3 \beta_d^3 \inner{\vct{w}_{\vct{\epsilon}}^{(t)}}{\vct{v}_d}^2}{3 \alpha \beta_e^3 \inner{\vct{w}_{\vct{\epsilon}}^{(t)}}{\vct{v}_e}^2} &\geq \frac{3 \beta_d^3 \inner{\vct{w}^{(t)}}{\vct{v}_d}^2}{3 \alpha \beta_e^3 \inner{\vct{w}^{(t)}}{\vct{v}_e}^2} \label{eq:ratio_of_coefficients}\\
        \frac{\inner{\vct{w}_{\vct{\epsilon}}^{(t)}}{\vct{v}_d}}{\inner{\vct{w}_{\vct{\epsilon}}^{(t)}}{\vct{v}_e}} &\geq \frac{\inner{\vct{w}^{(t)}}{\vct{v}_d}}{\inner{\vct{w}^{(t)}}{\vct{v}_e}} \\
        \frac{\inner{\vct{w}^{(t)}}{\vct{v}_d} - 3 \rho^{(t)} \beta_d^3 \inner{\vct{w}^{(t)}}{\vct{v}_d}^2}{\inner{\vct{w}^{(t)}}{\vct{v}_e} - 3 \rho^{(t)} \alpha \beta_e^3 \inner{\vct{w}^{(t)}}{\vct{v}_e}^2} &\geq \frac{\inner{\vct{w}^{(t)}}{\vct{v}_d}}{\inner{\vct{w}^{(t)}}{\vct{v}_e}} \\
        1 - 3 \rho^{(t)} \beta_d^3 \inner{\vct{w}^{(t)}}{\vct{v}_d} &\geq 1- 3 \rho^{(t)} \alpha \beta_e^3 \inner{\vct{w}^{(t)}}{\vct{v}_e} \\
        \alpha \beta_e^3 \inner{\vct{w}^{(t)}}{\vct{v}_e} &\geq \beta_d^3 \inner{\vct{w}^{(t)}}{\vct{v}_d}
    \end{align}
    The last inequality holds due to $\alpha \beta_e^3 > \beta_d^3$ and $\inner{\vct{w}^{(t)}}{\vct{v}_e} \geq \inner{\vct{w}^{(t)}}{\vct{v}_d}$ from Theorem~\ref{the:sam_learn_more_uniform_than_gd}.
\end{proof}

From Equation~\ref{eq:ratio_of_coefficients} in the above proof, it can be seen clearly that amplifying the \slow feature strength in either GD or SAM, i.e., increasing $\beta_d$, places a larger weight on the \slow feature. Thus, we have the next theorem.

\begin{theorem}[Restatement of Theorem~\ref{the:one_step_upsampling_main}]\label{the:one_step_upsampling}
Consider the training dataset $D = \{ (\vct{x}_i, y_i) \}_{i=1}^N$ that follows the data distribution $\gD(\beta_e, \beta_d, \alpha)$ using the two-layer nonlinear CNN model initialized with $\mtx{W}^{(0)} \sim \gN(0, \sigma_0^2)$. Assume that the noise is sufficiently small (ref. Lemmas~\ref{lem:bound_noise_gd} and~\ref{lem:bound_noise_sam}) and the \fast feature strength is significantly larger $\alpha^{1/3} \beta_e > \beta_d$. We have the following results for one-step upsampling, i.e. increasing $\beta_d$, from any iteration $t$ during early training
\begin{enumerate}
    \item The normalized gradient of the one-step SAM update has a larger weight on the \slow feature compared to that of GD.
    \item Amplifying the \slow feature strength puts a larger weight on the \slow feature in the normalized gradients of GD and SAM.
    \item There exists an upsampling factor $k$ such that the normalized gradient of the one-step GD update on $\gD(\beta_e, k \beta_d, \alpha)$ recovers the normalized gradient of the one-step SAM update on $\gD(\beta_e, \beta_d, \alpha)$.
\end{enumerate}
\end{theorem}

\begin{proof}
    The first result has already been proved in Lemma~\ref{lem:one_step_normalized_gradient}. Now, consider increasing the \slow feature strength from $\beta_d$ to $\beta_d'$. Similar to the proof of Corollary~\ref{lem:one_step_normalized_gradient}, to verify that the normalized of the new normalized gradient of GD favors the \slow feature, it is sufficient to show 
    \begin{align}
        \frac{(\beta_d')^3 \inner{\vct{w}^{(t)}}{\vct{v}_d}^2}{\beta_e^3 \inner{\vct{w}^{(t)}}{\vct{v}_e}^2} &\geq \frac{\beta_d^3 \inner{\vct{w}^{(t)}}{\vct{v}_d}^2}{\beta_e^3 \inner{\vct{w}^{(t)}}{\vct{v}_e}^2}
    \end{align}
    which is trivial because $\beta_d' > \beta_d$. Similarly, we can verify the result for SAM. Now, let's find the new coefficient $\beta_d' = k \beta_d (k > 1)$ such that training one-step GD on $\gD(\beta_e, \beta_d', \alpha)$ can recover the normalized gradient of the one-step SAM update on the original data distribution $\gD(\beta_e, \beta_d, \alpha)$. Using Equation~\ref{eq:ratio_of_coefficients}, we have
    \begin{align}
        \frac{\beta_d^3 \inner{\vct{w}_{\vct{\epsilon}}^{(t)}}{\vct{v}_d}^2}{\beta_e^3 \inner{\vct{w}_{\vct{\epsilon}}^{(t)}}{\vct{v}_e}^2} &= \frac{(\beta_d')^3 \inner{\vct{w}^{(t)}}{\vct{v}_d}^2}{\beta_e^3 \inner{\vct{w}^{(t)}}{\vct{v}_e}^2} \\
        \frac{\inner{\vct{w}_{\vct{\epsilon}}^{(t)}}{\vct{v}_d}}{\inner{\vct{w}_{\vct{\epsilon}}^{(t)}}{\vct{v}_e}} &= k^{3/2} \frac{\inner{\vct{w}^{(t)}}{\vct{v}_d}}{\inner{\vct{w}^{(t)}}{\vct{v}_e}} \\
        \frac{\inner{\vct{w}^{(t)}}{\vct{v}_d} - 3 \rho^{(t)} \beta_d^3 \inner{\vct{w}^{(t)}}{\vct{v}_d}^2}{\inner{\vct{w}^{(t)}}{\vct{v}_e} - 3 \rho^{(t)} \alpha \beta_e^3 \inner{\vct{w}^{(t)}}{\vct{v}_e}^2} &= k^{3/2} \frac{\inner{\vct{w}^{(t)}}{\vct{v}_d}}{\inner{\vct{w}^{(t)}}{\vct{v}_e}} \\
        k^{3/2} &= \frac{1 - 3 \rho^{(t)} \beta_d^3 \inner{\vct{w}^{(t)}}{\vct{v}_d}}{1 - 3 \rho^{(t)} \alpha \beta_e^3 \inner{\vct{w}^{(t)}}{\vct{v}_e}} \\ 
        k &= \left( \frac{1 - 3 \rho^{(t)} \beta_d^3 \inner{\vct{w}^{(t)}}{\vct{v}_d}}{1 - 3 \rho^{(t)} \alpha \beta_e^3 \inner{\vct{w}^{(t)}}{\vct{v}_e}} \right)^{2/3}.
    \end{align}
    Therefore, with $\beta_d' = \left( \frac{1 - 3 \rho^{(t)} \beta_d^3 \inner{\vct{w}^{(t)}}{\vct{v}_d}}{1 - 3 \rho^{(t)} \alpha \beta_e^3 \inner{\vct{w}^{(t)}}{\vct{v}_e}} \right)^{2/3} \beta_d$, the normalized gradient of the one-step GD update on $\gD(\beta_e, \beta_d', \alpha)$ is similar to that of the one-step SAM update on $\gD(\beta_e, \beta_d, \alpha)$.
\end{proof}

\section{Auxiliary Lemmas}
\begin{lemma}[Claim D.20,~\cite{allen2020towards}]\label{lem:tensor_power_method_bound}
    Considering an increasing sequence $x_t \geq 0$ defined as $x_{t+1} = x_t + \eta C_t (x_t - \rho)^2$ for some $C_t = \Theta(1), 0 \leq \rho \leq x_0$, then we have for every $A > x_0$, every $\delta \in (0, 1)$, and every $\eta \in (0, 1]$:
    \begin{align}
        \sum_{t \geq 0, x_t \leq A} \eta C_t &\leq \frac{1 + \delta}{x_0} + \frac{O(\eta (A - \rho)^2)}{x_0^2} \frac{\log(A/x_0)}{\log(1 + \delta)} \\
        \sum_{t \geq 0, x_t \leq A} \eta C_t &\geq \frac{1 - \frac{(1 + \delta) x_0}{A}}{x_0 (1 + \delta)} - \frac{O(\eta (A - \rho)^2)}{x_0^2} \frac{\log(A/x_0)}{\log(1 + \delta)} 
    \end{align}
\end{lemma}

\begin{proof}
    For every $g = 0, 1, \ldots$, let $T_g$ be the first iteration such that $x_t \geq (1 + \delta)^g x_0$. Let $b$ be the smallest integer such that $(1 + \delta)^b x_0 \geq A$. Suppose for notation simplicity that we replace $x_t$ with exactly $A$ whenever $x_t \geq A$. By the definition of $T_g$, we have
    \begin{align}
        \sum_{t \in [T_g, T_{g+1})} \eta C_t [(1 + \delta)^g x_0]^2 &\leq x_{T_{g+1}} - x_{T_{g}} \leq \delta (1 + \delta)^g x_0 + O(\eta (A - \rho)^2) \\
         \sum_{t \in [T_g, T_{g+1})} \eta C_t [(1 + \delta)^{g+1} x_0]^2 &\geq x_{T_{g+1}} - x_{T_{g}} \geq \delta (1 + \delta)^g x_0 - O(\eta (A - \rho)^2) \\
    \end{align}
    These imply that
    \begin{align}
        \sum_{t \in [T_g, T_{g+1})} \eta C_t &\leq \frac{\delta}{(1 + \delta)^g x_0} + \frac{O(\eta (A - \rho)^2)}{x_0^2} \\
        \sum_{t \in [T_g, T_{g+1})} \eta C_t &\geq \frac{\delta}{(1 + \delta)^{g+2} x_0} - \frac{O(\eta (A - \rho)^2)}{x_0^2}
    \end{align}
    Recall $b$ is the smallest integer such that $(1 + \delta)^b x_0 \geq A$, so we can calculate
    \begin{align}
        \sum_{t \geq 0, x_t \leq A} \eta C_t &\leq \sum_{g=0}^{b-1} \frac{\delta}{(1 + \delta)^g x_0} + \frac{O(\eta (A - \rho)^2)}{x_0^2} b \\
        &= \frac{\delta}{1 - \frac{1}{1 + \delta}} \frac{1}{x_0} + \frac{O(\eta (A - \rho)^2)}{x_0^2} b \\
        &= \frac{1 + \delta}{x_0} + \frac{O(\eta (A - \rho)^2)}{x_0^2} \frac{\log(A/x_0)}{\log(1 + \delta)} \\
        \sum_{t \geq 0, x_t \leq A} \eta C_t &\geq \sum_{g=0}^{b-2} \frac{\delta}{(1 + \delta)^{g+2} x_0} - \frac{O(\eta (A - \rho)^2)}{x_0^2} b \\
        &= \frac{\delta (1 + \delta)^{-1} (1- \frac{1}{(1 + \delta)^{(b - 1)}})}{1 - \frac{1}{1 + \delta}} \frac{1}{x_0} - \frac{O(\eta (A - \rho)^2)}{x_0^2} b \\
        &= \frac{1 - \frac{(1 + \delta) x_0}{A}}{x_0 (1 + \delta)} - \frac{O(\eta (A - \rho)^2)}{x_0^2} \frac{\log(A/x_0)}{\log(1 + \delta)}
    \end{align}
    Thus, the two desired inequalities are proved.
\end{proof}

\begin{lemma}[Lemma D.19,~\cite{allen2020towards}.]\label{lem:grow_speed_two_sequences}
    Let $\{x_t, y_t\}_{t=1, \ldots}$ be two positive sequences that satisfy
    \begin{align*}
        x_{t+1} &\geq x_t + \eta \cdot C_t (x_t - \rho)^2, \\
        y_{t+1} &\leq y_t + S\eta \cdot C_t y_t^2,
    \end{align*}
    for some $C_t = \Theta(1)$. Suppose $x_0 \geq y_0 S \frac{1+2G}{1-3G}$ where $S \in (0, 1), G \in (0, 1/3)$ and $0 < \eta \leq \min\{ \frac{G^2 x_0}{\log(A/x_0)}, \frac{G^2 y_0}{\log(1/G)}\}, 0 \leq \rho < O(x_0)$, and for all $A \in (x_0, O(1)]$, let $T_x$ be the first iteration such that $x_t \geq A$. Then, we have 
    $y_{T_x} \leq O(G^{-1} y_0)$.
\end{lemma}

\begin{proof}
    Let $T_x$ be the first iteration $t$ in which $x_t \geq A$. Apply Lemma~\ref{lem:tensor_power_method_bound} for the $x_t$ sequence with $C_t = C_t$ and threshold $A$, we have
    \begin{align}
        \sum_{t=0}^{T_x} \eta C_t &\leq \frac{1 + \delta}{x_0} + \frac{O(\eta (A - \rho)^2)}{x_0^2} \frac{\log(A/x_0)}{\log(1 + \delta)} \\
        &= \frac{1 + \delta}{x_0} + O\left( \frac{\eta (A - \rho)^2 \log(A/x_0)}{\delta x_0^2} \right) \\
        &\leq \frac{1 + \delta}{x_0} + O\left( \frac{\eta \log(A/x_0)}{\delta x_0^2} \right) \label{eq:bound_tx}
    \end{align}
    Let $T_y$ be the first iteration $t$ in which $y_t \geq A$. Apply Lemma~\ref{lem:tensor_power_method_bound} for the $y_t$ sequence with $\eta = S \eta, C_t = C_t, \rho = 0$ and threshold $A' = G^{-1} y_0$, we have
    \begin{align}
        \sum_{t=0}^{T_y} S\eta C_t &\geq \frac{1 - \frac{(1 + \delta) y_0}{A'}}{y_0 (1 + \delta)} - \frac{O(S\eta (A')^2)}{y_0^2} \frac{\log(A'/y_0)}{\log(1 + \delta)} \\
        &\geq \frac{1 - O(\delta + G)}{y_0} - O \left( \frac{S\eta (A')^2 \log(1/G)}{\delta y_0^2} \right) \\
        &\geq \frac{1 - O(\delta + G)}{y_0} - O \left( \frac{S\eta \log(1/G)}{\delta y_0^2} \right) \label{eq:bound_ty}
    \end{align}
    Compare Equation~\ref{eq:bound_tx} and~\ref{eq:bound_ty}. Choosing $\delta = G$ and $\eta \leq \min\{ \frac{G^2 x_0}{\log(A/x_0)}, \frac{G^2 y_0}{\log(1/G)}\}$, together with $x_0 \geq y_0 S \frac{1+2G}{1-3G}$ we have $T_x \leq T_y$.
\end{proof}

\begin{lemma}[Lemma K.15,~\cite{jelassi2022towards}.]\label{lem:grow_speed_one_sequence_two_equations}
    Let $\{z_t\}_{t=0}^T$ be a positive sequence defined by the following recursions
    \begin{align*}
        z_{t+1} &\geq z_t + m (z_t - \rho)^2, \\
        z_{t+1} &\leq z_t + M (z_t)^2,
    \end{align*}
    where $z_0 > \rho \geq 0$ is the initialization and $m, M > 0$ are some constants. Let $v > z_0$, then the time $T_v$ such that $z_{T_v} \geq v$ for all $t \geq T_v$ is
    \begin{equation}
        T_v = \frac{2 z_0}{m (z_0 - \rho)^2} + \frac{4M z_0^2}{m (z_0 - \rho)^2} \Big\lceil \frac{\log(v/z_0)}{\log(2)}\Big\rceil.
    \end{equation}
\end{lemma}

\begin{proof}
    Let $n \in \sN^\star$. Let $T_n$ be the first time that $z_t \geq 2^n z_0$. We want to find an upper bound of $T_n$. We start with the case $n = 1$. By summing the recursion, we have:
    \begin{align}
        z_{T_1} \geq z_0 + m \sum_{t = 0}^{T_1 - 1} (z_t - \rho)^2
    \end{align}
    Because $z_t \geq z_0$, we obtain
    \begin{align}\label{eq:bound_t1}
        T_1 \leq \frac{z_{T_1} - z_0}{m (z_0 - \rho)^2}
    \end{align}
    Now, we want to bound $z_{T_1} - z_0$. Using again the recursion and $z_{T_1 - 1} \leq 2 z_0$, we have
    \begin{align}\label{eq:bound_zt1}
        z_{T_1} \leq z_{T_1 - 1} + M (z_{T_1 - 1})^2 \leq 2 z_0 + 4M z_0^2.
    \end{align}
    Combining Equation~\ref{eq:bound_t1} and~\ref{eq:bound_zt1}, we get a bound on $T_1$ as
    \begin{align}
        T_1 \leq \frac{z_0 + 4M z_0^2}{m (z_0 - \rho)^2} = \frac{z_0}{m (z_0 - \rho)^2} + \frac{4M z_0^2}{m (z_0 - \rho)^2}
    \end{align}
    Now, let's find a bound for $T_n$. Starting from the recursion and using the fact that $z_t \geq 2^{n-1} z_0$ for $t \geq T_{n-1}$ we have
    \begin{align}\label{eq:bound_tn}
        z_{T_n} &\geq z_{T_{n-1}} + m \sum_{t = T_{n-1}}^{T_n - 1} (z_t - \rho)^2 \\
        &\geq z_{T_{n-1}} + (2^{n-1})^2 m (z_0 - \rho)^2 (T_n - T_{n-1})
    \end{align}
    On the other hand, by using $z_{T_n - 1} \leq 2^n z_0$ we upper bound $z_{T_n}$ as
    \begin{align}
        z_{T_n} \leq z_{T_n - 1} + M(z_{T_n - 1})^2 \leq 2^n z_0 + M 2^{2n} z_0^2 
    \end{align}
    Besides, we know that $z_{T_{n-1}} \geq 2^{n-1} z_0$. Therefore, we upper bound $z_{T_n} - z_{T_{n-1}}$ as
    \begin{align}\label{eq:bound_ztn}
        z_{T_n} - z_{T_{n-1}} \leq 2^{n-1} z_0 + M 2^{2n} z_0^2 
    \end{align}
    Combining Equations~\ref{eq:bound_tn} and~\ref{eq:bound_ztn} yields
    \begin{align}\label{eq:bound_tn_recursion}
        T_n &\leq T_{n-1} + \frac{2^{n-1} z_0 + M 2^{2n} z_0^2}{(2^{n-1})^2 m (z_0 - \rho)^2} \\
        &= T_{n-1} + \frac{z_0}{2^{n-1} m (z_0 - \rho)^2} + \frac{4M z_0^2}{m (z_0 - \rho)^2}
    \end{align}
    Summing Equation~\ref{eq:bound_tn_recursion} for $n = 2, \ldots, n$ we have
    \begin{align}\label{eq:bound_tn_final}
        T_n \leq \sum_{i = 1}^n \frac{z_0}{2^{i-1} m (z_0 - \rho)^2} + \frac{4Mn z_0^2}{m (z_0 - \rho)^2}  \leq \frac{2 z_0}{m (z_0 - \rho)^2} + \frac{4Mn z_0^2}{m (z_0 - \rho)^2}
    \end{align}
    Lastly, we know that $2^n z_0 \geq v$ this implies that we can set $n = \Big\lceil \frac{\log(v/z_0)}{\log(2)}\Big\rceil$ in Equation~\ref{eq:bound_tn_final}.
\end{proof}

We make the following assumptions for every $t \leq T$ as the same in~\cite{jelassi2022towards}.

\begin{lemma}[Induction hypothesis D.1,~\cite{jelassi2022towards}]\label{lem:bound_noise_gd}
    Throughout the training process using GD for $t \leq T$, we maintain that, for every $i$ and $j \in [J]$,
    \begin{equation}
        |\inner{\vct{w}_j^{(t)}}{\vct{\xi}_i}| \leq \tilde{O}(\sigma_0 \sigma_p \sqrt{d}).
    \end{equation}
\end{lemma}

\begin{lemma}[Lemma G.4,~\cite{deng2023robust}]\label{lem:approx_deriv_data_gd}
    For every $i$, we have $l_i^{(t)} = \Theta(1) g_1(t)$, where 
    \begin{equation}
        g_1(t) = \text{sigmoid} \left( \sum_{j=1}^J -\beta_d^3 \inner{\vct{w}_j^{(t)}}{\vct{v}_d}^3 - \beta_e^3 \inner{\vct{w}_j^{(t)}}{\vct{v}_e}^3 \right).
    \end{equation}
\end{lemma}

\begin{lemma} [Lemma K.5, \cite{jelassi2022towards}] Let $X \in \R^d$ be a Gaussian random vector, $X \sim \mathcal{N}(0, \sigma^2 I_d)$. Then with probability at least $1 - o(1)$, we have $\|X\|_2^2 = \Theta(\sigma^2 \sqrt{d})$.
\end{lemma}

\begin{lemma} [Lemma K.7, \cite{jelassi2022towards}] \label{lem:inner_product_tail_bound}
Let $X$ and $Y$ be independent Gaussian random vectors on $\R^d$ and $X \sim \mathcal{N}(0, \sigma^2 \mtx{I}_d)$, $Y \sim \mathcal{N}(0, \sigma_0^2 \mtx{I}_d)$. Assume that $\sigma \sigma_0 \leq \frac{1}{d}$. Then, with probability at least $1 - \delta$, we have
\begin{align*}
    |\inner{X}{Y}| \leq \sigma \sigma_0 \sqrt{2d \log{\frac{2}{\delta}}}
\end{align*}
\end{lemma}

\begin{lemma} [Bound on noise inner products] \label{lem:inner_product_tail_bound_extended}
Let $N = O(poly(d))$. The following hold with probability at least $1 - o(1)$:
\begin{align*}
    \max\left\{|\inner{\vct{w}_{j, \vct{\epsilon}}^{(0)}}{\vct{\xi}_i}|\right\} = \tilde{O}(\sigma \sigma_0 \sqrt{d}) \\
    \max_i\left\{\frac{1}{n}\sum_{k=1}^n |\inner{\vct{\xi}_k}{\vct{\xi}_i}|\right\} = \tilde{O}(\frac{\sigma^2 d}{N}  + \sigma^2 \sqrt{d})
\end{align*}
\end{lemma}
\begin{proof}
For the first inequality, Lemma \ref{lem:inner_product_tail_bound} implies that with probability at least $1 - \frac{1}{dN}$,
\begin{equation}
    |\inner{\vct{w}_{j, \vct{\epsilon}}^{(0)}}{\vct{\xi}_i}| \leq \sigma \sigma_0 \sqrt{2d \log(\frac{2}{dN}}) = \tilde{O}(\sigma \sigma_0 \sqrt{d})
\end{equation}
Taking a union bound over $n = 1, \dots, N$ gives the result.

The second statement is proved similarly.
\end{proof}

\begin{lemma}[Bound on the noise component for SAM]\label{lem:bound_noise_sam}
    Assume that $\rho = o(\sigma_0)$ and $\omega(1) \leq N \leq O(poly(d))$. Throughout the training process using SAM for $t \leq T$, we maintain that, for every $i$ and $j \in [J]$,
    \begin{align}
        |\inner{\vct{w}_{j, \vct{\epsilon}}^{(t)}}{\vct{\xi}_i}| &\leq \tilde{O}(\sigma \sigma_0 \sqrt{d})
    \end{align}
\end{lemma}

\begin{proof}
Let $\chi_t = \max\{|\inner{\vct{w}_{j, \vct{\epsilon}}^{(t)}}{\vct{\xi}_i}|\}$, $\alpha = \max_i\{\frac{1}{n}\sum_{k=1}^n |\inner{\vct{\xi}_k}{\vct{\xi}_i}|\}$. Combined with $l_{k,\epsilon}^{(t)} \leq 1$, the noise gradient update rule can be bounded as
$$\chi_{t+1} \leq \chi_t + 3\eta \alpha \chi_t^2$$
Suppose that $a(t)$ satisfies the differential equation
\begin{align*}
    a' &= 3\alpha \eta a^2 \\
    a(0) &= \chi_0
\end{align*}
Observe that $a(t)$ is increasing so, by the Mean Value Theorem there exists $\tau \in (t, t+1)$  such that
\begin{align*}
    a(t+1) - a(t) &= a'(\tau) \\
    &= 3 \alpha \eta a(\tau)^2 \\
    &\geq 3 \alpha \eta a(t)^2
\end{align*}
So an \fast induction shows that $a(t) \geq \chi_t$. 

Now solving for $a(t)$,
$$a(t) = \frac{1}{\frac{1}{a(0)} - 3\alpha \eta t}, \qquad t \leq \frac{1}{3 \alpha \eta a(0)}.$$

Using the high probability tail bounds \ref{lem:inner_product_tail_bound_extended}, 
\begin{align*}
    a(0) &= \chi_0 = \tilde{O}(\sigma \sigma_0 \sqrt{d}) \\
    \alpha &= \tilde{O}(\frac{\sigma^2 d}{N} + \sigma^2 \sqrt{d})
\end{align*}
where $\sigma = \frac{\sigma_p}{\sqrt{d}}$. Substituting these bounds gives
$$a(t) = \frac{1}{\tilde{\Omega}(\frac{1}{\sigma \sigma_0 \sqrt{d}}) - \tilde{O}(\eta t(\frac{\sigma^2 d}{N} + \sigma^2 \sqrt{d}))}$$
Now Theorem \ref{the:easy_difficult_sam} and $\rho = o(\sigma_0)$ implies that $\eta T_0 = \tilde{\Theta}(\frac{1}{\sigma_0 \beta_\epsilon^3})$. Combined with the assumption that $N = \Omega(1),$ the second term in the denominator is of lower order than the first term, so
$$a(T_0) = \tilde{O}(\sigma \sigma_0 \sqrt{d}).$$
We conclude that
\begin{equation}
    |\inner{\vct{w}_j^{(t)}}{\vct{\xi}_i}| \leq \tilde{O}(\sigma_0 \sigma_p \sqrt{d}).
\end{equation}

\end{proof}

\begin{lemma}\label{lem:approx_deriv_data_sam}
    For every $i$, we have $l_i^{(t)} = \Theta(1) g_1(t)$ and $l_{i, \vct{\epsilon}}^{(t)} = \Theta(1) g_2(t)$, where 
    \begin{align}
        g_1(t) &= \text{sigmoid} \left( \sum_{j=1}^J -\beta_d^3 \inner{\vct{w}_j^{(t)}}{\vct{v}_d}^3 - \beta_e^3 \inner{\vct{w}_j^{(t)}}{\vct{v}_e}^3 \right), \\
        g_2(t) &= \text{sigmoid} \left( \sum_{j=1}^J -\beta_d^3 \inner{\vct{w}_{j, \vct{\epsilon}}^{(t)}}{\vct{v}_d}^3 - \beta_e^3 \inner{\vct{w}_{j, \vct{\epsilon}}^{(t)}}{\vct{v}_e}^3 \right).
    \end{align}
\end{lemma}

The proof is the same as~\cite{deng2023robust}[Lemma G.4].

\section{Additional Experimental Settings}\label{app:additional_settings}

\begin{figure*}[t!]
    \centering
    \includegraphics[width=\textwidth]{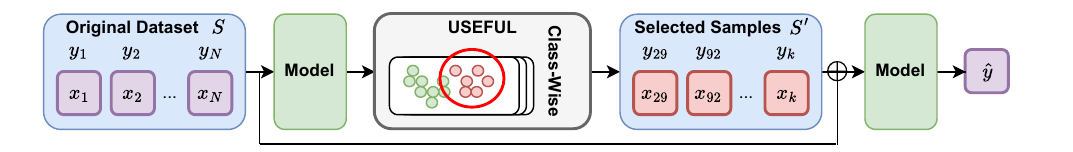}
    \vspace{-4mm}
    \caption{%
    \ourmethod\ first trains the model for a few epochs $t$, which in practice is around 5-10\% of the total training epochs. It then clusters examples in every class into 2 groups and upsamples %
    the cluster with higher average loss. Finally, the base model is retrained from scratch on the modified data distribution. }
    \label{fig:method_overview}
    \vspace{-3mm}
\end{figure*}

\subsection{Datasets and Training Details}\label{app:experiment_settings}
\textbf{Datasets.} The CIFAR10 dataset~\cite{krizhevsky2009learning} consists of 60,000 32 × 32 color images in 10 classes, with 6000 images per class. The CIFAR100 dataset~\cite{krizhevsky2009learning} is just like the CIFAR10, except it has 100 classes containing 600 images each. For both of these datasets, the training set has 50,000 images (5,000 per class for CIFAR10 and 500 per class for CIFAR100) with the test set having 10,000 images. CINIC10~\cite{darlow2018cinic} represents an image classification dataset consisting of 270,000 images, which is 4.5 times larger than CIFAR10. The dataset is created by merging CIFAR10 with images extracted from the ImageNet database, specifically selecting and downsampling images from the same 10 classes present in CIFAR10. Tiny-ImageNet~\cite{le2015tiny} comprises 100,000 images distributed across 200 classes of ImageNet~\cite{deng2009imagenet}, with each class containing 500 images. These images have been resized to 64×64 dimensions and are in color. The dataset consists of 500 training images, 50 validation images, and 50 test images per class. The STL10 dataset~\cite{coates2011analysis} includes 5000 96x96 training labeled images, 500 per CIFAR10 class. The test set consists of 800 images per class, this counts up to 8,000 images in total. \looseness=-1

\textbf{Training on different datasets.} Follow the setting from~\cite{andriushchenko2022towards,andriushchenko2023sharpness}, we trained Pre-Activation ResNet18 on all datasets except for CIFAR100 which was trained with ResNet34. We trained our models for 200 epochs with a batch size of 128 and used basic data augmentations such as random mirroring and random crop. We used SGD with the momentum parameter of 0.9 and set weight decay to 0.0005. We also fixed $\rho = 0.1$ for SAM unless further specified. For all datasets, we used a learning rate schedule where we set the initial learning rate to 0.1. The learning rate is decayed by a factor of 10 after 50\% and 75\% epochs, i.e., we set the learning rate to 0.01 after 100 epochs and to 0.001 after 150 epochs. \looseness=-1

\textbf{Training with different architectures.} We used the same training procedures for Pre-Activation ResNet18, VGG19, and DenseNet121. We directly used the official Pytorch~\cite{paszke2019pytorch} implementation for VGG19 and DenseNet121. For 3-layer MLPs, we used a hidden size of 512 with a dropout of 0.1 to avoid overfitting and set $\rho = 0.01$. For ViT-S~\cite{yuan2021tokens}, we adopted a Pytorch implementation at \href{https://github.com/lucidrains/vit-pytorch}{https://github.com/lucidrains/vit-pytorch}.
In particular, the hidden size, the depth, the number of attention heads, and the MLP size are set to 768, 8, 8, and 2304, respectively. We adjusted the patch size to 4 to fit the resolution of CIFAR10 and set both the initial learning rate and $\rho$ to 0.01.

\textbf{Computational resources.} Each model is trained on 1 NVIDIA RTX A5000 GPU. 

\subsection{Other Implementation Details}\label{app:implementation_details}

\textbf{When to separate the examples?} We selected the best-separating epoch $t$ in the set of $\{4, 5, 6, 7, 8, 10\}$ for CIFAR10 and $\{12, 14, 16, 18, 20, 22\}$ for CIFAR100. 
Particularly, we separated examples of CIFAR10 at around epoch 8 while that of CIFAR100 is near epoch 20. Near this point, the gain in training error diminishes significantly as shown in Figure~\ref{fig:training_error}, which shows a sign that the model successfully learns \fast features. In addition, we reported the results for different separating epochs in Appendix~\ref{app:ablation_studies}. In Figure~\ref{fig:vary_datasets}, the best separating epochs for STL10, CINIC10, and Tiny-ImageNet are 11, 4, and 10, respectively. The separating epoch for Waterbirds is 5.\looseness=-1

\textbf{Forgetting score.} To compute forgetting scores of training examples in each dataset, we collected the same statistics as in~\cite{toneva2018empirical} but computed at the end of each epoch. The reason is to make the statistics consistent between two versions of the same \slow example which is repeated in the upsampled dataset. \looseness=-1

\textbf{Hessian spectra.} We approximated the density of the Hessian spectrum using the Lanczos algorithm~\cite{papyan2018full,ghorbani2019investigation}. The Hessian matrix is approximated by 1000 examples (100 per class of CIFAR10). Then we extract the top eigenvalues to calculate the maximum Hessian eigenvalue $(\lambda_{\text{max}})$ and the bulk of spectra $(\lambda_{\text{max}}/\lambda_5)$~\cite{jastrzebski2020break}.

\section{Additional Results}\label{app:additional_results}

\begin{figure}[t]
    \centering
    \includegraphics[width=0.5\textwidth]{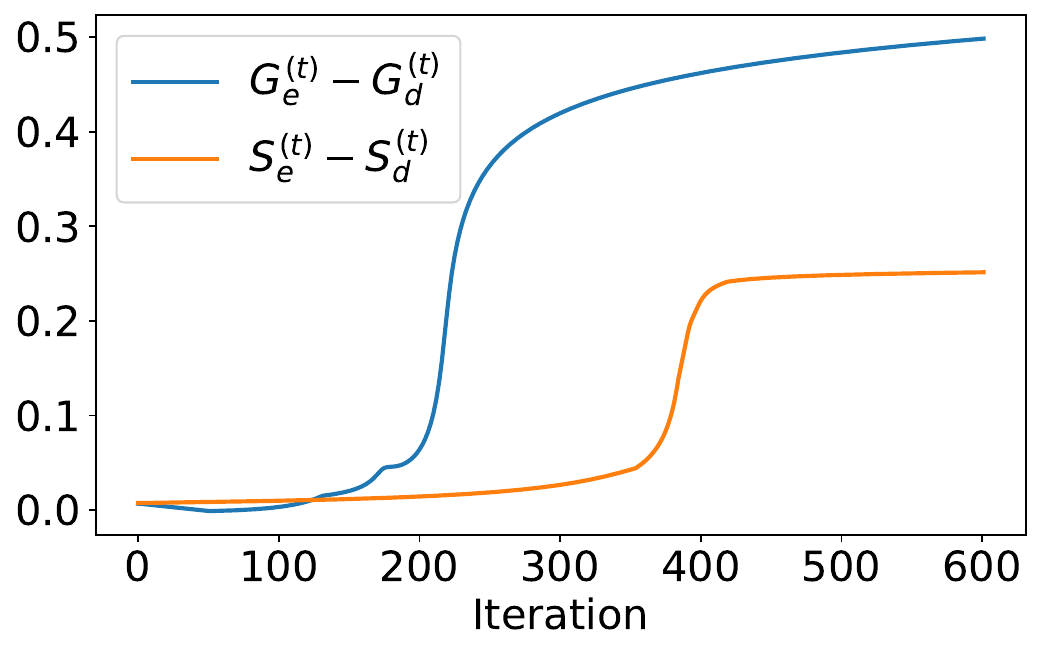} 
    \vspace{-2mm}
    \caption{The gap between contribution of \fast and \slow features towards the model output in SAM and GD. The toy datasets is generated from the distribution in Definition~\ref{def:data_distribution} with $\beta_d = \beta_e = \alpha = 1$.}
    \label{fig:toy_dataset_extreme_setting}
    \vspace{-3mm}
\end{figure}

\subsection{An Extreme Case of Toy Datasets}
We consider an extreme setting when two features have identical strength and no missing \fast features, i.e., $\beta_e = \beta_d = \alpha = 1$, the gap between LHS and RHS in Equation~\ref{eq:contribution_gap} is not small as shown in the figure~\ref{fig:toy_dataset_extreme_setting}. The gap is consistently around 0.2-0.3 from epoch 250 onwards. 

\subsection{\ourmethod~is Useful for MLP}\label{app:mlp}
\begin{figure*}[t!]
    \centering
    \includegraphics[width=0.8\textwidth]{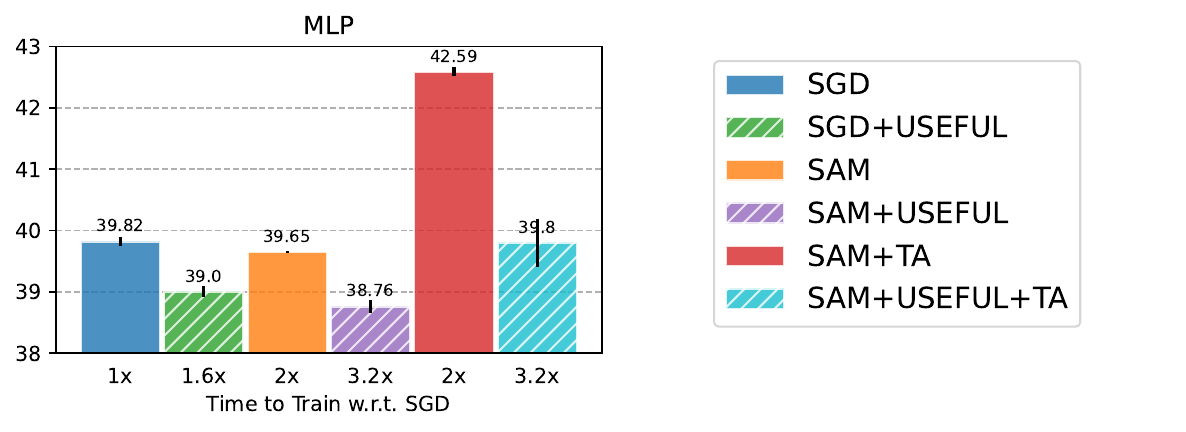} 
    \vspace{-2mm}
    \caption{\textbf{{Test classification errors of 3-layer MLP on CIFAR10.}} The number below each bar indicates the estimated cost to train the model and the tick on top shows the standard deviation over three runs. \ourmethod~improves the performance of SGD and SAM when training with 3-layer MLP.}
    \label{fig:vary_architectures_app}
    \vspace{-2mm}
\end{figure*}

\begin{figure}[t!]
    \centering
    \includegraphics[width=0.8\textwidth]{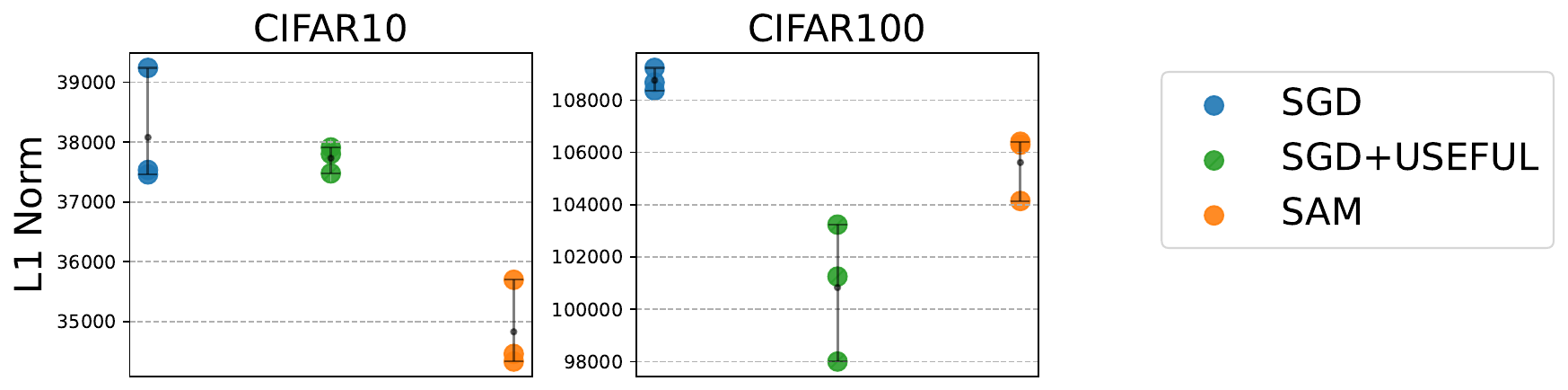}  
    \caption{\textbf{{L1 norm of ResNet18 trained on CIFAR10 and ResNet34 trained on CIFAR100.}} Lower L1 norm indicates a sparser solution and stronger implicit regularization properties~\cite[Section 4.2]{andriushchenko2022towards}. SAM has a lower L1 norm than SGD, and \ourmethod~further reduces the L1 norm of SGD and SAM.}
    \label{fig:l1_scatter}
\end{figure}

\begin{table}[!t]
    \caption{\textbf{{Sharpness of solution at convergence.}} We train ResNet18 on CIFAR10 and measure the maximum Hessian eigenvalue $\lambda_{\text{max}}$ and the bulk spectra measured as $\lambda_{\text{max}}/\lambda_5$.}
    \label{tab:compare_sharpness}
    \begin{center}
    \begin{sc}
    \scalebox{1.0}{
        \begin{tabular}{lccc}
        \toprule
        Metric & SGD & SGD+\ourmethod & SAM \\
        \midrule
        $\lambda_{\textnormal{max}}$ & 53.8 & 41.8 & 12.4 \\
        $\lambda_{\textnormal{max}}/\lambda_5$ & 3.8 & 1.5 & 2.4 \\ 
        \bottomrule
        \end{tabular}
    }
    \end{sc}
    \end{center}
\end{table}

\begin{figure}[t]
    \centering
    \includegraphics[width=0.9\textwidth]{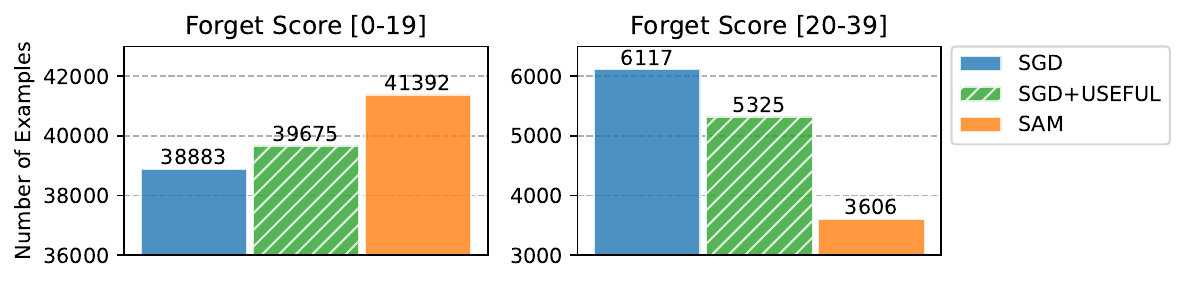} 
    \vspace{-2mm}
    \caption{\textbf{{Forgetting scores for training ResNet18 on CIFAR10.}} Forgetting scores measure the learning speed of examples in training data. \ourmethod\ approaches the training dynamics of SAM, with more examples being forgotten infrequently and fewer examples being forgotten frequently.}
    \label{fig:forget_score_bars}
    \vspace{-2mm}
\end{figure}

Figure~\ref{fig:vary_architectures_app} shows that for 3-layer MLP, \ourmethod~successfully reduces the test error of both SGD and SAM by nearly 1\%. Additionally, SGD+\ourmethod~yields better performance than SAM alone. 

\subsection{SAM \& \ourmethod~Reduce Forgetting Scores}\label{app:close_to_sam}
We used forgetting scores~\cite{toneva2018empirical} to partition examples in CIFAR10 into different groups. Forgetting scores count the number of times an example is misclassified after being correctly classified during training and is an indicator of the learning-speed of examples.
Figure~\ref{fig:forget_score_bars} illustrates that SGD+\ourmethod~and SAM have fewer examples with high forgetting scores than SGD does. This aligns with our theoretical analysis in Theorem~\ref{the:sam_learn_more_uniform_than_gd_main} and results on the toy datasets. By upsampling \slow examples in the dataset, they contribute more to learning and hence SGD+\ourmethod~learns \slow features faster than SGD. 

Furthermore, Tables~\ref{tab:fs_cifar10} and~\ref{tab:fs_cifar100} illustrate the average forgetting score, first learned iteration (i.e., at this epoch, the model predicts correctly for the first time) and iteration learned (i.e., after this epoch, the prediction is always correct) of examples in \fast and \slow clusters. Iteration learned is highly correlated with prediction depth~\cite{baldock2021deep}, which is another notion of data difficulty. 
It can be seen clearly that examples in \fast clusters have a lower difficulty score for every metric, indicating that USEFUL successfully identifies \fast examples early in training.

\begin{table}[!t]
    \caption{Average scores for two clusters on CIFAR10.}
    \label{tab:fs_cifar10}
    \begin{center}
    \begin{sc}
    \scalebox{1}{
    \begin{tabular}{l|ccc}
        \toprule
        Metric & \Fast & \Slow \\
        \midrule
        Forgetting score & 3.8 $\pm$ 6.1 & 14.7 $\pm$ 9.0 \\
        \midrule
        First learned iteration & 0.9 $\pm$ 1.2 & 3.7 $\pm$ 8.0 \\
        \midrule
        Iteration learned & 45.6 $\pm$ 50.0 & 105.8 $\pm$ 41.0 \\
        \bottomrule
    \end{tabular}
    }
    \end{sc}
    \end{center}
\end{table}

\begin{table}[!t]
    \caption{Average scores for two clusters on CIFAR100.}
    \label{tab:fs_cifar100}
    \begin{center}
    \begin{sc}
    \scalebox{1}{
    \begin{tabular}{l|ccc}
        \toprule
        Metric & \Fast & \Slow \\
        \midrule
        Forgetting score & 10.2 $\pm$ 8.6 & 16.8 $\pm$ 7.4 \\
        \midrule
        First learned iteration & 4.7 $\pm$ 6.8 & 9.8 $\pm$ 13.7 \\
        \midrule
        Iteration learned & 86.6 $\pm$ 46.1 & 115.7 $\pm$ 31.0 \\
        \bottomrule
    \end{tabular}
    }
    \end{sc}
    \end{center}
\end{table}

\begin{table}[!t]
    \caption{Test classification errors for training SAM and ASAM on the original CIFAR10 and modified datasets by \ourmethod. Results are averaged over 3 seeds.}
    \label{tab:asam}
    \begin{center}
    \begin{sc}
    \scalebox{1.0}{\begin{tabular}{lcc}
        \toprule
         & SAM & ASAM \\
        \midrule
        + & 4.23 $\pm$ 0.08 & 4.33 $\pm$ 0.19 \\
        + \ourmethod & \textbf{4.04 $\pm$ 0.06} & \textbf{4.09 $\pm$ 0.10} \\
        \midrule
        + TA & 4.06 $\pm$ 0.08 & 3.93 $\pm$ 0.11 \\
        + TA + \ourmethod & \textbf{3.49 $\pm$ 0.09} & \textbf{3.46 $\pm$ 0.01} \\
        \bottomrule
    \end{tabular}}
    \end{sc}
    \end{center}
\end{table}

\begin{figure*}[t!]
    \centering
    \includegraphics[width=\textwidth]{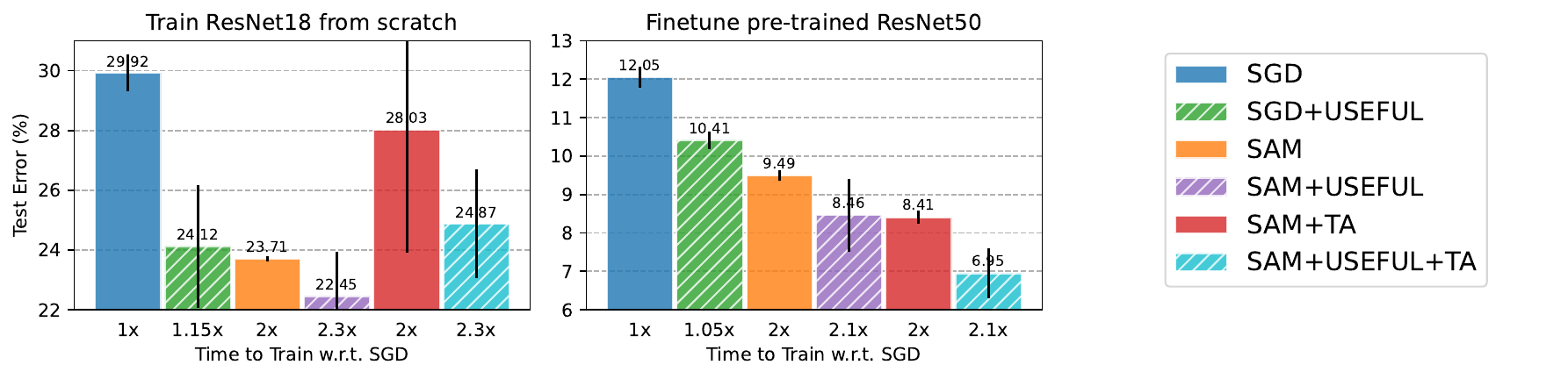} 
    \caption{\textbf{Comparing test classification errors on Waterbirds.} The number below each bar indicates the approximated cost to train the model and the tick on top shows the standard deviation over three runs. \ourmethod~boosts the performance of SGD and SAM on the balanced test set, showing its generalization to the OOD setting. In addition, the success of~\ourmethod~in fine-tuning reveals its new application to the transfer learning setting. 
    }
    \label{fig:ood_results}
\end{figure*}

\begin{table}[!t]
    \caption{Comparison between \ourmethod~and upweighting loss regarding test classification errors. Upweighting loss doubled the loss for all examples in the \slow clusters found by \ourmethod, which is different from dynamically upweighting examples during the training~\cite{zhai2022understanding}. Results are averaged over 3 seeds.}
    \label{tab:upweighting_loss}
    \begin{center}
    \begin{sc}
    \begin{tabular}{lcc|cc}
        \toprule
        Method & \multicolumn{2}{c}{SGD} & \multicolumn{2}{c}{SAM} \\
        \cmidrule{2-5}
        & CIFAR10 & CIFAR100 & CIFAR10 & CIFAR100 \\
        \midrule
        Upweighting loss & 5.33 $\pm$ 0.09 & 26.70 $\pm$ 3.25 & 4.28 $\pm$ 0.02 & 21.75 $\pm$ 0.25 \\
        \midrule
        \ourmethod & \textbf{4.79 $\pm$ 0.05} & \textbf{22.58 $\pm$ 0.08} & \textbf{4.04 $\pm$ 0.06} & \textbf{20.40 $\pm$ 0.36} \\
        \bottomrule
    \end{tabular}
    \end{sc}
    \end{center}
    \vspace{-3mm}
\end{table}

\subsection{\ourmethod~generalizes to other SAM variants}\label{app:sam_variants}
In this experiment, we show that SAM can also generalize to other variants of SAM. We chose ASAM, which is proposed by~\citeauthor{kwon2021asam} to address the sensitivity of parameter re-scaling~\cite{dinh2017sharp}. Following the recommended settings in ASAM, we trained it with a perturbation radius $\rho = 1.0$, which is 10 times that of SAM. Other settings are identical to the standard settings in Appendix~\ref{app:additional_settings}. Table~\ref{tab:asam} demonstrates the results for training ResNet18 on CIFAR10. Both SAM and ASAM can be combined with \ourmethod to improve the test classification error. When using TA, ASAM shows a slightly better performance than SAM.

\subsection{\ourmethod~shows promising results for the OOD settings}\label{app:ood_experiments}
While our main contribution is providing a novel and effective method to improve the in-distribution generalization performance, we conducted new experiments confirming the benefits of our method to improving out-of-distribution (OOD) generalization performance. As a few very recent works~\cite{cha2021swad,wang2023sharpness} showed the benefit of SAM in improving OOD performance, it is expected that \ourmethod~also extend to this setting. 

\textbf{Spurious correlation (Waterbirds).} As can be seen in Figure~\ref{fig:ood_results}, both SAM and \ourmethod~effectively improve the performance on the balanced test set though the model is trained on the spurious training set. When training ResNet18 with SGD from scratch, \ourmethod~decreases the classification errors by 5.8\%. In addition, it can be successfully applied to the pre-trained ResNet50 (on ImageNet), opening up a promising application to the transfer learning setting.

\textbf{Long-tail distribution (Long-tail CIFAR10).} We conducted new experiments on long-tail CIFAR10~\cite{cui2019class} with an imbalance ratio of 10. Table~\ref{tab:longtail_cifar10} shows that \ourmethod~can also improve the performance of SGD and SAM, by reducing the simplicity bias on the long-tail data. Figure~\ref{fig:long_tail_cifar10} visualizes the distribution of classes before and after upsampling by \ourmethod. Interestingly, we see that USEFUL upsamples more examples from some of the larger classes and still improves the accuracy on the balanced test set. This improvement is attributed to the more uniform speed of feature learning, and not balancing the training data distribution. Notably,~\ourmethod~outperforms class balancing to address long tail distribution. Besides,~\ourmethod~can be stacked with methods to address long-tail data to further improve performance, as we confirmed in the last row.

\begin{table}[!t]
    \caption{Test error on long-tailed CIFAR10. \textsc{Balancing} means that we upsampled small classes to make the classes balanced. Results are averaged over 3 seeds.}
    \label{tab:longtail_cifar10} 
    \begin{center}
    \begin{sc}
    \scalebox{1}{
    \begin{tabular}{ccccc}
        \toprule
        Ratio & SGD & SGD+\ourmethod & SAM & SAM+\ourmethod \\
        \midrule
        1:10 & 10.01 $\pm$ 0.21 & 9.53 $\pm$ 0.13 & 8.85 $\pm$ 0.08 & 8.22 $\pm$ 0.04 \\
        Balancing & 9.77 $\pm$ 0.17 & 9.25 $\pm$ 0.11 & 8.31 $\pm$ 0.11 & 7.93 $\pm$ 0.02 \\
        \bottomrule
    \end{tabular}
    }
    \end{sc}
    \end{center}
\end{table}

\begin{figure}[h]
    \begin{center}
    \includegraphics[width=0.5\textwidth]{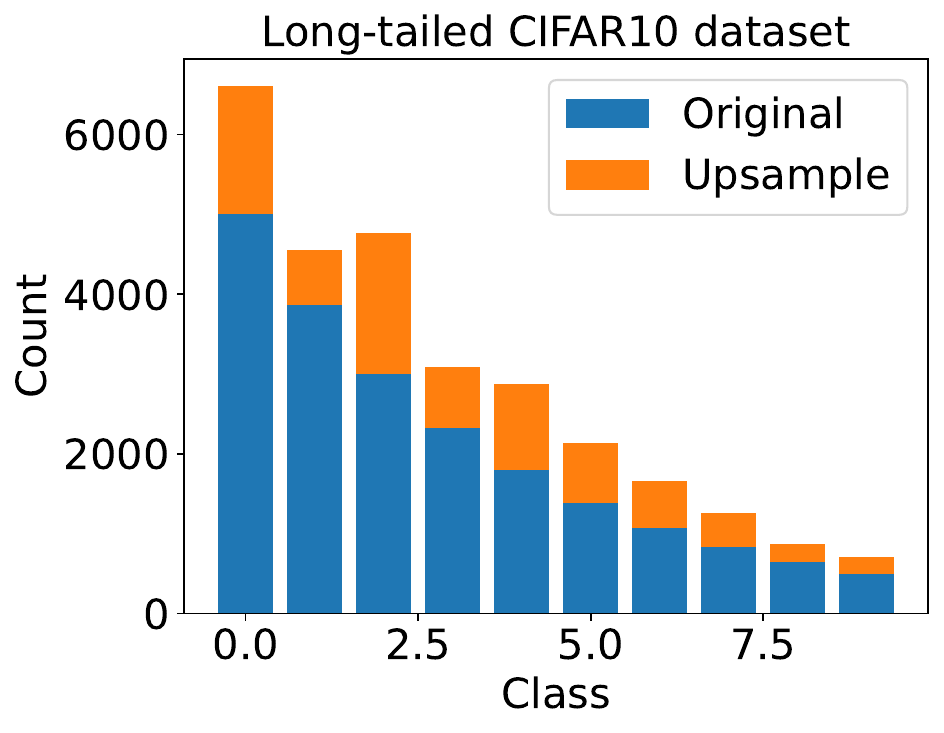}
    \caption{Class count before and after upsampling by \ourmethod~on long-tail CIFAR10 dataset.}
    \label{fig:long_tail_cifar10}
    \end{center}
\end{figure}

\begin{table}[!t]
    \caption{\textbf{Test classification errors of SGD} for different partition methods. Results are averaged over 3 seeds.}
    \label{tab:compare_partition_methods}
    \vspace{-3mm}
    \begin{center}
    \begin{sc}
    \begin{tabular}{lcc}
        \toprule
        Partition method & CIFAR10 & CIFAR100 \\
        \midrule
        Quantile & 5.27 $\pm$ 0.10 & 23.49 $\pm$ 0.82\\
        \midrule
        Misclassification & 4.98 $\pm$ 0.17 & 23.86 $\pm$ 0.70 \\
        \midrule
        \ourmethod & \textbf{4.79 $\pm$ 0.05} & \textbf{22.58 $\pm$ 0.08} \\
        \bottomrule
    \end{tabular}
    \end{sc}
    \end{center}
\end{table}

\subsection{\ourmethod~is also effective for noisy label data}\label{app:label_noise}
Our method and analysis consider a clean dataset. But, as we confirmed in Table~\ref{tab:labelnoise_cifar10}, \ourmethod~can easily stack with MixUp~\cite{zhang2017mixup}, a robust method for learning against noisy labels, to reduce the simplicity bias and improve their performance. Applying \ourmethod~on top of MixUp to CIFAR10 with 10-20\% random label flip  successfully boosts the performance of both SGD and SAM when training on data with corrupt labels.

\begin{table}[!t]
    \caption{Test error on label noise CIFAR10 (all methods are with MixUp). Results are averaged over 3 seeds.}
    \label{tab:labelnoise_cifar10}
    \begin{center}
    \begin{sc}
    \scalebox{1}{
    \begin{tabular}{ccccc}
        \toprule
        Rate & SGD & SGD+\ourmethod & SAM & SAM+\ourmethod \\
        \midrule
        10\% & 7.20 $\pm$ 0.17 & 6.64 $\pm$ 0.10 & 5.15 $\pm$ 0.05 & 4.75 $\pm$ 0.09 \\
        20\% & 9.26 $\pm$ 0.23 & 8.88 $\pm$ 0.07 & 6.08 $\pm$ 0.06 & 5.82 $\pm$ 0.05 \\
        \bottomrule
    \end{tabular}
    }
    \end{sc}
    \end{center}
\end{table}

\subsection{Comparison with simplicity bias mitigation methods}\label{app:simplicity_bias}
While previous works show that reducing simplicity bias benefits the OOD settings, we show that reducing the simplicity bias also benefits the ID settings. To confirm our hypothesis on our simplicity bias mitigation baselines, we applied EIIL~\cite{creager2021environment} and JTT~\cite{liu2021just} to train ResNet18 on CIFAR10. The choice of the baselines is because they have publicly available code and fewer hyperparameters to tune in our limited rebuttal time. 
For EIIL, we tuned lr $\in$\{1e-1, 1e-2, 1e-3, 5e-4, 1e-4\}, number of epochs $\in$\{1, 2, 4, 8\} for training the reference model, and the weight decay $\in$\{1e-3, 5e-4, 1e-4\} for training GroupDRO. For JTT, we tuned the separating epoch $\in$\{4, 5, 6, 7, 8, 10\} and upsampling factor $\in$\{2, 3, 4, 5\}, while lr and weight decay follow the standard training of ResNet18 on CIFAR10.
Table~\ref{tab:simplicity_bias} shows that all methods successfully reduce the simplicity bias, yielding an improvement over SGD. While EIIL requires tuning 3 hyperparameters with a total of 60 combinations, \ourmethod~only requires one hyperparameter, which is the separating epoch within a small range (around the time when the slope of the training loss curve diminishes).

\begin{table}[!t]
    \caption{Test errors of different simplicity bias reduction methods on CIFAR10. Results are averaged over 3 seeds.}
    \label{tab:simplicity_bias}
    \begin{center}
    \begin{sc}
    \begin{tabular}{cccc}
        \toprule
        SGD & EIIL & JTT & SGD+\ourmethod \\
        \midrule
        5.07 $\pm$ 0.04 & 5.04 $\pm$ 0.04 & 4.89 $\pm$ 0.03 & \textbf{4.79 $\pm$ 0.05} \\
        \bottomrule
    \end{tabular}
    \end{sc}
    \end{center}
\end{table}

\begin{figure*}[t!]
    \centering
    \includegraphics[width=\textwidth]{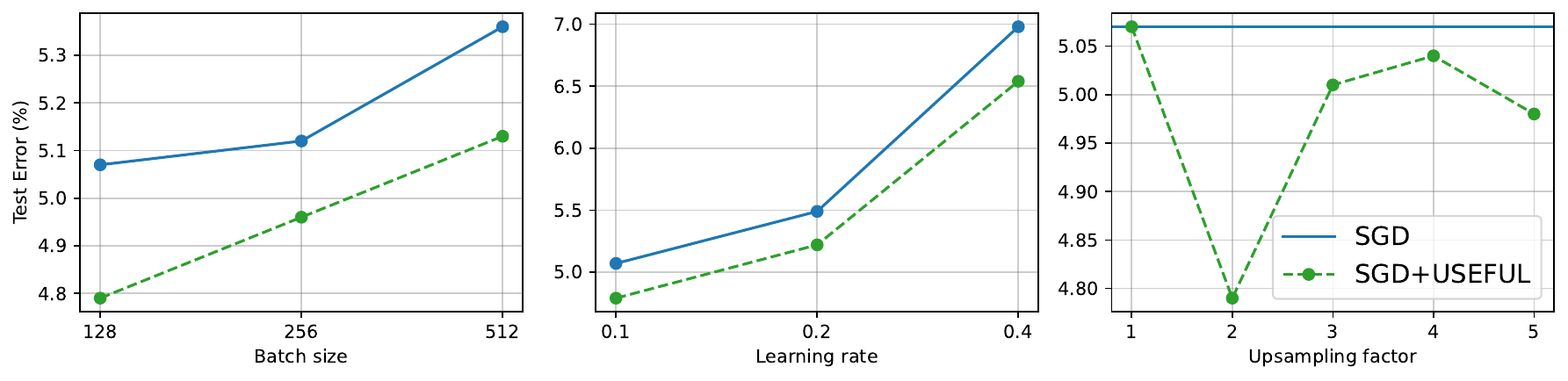} 
    \vspace{-4mm}
    \caption{\textbf{Ablation studies of training ResNet18 on CIFAR10.} In each experiment, we used the standard training settings while (left) varying training batch size or (middle) varying learning rate, or (right) varying upsampling factor.}
    \label{fig:ablation_study}
    \vspace{-3mm}
\end{figure*}

\begin{figure*}[t!]
    \centering
    \begin{subfigure}{0.48\textwidth}
        \centering
        \includegraphics[width=\columnwidth]{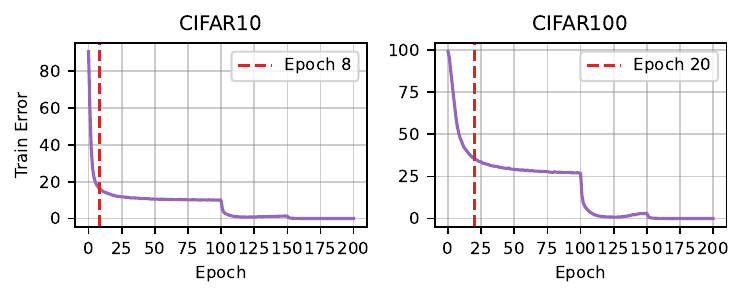}
        \caption{Training error trajectories}
        \label{fig:training_error}
    \end{subfigure}
    \hfill
    \begin{subfigure}{0.48\textwidth}
        \centering
        \includegraphics[width=\columnwidth]{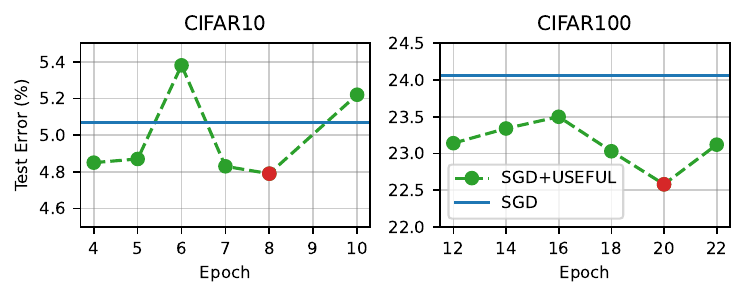}
        \caption{Test classification error}
        \label{fig:vary_separating_epochs}
    \end{subfigure}
    \hfill
    \caption{\textbf{Separating epoch analysis.} (left) Red lines indicate our optimal choice of $t$ to separate examples at and restart training. The early epoch that can best separate the examples is when the change in training error starts to shrink. (right) Red points indicate our optimal choice of $t$.}
    \label{fig:separating_epochs}
    \vspace{-3mm}
\end{figure*} 

\subsection{Ablation studies}\label{app:ablation_studies}
\textbf{\ourmethod~vs Upweighting loss.} We compare \ourmethod~with upweighting loss of examples in the \slow clusters. As can be seen in Table~\ref{tab:upweighting_loss}, when coupling with either SGD or SAM, \ourmethod~clearly outperforms upweighting loss on both CIFAR datasets. It is worth mentioning that upweighting loss is different from iteratively importance sampling methods such as GRW~\cite{zhai2022understanding}, which dynamically upweights examples during the training by factors that depend on the loss value. In addition, GRW is dedicated to the distribution shift setting while our paper considers the in-distribution setting.

\textbf{Data selection method.} In this experiment, we compare clustering with other methods for partitioning data. The first baseline is to upsample misclassified examples (\textsc{Misclassification}) while the second baseline is to upsample all examples whose training errors are larger than the median value (\textsc{Quantile}). All the three methods are performed at the same epoch $t$. Table~\ref{tab:compare_partition_methods} shows that \ourmethod~selects a better set of upsampling examples, leading to the best model performance. 

\textbf{Training batch size.} Figure~\ref{fig:ablation_study} left shows the gap between \ourmethod and SGD when changing the batch size. Our method consistently improves the performance, proving its effectiveness is not simply captured by the gradient variance caused by the training batch size. 

\textbf{Learning rate.} The small learning rate is a nuance in our theoretical results to guarantee that \fast and \slow features are separable in early training. In general, a small learning rate is required for most theoretical results on gradient descent and its convergence and is a standard theoretical assumption~\cite{roux2012stochastic,schmidt2017minimizing,wen2022does,chen2023does}. In practice, both for separating \fast vs \slow examples and for training on the upsampled data, we used the standard learning rate that results in the best generalization performance for both SGD and SAM following prior work~\cite{kwon2021asam,zheng2021regularizing,andriushchenko2022towards}. While the theoretical requirement for the learning rate is always smaller than the one that is used in practice, empirically a larger learning rate does not yield better generalization for the problems we considered in contrast to other settings~\cite{li2019towards,puli2024don}. As shown in Figure~\ref{fig:ablation_study} middle, increasing the learning rate has an adverse effect on the model performance. Indeed, for a fair comparison, the algorithms should be trained with hyperparameters that empirically yield the best performance; otherwise, the conclusions are not valid. Nevertheless, \ourmethod~always improves the test error across different learning rates. 

\textbf{Upsampling factor.} We empirically found the upsampling factor of 2 to consistently work well across different datasets and architectures. Using larger upsampling factors results in a too-large discrepancy between the training and test distribution and does not work better, as is expected and discussed in Section~\ref{sec:introduction}. As illustrated in Figure~\ref{fig:ablation_study} right, while all factors from 2 to 5 bring performance improvement, the upsampling factor of 2 yields the best performance, as it reduces the simplicity bias with minimum change to the rest of the data distribution.

\textbf{Separating epoch.}
Fig~\ref{fig:training_error} shows that the early epoch that can best separate the examples is when the change in training error starts to shrink. At this time, examples that are learned can be separated from the rest by clustering model outputs as analyzed in our theoretical results in Section~\ref{subsec:theoretical_results}. Figure~\ref{fig:vary_separating_epochs} demonstrates the performance of \ourmethod~when separating examples at different epochs in early training. Too early or too late epochs do not cluster examples well, i.e., some examples with \fast features fall into the \slow clusters and vice versa. This ablation study shows that upsampling correct examples and with enough amount is important for our method to achieve its best. Note that there is no universal separating epoch. This is reasonable because each data has a different data distribution, i.e. \slow~and \fast~features, thus, a different theoretical time $T$ in Theorems~\ref{the:easy_difficult_gd_main} and~\ref{the:easy_difficult_sam_main}.

\newpage
\section*{NeurIPS Paper Checklist}

\begin{enumerate}

\item {\bf Claims}
    \item[] Question: Do the main claims made in the abstract and introduction accurately reflect the paper's contributions and scope?
    \item[] Answer: \answerYes{} %
    \item[] Justification: Our claims match the theoretical and experimental results
    \item[] Guidelines:
    \begin{itemize}
        \item The answer NA means that the abstract and introduction do not include the claims made in the paper.
        \item The abstract and/or introduction should clearly state the claims made, including the contributions made in the paper and important assumptions and limitations. A No or NA answer to this question will not be perceived well by the reviewers. 
        \item The claims made should match theoretical and experimental results, and reflect how much the results can be expected to generalize to other settings. 
        \item It is fine to include aspirational goals as motivation as long as it is clear that these goals are not attained by the paper. 
    \end{itemize}

\item {\bf Limitations}
    \item[] Question: Does the paper discuss the limitations of the work performed by the authors?
    \item[] Answer: \answerYes{} %
    \item[] Justification: We do not have theoretical results for non-CNN architectures.
    \item[] Guidelines:
    \begin{itemize}
        \item The answer NA means that the paper has no limitation while the answer No means that the paper has limitations, but those are not discussed in the paper. 
        \item The authors are encouraged to create a separate "Limitations" section in their paper.
        \item The paper should point out any strong assumptions and how robust the results are to violations of these assumptions (e.g., independence assumptions, noiseless settings, model well-specification, asymptotic approximations only holding locally). The authors should reflect on how these assumptions might be violated in practice and what the implications would be.
        \item The authors should reflect on the scope of the claims made, e.g., if the approach was only tested on a few datasets or with a few runs. In general, empirical results often depend on implicit assumptions, which should be articulated.
        \item The authors should reflect on the factors that influence the performance of the approach. For example, a facial recognition algorithm may perform poorly when image resolution is low or images are taken in low lighting. Or a speech-to-text system might not be used reliably to provide closed captions for online lectures because it fails to handle technical jargon.
        \item The authors should discuss the computational efficiency of the proposed algorithms and how they scale with dataset size.
        \item If applicable, the authors should discuss possible limitations of their approach to address problems of privacy and fairness.
        \item While the authors might fear that complete honesty about limitations might be used by reviewers as grounds for rejection, a worse outcome might be that reviewers discover limitations that aren't acknowledged in the paper. The authors should use their best judgment and recognize that individual actions in favor of transparency play an important role in developing norms that preserve the integrity of the community. Reviewers will be specifically instructed to not penalize honesty concerning limitations.
    \end{itemize}

\item {\bf Theory Assumptions and Proofs}
    \item[] Question: For each theoretical result, does the paper provide the full set of assumptions and a complete (and correct) proof?
    \item[] Answer: \answerYes{} %
    \item[] Justification: We provide full assumptions and proofs in Appendix.
    \item[] Guidelines:
    \begin{itemize}
        \item The answer NA means that the paper does not include theoretical results. 
        \item All the theorems, formulas, and proofs in the paper should be numbered and cross-referenced.
        \item All assumptions should be clearly stated or referenced in the statement of any theorems.
        \item The proofs can either appear in the main paper or the supplemental material, but if they appear in the supplemental material, the authors are encouraged to provide a short proof sketch to provide intuition. 
        \item Inversely, any informal proof provided in the core of the paper should be complemented by formal proofs provided in appendix or supplemental material.
        \item Theorems and Lemmas that the proof relies upon should be properly referenced. 
    \end{itemize}

    \item {\bf Experimental Result Reproducibility}
    \item[] Question: Does the paper fully disclose all the information needed to reproduce the main experimental results of the paper to the extent that it affects the main claims and/or conclusions of the paper (regardless of whether the code and data are provided or not)?
    \item[] Answer: \answerYes{} %
    \item[] Justification: We provide detailed experimental settings in the Experiment section and in Appendix.
    \item[] Guidelines:
    \begin{itemize}
        \item The answer NA means that the paper does not include experiments.
        \item If the paper includes experiments, a No answer to this question will not be perceived well by the reviewers: Making the paper reproducible is important, regardless of whether the code and data are provided or not.
        \item If the contribution is a dataset and/or model, the authors should describe the steps taken to make their results reproducible or verifiable. 
        \item Depending on the contribution, reproducibility can be accomplished in various ways. For example, if the contribution is a novel architecture, describing the architecture fully might suffice, or if the contribution is a specific model and empirical evaluation, it may be necessary to either make it possible for others to replicate the model with the same dataset, or provide access to the model. In general. releasing code and data is often one good way to accomplish this, but reproducibility can also be provided via detailed instructions for how to replicate the results, access to a hosted model (e.g., in the case of a large language model), releasing of a model checkpoint, or other means that are appropriate to the research performed.
        \item While NeurIPS does not require releasing code, the conference does require all submissions to provide some reasonable avenue for reproducibility, which may depend on the nature of the contribution. For example
        \begin{enumerate}
            \item If the contribution is primarily a new algorithm, the paper should make it clear how to reproduce that algorithm.
            \item If the contribution is primarily a new model architecture, the paper should describe the architecture clearly and fully.
            \item If the contribution is a new model (e.g., a large language model), then there should either be a way to access this model for reproducing the results or a way to reproduce the model (e.g., with an open-source dataset or instructions for how to construct the dataset).
            \item We recognize that reproducibility may be tricky in some cases, in which case authors are welcome to describe the particular way they provide for reproducibility. In the case of closed-source models, it may be that access to the model is limited in some way (e.g., to registered users), but it should be possible for other researchers to have some path to reproducing or verifying the results.
        \end{enumerate}
    \end{itemize}

\item {\bf Open access to data and code}
    \item[] Question: Does the paper provide open access to the data and code, with sufficient instructions to faithfully reproduce the main experimental results, as described in supplemental material?
    \item[] Answer: \answerYes{} %
    \item[] Justification: We use open datasets, which have been cited and described in the paper. We also provide our code for reproducing the experimental results.
    \item[] Guidelines:
    \begin{itemize}
        \item The answer NA means that paper does not include experiments requiring code.
        \item Please see the NeurIPS code and data submission guidelines (\url{https://nips.cc/public/guides/CodeSubmissionPolicy}) for more details.
        \item While we encourage the release of code and data, we understand that this might not be possible, so “No” is an acceptable answer. Papers cannot be rejected simply for not including code, unless this is central to the contribution (e.g., for a new open-source benchmark).
        \item The instructions should contain the exact command and environment needed to run to reproduce the results. See the NeurIPS code and data submission guidelines (\url{https://nips.cc/public/guides/CodeSubmissionPolicy}) for more details.
        \item The authors should provide instructions on data access and preparation, including how to access the raw data, preprocessed data, intermediate data, and generated data, etc.
        \item The authors should provide scripts to reproduce all experimental results for the new proposed method and baselines. If only a subset of experiments are reproducible, they should state which ones are omitted from the script and why.
        \item At submission time, to preserve anonymity, the authors should release anonymized versions (if applicable).
        \item Providing as much information as possible in supplemental material (appended to the paper) is recommended, but including URLs to data and code is permitted.
    \end{itemize}

\item {\bf Experimental Setting/Details}
    \item[] Question: Does the paper specify all the training and test details (e.g., data splits, hyperparameters, how they were chosen, type of optimizer, etc.) necessary to understand the results?
    \item[] Answer: \answerYes{} %
    \item[] Justification: We provide detailed experimental settings in the Experiment section and in Appendix.
    \item[] Guidelines:
    \begin{itemize}
        \item The answer NA means that the paper does not include experiments.
        \item The experimental setting should be presented in the core of the paper to a level of detail that is necessary to appreciate the results and make sense of them.
        \item The full details can be provided either with the code, in appendix, or as supplemental material.
    \end{itemize}

\item {\bf Experiment Statistical Significance}
    \item[] Question: Does the paper report error bars suitably and correctly defined or other appropriate information about the statistical significance of the experiments?
    \item[] Answer: \answerYes{} %
    \item[] Justification: We report the error bar for averaging multiple seeds in our experiments.
    \item[] Guidelines:
    \begin{itemize}
        \item The answer NA means that the paper does not include experiments.
        \item The authors should answer "Yes" if the results are accompanied by error bars, confidence intervals, or statistical significance tests, at least for the experiments that support the main claims of the paper.
        \item The factors of variability that the error bars are capturing should be clearly stated (for example, train/test split, initialization, random drawing of some parameter, or overall run with given experimental conditions).
        \item The method for calculating the error bars should be explained (closed form formula, call to a library function, bootstrap, etc.)
        \item The assumptions made should be given (e.g., Normally distributed errors).
        \item It should be clear whether the error bar is the standard deviation or the standard error of the mean.
        \item It is OK to report 1-sigma error bars, but one should state it. The authors should preferably report a 2-sigma error bar than state that they have a 96\% CI, if the hypothesis of Normality of errors is not verified.
        \item For asymmetric distributions, the authors should be careful not to show in tables or figures symmetric error bars that would yield results that are out of range (e.g. negative error rates).
        \item If error bars are reported in tables or plots, The authors should explain in the text how they were calculated and reference the corresponding figures or tables in the text.
    \end{itemize}

\item {\bf Experiments Compute Resources}
    \item[] Question: For each experiment, does the paper provide sufficient information on the computer resources (type of compute workers, memory, time of execution) needed to reproduce the experiments?
    \item[] Answer: \answerYes{} %
    \item[] Justification: We provide the information on our computational resources in Appendix.
    \item[] Guidelines:
    \begin{itemize}
        \item The answer NA means that the paper does not include experiments.
        \item The paper should indicate the type of compute workers CPU or GPU, internal cluster, or cloud provider, including relevant memory and storage.
        \item The paper should provide the amount of compute required for each of the individual experimental runs as well as estimate the total compute. 
        \item The paper should disclose whether the full research project required more compute than the experiments reported in the paper (e.g., preliminary or failed experiments that didn't make it into the paper). 
    \end{itemize}
    
\item {\bf Code Of Ethics}
    \item[] Question: Does the research conducted in the paper conform, in every respect, with the NeurIPS Code of Ethics \url{https://neurips.cc/public/EthicsGuidelines}?
    \item[] Answer: \answerYes{} %
    \item[] Justification: Our paper conform to every aspect in the NeurIPS Code of Ethics.
    \item[] Guidelines:
    \begin{itemize}
        \item The answer NA means that the authors have not reviewed the NeurIPS Code of Ethics.
        \item If the authors answer No, they should explain the special circumstances that require a deviation from the Code of Ethics.
        \item The authors should make sure to preserve anonymity (e.g., if there is a special consideration due to laws or regulations in their jurisdiction).
    \end{itemize}

\item {\bf Broader Impacts}
    \item[] Question: Does the paper discuss both potential positive societal impacts and negative societal impacts of the work performed?
    \item[] Answer: \answerNo{} %
    \item[] Justification: This paper presents work whose goal is to advance the field of Machine Learning. There are many potential societal consequences of our work, none which we feel must be specifically highlighted here.
    \item[] Guidelines:
    \begin{itemize}
        \item The answer NA means that there is no societal impact of the work performed.
        \item If the authors answer NA or No, they should explain why their work has no societal impact or why the paper does not address societal impact.
        \item Examples of negative societal impacts include potential malicious or unintended uses (e.g., disinformation, generating fake profiles, surveillance), fairness considerations (e.g., deployment of technologies that could make decisions that unfairly impact specific groups), privacy considerations, and security considerations.
        \item The conference expects that many papers will be foundational research and not tied to particular applications, let alone deployments. However, if there is a direct path to any negative applications, the authors should point it out. For example, it is legitimate to point out that an improvement in the quality of generative models could be used to generate deepfakes for disinformation. On the other hand, it is not needed to point out that a generic algorithm for optimizing neural networks could enable people to train models that generate Deepfakes faster.
        \item The authors should consider possible harms that could arise when the technology is being used as intended and functioning correctly, harms that could arise when the technology is being used as intended but gives incorrect results, and harms following from (intentional or unintentional) misuse of the technology.
        \item If there are negative societal impacts, the authors could also discuss possible mitigation strategies (e.g., gated release of models, providing defenses in addition to attacks, mechanisms for monitoring misuse, mechanisms to monitor how a system learns from feedback over time, improving the efficiency and accessibility of ML).
    \end{itemize}
    
\item {\bf Safeguards}
    \item[] Question: Does the paper describe safeguards that have been put in place for responsible release of data or models that have a high risk for misuse (e.g., pretrained language models, image generators, or scraped datasets)?
    \item[] Answer: \answerNA{} %
    \item[] Justification: There are no potential harms to our models.
    \item[] Guidelines:
    \begin{itemize}
        \item The answer NA means that the paper poses no such risks.
        \item Released models that have a high risk for misuse or dual-use should be released with necessary safeguards to allow for controlled use of the model, for example by requiring that users adhere to usage guidelines or restrictions to access the model or implementing safety filters. 
        \item Datasets that have been scraped from the Internet could pose safety risks. The authors should describe how they avoided releasing unsafe images.
        \item We recognize that providing effective safeguards is challenging, and many papers do not require this, but we encourage authors to take this into account and make a best faith effort.
    \end{itemize}

\item {\bf Licenses for existing assets}
    \item[] Question: Are the creators or original owners of assets (e.g., code, data, models), used in the paper, properly credited and are the license and terms of use explicitly mentioned and properly respected?
    \item[] Answer: \answerYes{} %
    \item[] Justification: We cite all data, code, and models used properly. 
    \item[] Guidelines:
    \begin{itemize}
        \item The answer NA means that the paper does not use existing assets.
        \item The authors should cite the original paper that produced the code package or dataset.
        \item The authors should state which version of the asset is used and, if possible, include a URL.
        \item The name of the license (e.g., CC-BY 4.0) should be included for each asset.
        \item For scraped data from a particular source (e.g., website), the copyright and terms of service of that source should be provided.
        \item If assets are released, the license, copyright information, and terms of use in the package should be provided. For popular datasets, \url{paperswithcode.com/datasets} has curated licenses for some datasets. Their licensing guide can help determine the license of a dataset.
        \item For existing datasets that are re-packaged, both the original license and the license of the derived asset (if it has changed) should be provided.
        \item If this information is not available online, the authors are encouraged to reach out to the asset's creators.
    \end{itemize}

\item {\bf New Assets}
    \item[] Question: Are new assets introduced in the paper well documented and is the documentation provided alongside the assets?
    \item[] Answer: \answerNA{} %
    \item[] Justification: We do not release new assets.
    \item[] Guidelines:
    \begin{itemize}
        \item The answer NA means that the paper does not release new assets.
        \item Researchers should communicate the details of the dataset/code/model as part of their submissions via structured templates. This includes details about training, license, limitations, etc. 
        \item The paper should discuss whether and how consent was obtained from people whose asset is used.
        \item At submission time, remember to anonymize your assets (if applicable). You can either create an anonymized URL or include an anonymized zip file.
    \end{itemize}

\item {\bf Crowdsourcing and Research with Human Subjects}
    \item[] Question: For crowdsourcing experiments and research with human subjects, does the paper include the full text of instructions given to participants and screenshots, if applicable, as well as details about compensation (if any)? 
    \item[] Answer: \answerNA{} %
    \item[] Justification: Our paper does not involve crowdsourcing nor research with human subjects.
    \item[] Guidelines:
    \begin{itemize}
        \item The answer NA means that the paper does not involve crowdsourcing nor research with human subjects.
        \item Including this information in the supplemental material is fine, but if the main contribution of the paper involves human subjects, then as much detail as possible should be included in the main paper. 
        \item According to the NeurIPS Code of Ethics, workers involved in data collection, curation, or other labor should be paid at least the minimum wage in the country of the data collector. 
    \end{itemize}

\item {\bf Institutional Review Board (IRB) Approvals or Equivalent for Research with Human Subjects}
    \item[] Question: Does the paper describe potential risks incurred by study participants, whether such risks were disclosed to the subjects, and whether Institutional Review Board (IRB) approvals (or an equivalent approval/review based on the requirements of your country or institution) were obtained?
    \item[] Answer: \answerNA{} %
    \item[] Justification: Our paper does not involve crowdsourcing nor research with human subjects.
    \item[] Guidelines:
    \begin{itemize}
        \item The answer NA means that the paper does not involve crowdsourcing nor research with human subjects.
        \item Depending on the country in which research is conducted, IRB approval (or equivalent) may be required for any human subjects research. If you obtained IRB approval, you should clearly state this in the paper. 
        \item We recognize that the procedures for this may vary significantly between institutions and locations, and we expect authors to adhere to the NeurIPS Code of Ethics and the guidelines for their institution. 
        \item For initial submissions, do not include any information that would break anonymity (if applicable), such as the institution conducting the review.
    \end{itemize}

\end{enumerate}

\end{document}